\documentclass[lettersize,journal]{IEEEtran}
\usepackage{amsmath,amsfonts}
\usepackage{algorithmic}
\usepackage{algorithm}
\usepackage{array}
\usepackage[caption=false,font=normalsize,labelfont=sf,textfont=sf]{subfig}
\usepackage{textcomp}
\usepackage{stfloats}
\usepackage{url}
\usepackage{verbatim}
\usepackage{graphicx}
\usepackage{cite}
\usepackage{enumitem}
\usepackage[colorlinks,
            linkcolor=blue,
            anchorcolor=blue,
            citecolor=blue]{hyperref}
\usepackage{tcolorbox}
\usepackage{amsthm}
\usepackage{makecell}  
\usepackage{multirow} 
\usepackage{booktabs}
\usepackage{tikz}
\usetikzlibrary{positioning,arrows.meta}

\newtheorem{assumption}{Assumption}
\newtheorem{definition}{Definition}
\newtheorem{lemma}{Lemma}
\newtheorem{theorem}{Theorem}
\newtheorem{corollary}{Corollary}

\newtheorem{remark}{Remark}

\begin{document}

\title{Toward Trustworthy AI: Multi-Target Adversarial Attacks and Robust Defenses for Continuous Data Summarization}
\author{
\thanks{Manuscript received April 19, 2021; revised August 16, 2021.}
Yuefang Lian, 
Longkun Guo$^{\dagger,\ddagger}$, 
\IEEEmembership{Senior Member, IEEE},
Zhongrui Zhao$^{\dagger}$, 
Zhigang Lu, 
\IEEEmembership{Member, IEEE},
Yanan Cai, 
Shuchao Pang, 
\IEEEmembership{Member, IEEE},
Dachuan Xu, 
Jason Xue
\IEEEmembership{Senior Member, IEEE},
\thanks{$\dagger$: equal contribution, $\ddagger$: corresponding author}
\thanks{Yuefang Lian is with Nankai University, Tianjin, China. This work was done when Yuefang was a visiting PhD student with James Cook University and Western Sydney University during PhD candidature at Beijing University of Technology.}
\thanks{Longkun Guo is with Fuzhou University, Fuzhou, China.}
\thanks{Zhongrui Zhao and Yanan Cai are with James Cook University, Townsville, Australia and Western Sydney University, Sydney, Australia.}
\thanks{Zhigang Lu is with Western Sydney University, Sydney, Australia. This work was partially done when Zhigang was a Lecturer with James Cook University.}
\thanks{Shuchao Pang is with Nanjing University of Science and Technology, Nanjing, China.}
\thanks{Dachuan Xu is with Beijing University of Technology, Beijing, China.}
\thanks{Jason Xue is with CSIRO's Data 61, Sydney, Australia and Responsible AI Research (RAIR) Centre, The University of Adelaide, Australia.}

}

\markboth{IEEE TRANSACTIONS ON INFORMATION FORENSICS AND SECURITY,~Vol.~14, No.~8, August~2021}%
{Lian \MakeLowercase{\textit{et al.}}: Toward Trustworthy AI: Multi-Target Adversarial Attacks and Robust Defense}


\maketitle

\begin{abstract}

Trustworthy AI requires reliable data-processing pipelines, not only robust downstream predictive models. 
As an upstream component, data summarization determines which information is retained and passed to subsequent learning or decision modules. Therefore, adversarial perturbations to the summarization process can compromise trustworthy AI in an upstream manner: they may alter the selected summary, reduce its representativeness, and further degrade the utility of subsequent learning tasks. In this paper, we study adversarial attacks on continuous data summarization under similarity-level perturbations through DR-submodular optimization. We show that a class of multi-resolution image summarization objectives can be formulated as multilinear extensions of non-negative submodular set functions and satisfy DR-submodularity with $m$-weak monotonicity. 
We then formulate multi-target attack generation as a min-max problem, where one admissible perturbation of the similarity structure is optimized to degrade multiple target summarization models. 
To mitigate such perturbations, we formulate robust defense against mixed attack types as a regularized max-min problem. 
For both problems, we develop approximation algorithms with theoretical guarantees.  Experiments on real-data and controlled clustered benchmarks show that the proposed attack is effective in representative low-to-moderate budget regimes and can induce downstream task-performance loss. 
The proposed defense improves the robustness--mitigation trade-off in structured settings, while also revealing the parameter sensitivity of robust protection on real data.
\end{abstract}

\section{Introduction}

\IEEEPARstart{T}{rustworthy} AI requires the reliability of the entire data-processing pipeline, not only the robustness of downstream predictive models. In many AI systems, data summarization and sample selection are used as upstream components to extract representative information before training, retrieval, storage, or decision making. Although these procedures are often treated as benign preprocessing steps, they determine which data are retained and passed to subsequent modules. Therefore, adversarial perturbations to the summarization process may compromise trustworthy AI in an upstream manner: they can remove representative information, increase redundancy, distort the selected training or retrieval set, and eventually reduce downstream task reliability. This makes adversarial robustness of data summarization a security-relevant problem for trustworthy AI pipelines.

Continuous submodular maximization provides a natural mathematical framework for modeling summarization and selection tasks. Classical results on submodular functions and multilinear relaxations~\cite{Lovasz1983,Chekuri2011} established the importance of diminishing-returns structure in discrete and continuous optimization, and subsequent studies have used this structure for budget allocation, data summarization, and representative selection~\cite{soma2014,schreiber2020apricot,kaushal2022submodlib}. Submodular optimization has also appeared in security-related data-processing problems, such as attack detection scheduling in large-scale networks under false data injection attacks~\cite{Suo2024}. 
However, these works do not study adversarial vulnerability of the summarization process itself as an upstream component in trustworthy AI pipelines.

This upstream vulnerability becomes particularly relevant in similarity-based continuous summarization.
Many summarization objectives are constructed from pairwise similarities or feature-induced similarity scores, and the summarizer optimizes a soft selection vector based on this similarity structure. 
Thus, the similarity-construction stage naturally becomes an attack surface: an adversary may perturb the similarity matrix without directly modifying the final summary decision variable. 
Such perturbations may arise from manipulated feature representations in retrieval-style systems or poisoned samples that alter neighborhood- or similarity-based structures~\cite{Tolias2019,chen2021data}. 
After the summarization module is executed on the perturbed objective, the selected summary may become less representative when evaluated under the original clean objective. 
The risk is further amplified in multi-target settings, where several summarization models may be built from related data sources or similar feature representations. 
In such cases, a single perturbation to the upstream similarity structure may simultaneously degrade multiple target summarization models, motivating a multi-target attack formulation for similarity-based continuous summarization.

Existing studies on adversarial attacks and robust optimization provide important foundations, but they do not directly address the setting considered in this paper. 
On the attack side, prior work has studied adversarial sample generation for discrete submodular optimization~\cite{adibi2022,Lei2019} and general non-convex min-max formulations~\cite{wang2021}. 
However, these methods do not explicitly model the combination of continuity, DR-submodularity, and $m$-weak monotonicity that arises in the continuous summarization problem considered here, nor do they focus on generating one shared perturbation that attacks multiple summarization models simultaneously. 
On the defense side, robust submodular optimization has been studied for set functions, monotone continuous models, and general non-convex robust formulations~\cite{staib2017,lee2022,lian2024zeroth,wang2021}. 
Yet these methods do not explicitly address robust protection for weakly monotone DR-submodular summarization models under mixed similarity-level attack types. 
As a result, the attack and defense mechanisms of continuous summarization remain insufficiently understood, especially from the perspective of downstream reliability in trustworthy AI pipelines.

In this paper, we study the adversarial vulnerability and robust protection of continuous data summarization under perturbations of the upstream similarity structure. 
Instead of directly modifying the final summary decision, the adversary perturbs the similarity matrix that defines the summarization objective. 
This similarity-level threat model captures an upstream attack surface in trustworthy AI pipelines, where degraded summaries may affect the representative information passed to downstream modules. To address this issue, we develop a DR-submodular optimization framework for multi-target attack generation and robust defense, and evaluate it on real-data multilinear-extension summarization and a controlled clustered multi-target benchmark. 
The former tests the theory-aligned objective on real data, while the latter provides an interpretable setting for analyzing the attack mechanism and downstream consequences through known cluster and representative structure.

\textbf{Contributions.} The main contributions of this paper are summarized as follows.
\begin{itemize}
\item We provide a DR-submodular formulation for a class of multi-resolution image summarization objectives. 
Specifically, we show that the corresponding discrete utility is non-negative, submodular, and \(m\)-weakly monotone under a controlled redundancy condition, and that its multilinear extension preserves non-negativity, DR-submodularity, and weak monotonicity.

\item We introduce a similarity-level threat model for continuous data summarization, where the adversary perturbs the upstream similarity structure rather than directly modifying the final summary decision. 
Under this model, we formulate multi-target attack generation for data summarization as a structured min-max optimization problem and develop a structure-aware approximation algorithm for constructing one shared perturbation across multiple target summarization models.

\item We formulate robust defense against admissible similarity-level perturbations as a regularized max-min optimization problem. 
We develop a robust continuous grtype algorithm for \(m\)-weakly monotone DR-submodular objectives and establish an approximation guarantee under standard smoothness.

\item We conduct empirical evaluations on real-data multilinear-extension summarization and a controlled clustered multi-target benchmark. 
The results show that optimized similarity perturbations can induce measurable summarization degradation in representative budget regimes. 
Downstream evaluations further demonstrate that such degradation can translate into task-level performance loss, and that robust summaries can recover downstream reliability by restoring class/cluster coverage. 
These findings connect adversarial summarization to the reliability of trustworthy AI pipelines.
\end{itemize}

\section{Related Work}

This section reviews prior work related to adversarial attack generation and robust optimization in submodular settings, with an emphasis on their relevance to secure data processing and trustworthy AI.

\subsection{Attack Generation with Submodular Structure}

Attack generation has been studied in several optimization and security-related settings where submodular structure is used to model diminishing returns in perturbation, selection, or degradation effects. In DC microgrids, Liu et al.~\cite{Liu2025} leveraged submodular structure to construct false data injection attacks by exploiting the submodularity of system state error. Zhu et al.~\cite{Zhu2023} studied attacks that degrade network robustness through edge perturbations, using greedy optimization for robustness-related objectives. In text classification, Lei et al.~\cite{Lei2019} showed that the perturbation search space for attacking certain neural models exhibits submodularity, allowing attack generation to be formulated as a submodular optimization problem. More closely related to our setting, Adibi et al.~\cite{adibi2022} studied convex-submodular min-max optimization and obtained theoretical guarantees for discrete submodular attack generation and the multilinear extension of monotone set-submodular functions. In a different direction, Wang et al.~\cite{wang2021} considered adversarial attack generation from a general non-convex min-max optimizaton perspective and analyzed convergence toward stationary solutions.

These works provide important foundations for studying adversarial vulnerability in structured optimization problems. However, they do not explicitly address the setting considered here, where the attack target consists of multiple continuous summarization models with DR-submodularity and $m$-weak monotonicity, and where the goal is to construct a single perturbation that simultaneously degrades several target summarization models. Our work complements these lines of research by focusing on adversarial robustness evaluation for continuous data summarization, a security-relevant component of trustworthy AI pipelines.

\subsection{Robust Submodular Optimization}

Robust submodular optimization has been studied in several important application domains, including observation selection~\cite{Krause2008} and influence maximization~\cite{chen2016,he2018}. These works typically consider maximizing the worst-case value over a family of submodular set functions. Another relevant line studies robust optimization in adversarial non-convex settings, where problems are formulated as nonconvex-concave saddle-point problems and solved by exploiting connections to continuous submodularity~\cite{staib2017}, with extensions to security games~\cite{Wilder2018} and distributionally robust settings~\cite{staib2019}.

For continuous submodular optimization, Lee et al.~\cite{lee2022} studied non-smooth and H\"older-smooth submodular maximization and derived approximation guarantees that can be used in robust formulations. Lian et al.~\cite{lian2024zeroth} proposed a zeroth-order approximation algorithm for robust DR-submodular maximization by solving a non-smooth up-concave optimization problem. Wang et al.~\cite{wang2021} also considered robust optimization against mixed adversarial attacks in a general non-convex setting, but the resulting guarantees are given in terms of stationary solutions. These works are closely related to the defense side of our problem, but they do not explicitly address robust protection for continuous DR-submodular summarization models with $m$-weak monotonicity under mixed attack types. Our work extends this literature by studying robust protection for an upstream data summarization component in trustworthy AI pipelines, where adversarial perturbations may affect multiple target summarization models simultaneously.

Table~\ref{tab:related_positioning} summarizes representative theoretical frameworks related to our attack and defense formulations. 
This table is intended to clarify differences in modeling assumptions and guarantee types; it is not an empirical baseline list, since several methods rely on monotonicity or single-target formulations and are therefore not directly applicable to our weakly monotone multi-target summarization setting.
\begin{table*}[t!]
\centering
\normalsize
\caption{Positioning of representative optimization frameworks related to attack generation and robust defense.}
\label{tab:related_positioning}
\resizebox{\textwidth}{!}{
\begin{tabular}{@{}c c c c c c c@{}}
\toprule
\textbf{Ref.} 
& \textbf{Role} 
& \textbf{Model setting} 
& \textbf{Cont.} 
& \textbf{Multi-target / mixed}
& \textbf{Guarantee} 
& \textbf{Relation to this work} \\
\midrule
\cite{adibi2022} 
& Attack 
& Set submodular, monotone 
& No 
& No 
& $(1-1/e)$-approx. 
& Discrete, single-target \\
\cite{adibi2022} 
& Attack 
& Multilinear extension, monotone 
& Yes 
& No 
& $(1/2)$-approx. 
& Continuous but monotone and single-target \\
\cite{wang2021} 
& Attack 
& General non-convex min-max 
& Yes 
& Yes 
& Stationary point 
& Baseline, no DR-submodular ratio \\
\midrule
\cite{lee2022} 
& Defense 
& DR-submodular, monotone 
& Yes 
& No 
& $(1/2)$-approx. 
& Monotone model, no mixed attack defense \\
\cite{lian2024zeroth} 
& Defense 
& DR-submodular, monotone 
& Yes 
& No 
& $(1-1/e)$-approx. 
& Monotone robust DR-submodular setting \\
\cite{wang2021} 
& Defense 
& General non-convex robust 
& Yes 
& Yes 
& Stationary point 
& Handles mixed attacks, stationary guarantees \\
\midrule
This work 
& Attack / defense 
& DR-submodular, $m$-weakly monotone 
& Yes 
& Yes 
& $m(1-1/e)$-approx. 
& Continuous weakly monotone summarization with multi-target attack and mixed defense \\
\bottomrule
\end{tabular}
}
\vspace{-8pt}
\end{table*}

\section{Preliminaries}\label{sec:preli}

\subsection{Basic Notation and Structural Properties}

For $\mathbf{x},\mathbf{y}\in\mathbb{R}^{n}$, we denote $(\mathbf{x}\vee \mathbf{y})_{i}=\max\{\mathbf{x}_{i},\mathbf{y}_{i}\}$ and $(\mathbf{x}\wedge \mathbf{y})_{i}=\min\{\mathbf{x}_{i},\mathbf{y}_{i}\}$. Moreover, $\mathbf{x}\le \mathbf{y}$ means that $\mathbf{x}_{i}\le \mathbf{y}_{i}$ for all $i\in[n]$. We first recall the definition of continuous DR-submodularity.

\begin{definition}[DR-submodularity~\cite{bian2017guaranteed}]
\label{def:DR-proper}
Given a convex set $\mathcal{D}\subseteq \mathbb{R}^{n}$, a continuous function $f:\mathcal{D}\rightarrow \mathbb{R}$ is said to be DR-submodular over $\mathcal{D}$ if, for any $\mathbf{x}\le \mathbf{y}\in\mathcal{D}$ and any $a\in\mathbb{R}_{+}$ such that $a\mathbf{e}_{i}+\mathbf{x}\in\mathcal{D}$ and $a\mathbf{e}_{i}+\mathbf{y}\in\mathcal{D}$, the following diminishing-returns property holds for every $i\in\{1,\dots,n\}$:
\[
f(a\mathbf{e}_{i}+\mathbf{x})-f(\mathbf{x})
\ge
f(a\mathbf{e}_{i}+\mathbf{y})-f(\mathbf{y}).
\]
\end{definition}

Throughout the paper, the feasible set $\mathcal{X}$ is assumed to be a down-closed convex set, namely, if $\mathbf{x}\in\mathcal{X}$ and $\mathbf{0}\le \mathbf{y}\le \mathbf{x}$, then $\mathbf{y}\in\mathcal{X}$. We also assume that the objective function is non-negative. To characterize DR-submodular functions that are not fully monotone, we use the notion of $m$-weak monotonicity over continuous domains, extending the corresponding concept for set functions in~\cite{Mualem2022}.

\begin{definition}[$m$-weakly monotone]
\label{def:m-weak}
For the maximization of a non-negative continuous function $f$ over a feasible set $\mathcal{X}$, we say that $f$ is $m$-weakly monotone if
\[
f(\mathbf{x}\vee \mathbf{y})\ge m f(\mathbf{x}), \qquad \forall \mathbf{x},\mathbf{y}\in\mathcal{X}.
\]
\end{definition}

When $m=1$, the above definition reduces to standard monotonicity, i.e., $f(\mathbf{x}\vee\mathbf{y})\ge f(\mathbf{x})$ for all $\mathbf{x},\mathbf{y}\in\mathcal{X}$. For additional properties of DR-submodular functions and examples of $m$-weak monotonicity, we refer readers to~\cite{bian2017guaranteed,Mualem2022}. 
For constrained maximization problems, we use the standard notion of an $(\alpha,\epsilon)$-approximation, meaning that the returned solution achieves at least an $\alpha$ fraction of the optimal value up to an additive error $\epsilon$, i.e., $f(\mathbf{x}_{\rm output})\ge \alpha f(\mathbf{x}^{*})-\epsilon$, where $\mathbf{x}^{*}$ is an optimal solution, $\alpha\in(0,1]$ is the approximation ratio, and $\epsilon\ge 0$ is the optimization accuracy.

\subsection{Multi-Resolution Image Summarization Model}\label{subsec:model}

We instantiate our study through multi-resolution image summarization, which serves as a representative upstream data summarization model in trustworthy AI pipelines. 
Given a dataset $\Omega$ with $|\Omega|=n$, the goal is to select a subset of representative images that provides good coverage of the dataset while avoiding excessive redundancy. 
Instead of optimizing directly over discrete subsets, we adopt a continuous relaxation in which each coordinate $x_\nu\in[0,1]$ represents the probability or soft weight of selecting image $\nu$. 
After optimization, a deterministic summary can be obtained either by thresholding, e.g.,
\[
S_{\tau}=\{\nu\in\Omega:x_\nu\ge \tau\},
\]
or by selecting the top-ranked images according to $\{x_\nu\}_{\nu\in\Omega}$.

We first define the corresponding discrete summarization utility. 
For a subset $S\subseteq\Omega$, let
\begin{equation}
\label{eq:discrete_summary}
f(S,\Omega)
=
r(S,\Omega)-\frac{\lambda}{n}q(S,\Omega),
\end{equation}
where
$$
r(S,\Omega)=\sum_{\mu\in\Omega}\max_{\nu\in S}s_{\mu,\nu}
$$
is the facility-location representativeness term, and
$$
q(S,\Omega)=
\sum_{\mu\in S}\sum_{\nu\in S,\,\nu\neq \mu}s_{\mu,\nu}
$$
is the redundancy penalty. 
Here $s_{\mu,\nu}\ge 0$ denotes the similarity between images $\mu$ and $\nu$, and $\lambda\ge 0$ controls the strength of the redundancy penalty. 
For the empty set, we adopt the convention $\max_{\nu\in\emptyset}s_{\mu,\nu}=0$, and hence $f(\emptyset,\Omega)=0$.

The first term rewards selected images that well represent the remaining data points, while the second term discourages the simultaneous selection of highly similar images. 
To ensure that the redundancy penalty does not dominate the representativeness term, we define the redundancy-to-representativeness ratio
\begin{equation*}
\rho_\Omega
=
\sup_{\emptyset\neq S\subseteq\Omega}
\frac{\frac{\lambda}{n}q(S,\Omega)}{r(S,\Omega)}.
\end{equation*}
Throughout this paper, we assume that $r(S,\Omega)>0$ for every nonempty feasible set $S$ and that $0\le \rho_\Omega<1$. 
This condition means that the redundancy penalty is controlled relative to the coverage utility. 
Under this assumption, we have that $f(S,\Omega)\ge 0$ for all $S\subseteq\Omega$.

The continuous objective studied in this paper is defined as the multilinear extension of the discrete utility in~\eqref{eq:discrete_summary}. 
Specifically, for $\mathbf{x}\in[0,1]^n$, let $S\sim\mathbf{x}$ denote a random subset of $\Omega$ in which each image $\nu$ is independently selected with probability $x_\nu$. 
We define
\begin{equation}
\label{eq:image_fun}
F(\mathbf{x},\Omega)
=
\mathbb{E}_{S\sim \mathbf{x}}
\left[
f(S,\Omega)
\right].
\end{equation}
Equivalently,
$F(\mathbf{x},\Omega)
=
R_{\mathrm{multi}}(\mathbf{x},\Omega)
-
Q_{\mathrm{multi}}(\mathbf{x},\Omega),
$
where
\[
R_{\mathrm{multi}}(\mathbf{x},\Omega)
=
\sum_{\mu\in\Omega}
\mathbb{E}_{S\sim \mathbf{x}}
\left[
\max_{\nu\in S}s_{\mu,\nu}
\right],
\]
and
\[
Q_{\mathrm{multi}}(\mathbf{x},\Omega)
=
\frac{\lambda}{n}
\sum_{\mu\in\Omega}\sum_{\nu\in\Omega,\,\nu\neq \mu}
x_\mu x_\nu s_{\mu,\nu}.
\]
The expression for $Q_{\mathrm{multi}}$ follows from the independence of the random selections, since $
\mathbb{E}[\mathbf{1}_{\mu\in S}\mathbf{1}_{\nu\in S}]=x_\mu x_\nu$, where $\mu\neq \nu$.

The following lemma shows that the continuous objective belongs to the structural class studied in this paper.

\begin{lemma}[Proof in Sec.~\ref{supp:proofs-pre}]
\label{lem:multi-reso}
Suppose that $s_{\mu,\nu}\ge 0$ for all $\mu,\nu\in\Omega$, and define
\[
\rho_\Omega
=
\sup_{\emptyset\neq S\subseteq\Omega}
\frac{\frac{\lambda}{n}q(S,\Omega)}{r(S,\Omega)}.
\]
Assume that $r(S,\Omega)>0$ for every nonempty feasible set $S$ and that $0\le \rho_\Omega<1$. 
Then the discrete utility $f(S,\Omega)$ defined in~\eqref{eq:discrete_summary} is non-negative, submodular, and $m_\Omega$-weakly monotone with
$m_\Omega=1-\rho_\Omega$. Consequently, its multilinear extension $F(\mathbf{x},\Omega)$ defined in~\eqref{eq:image_fun} is non-negative, DR-submodular over $[0,1]^n$, and $m_\Omega$-weakly monotone.
\end{lemma}

Lemma~\ref{lem:multi-reso} provides the structural foundation for applying DR-submodular optimization tools to continuous data summarization. 
It shows that the proposed continuous objective inherits the diminishing-returns structure of the underlying discrete summarization utility. 
However, once the similarity matrix is perturbed, the resulting objective must be re-examined to verify whether it remains in the same structural class. 
To formalize this, we perturb the similarity matrix entrywise. Specifically, for each victim model, the perturbed similarity matrix is given by
$
s_{\mu,\nu}(\mathbf{v}) = s_{\mu,\nu} + v_{\mu,\nu},
$
where the perturbation satisfies
\[
\mathbf{v}\in\mathcal P_p(\epsilon_p)
=
\left\{
\mathbf{v}:
\|\mathbf{v}\|_p\le \epsilon_p,\;
0\le s_{\mu,\nu}+v_{\mu,\nu}\le 1,\ \forall \mu,\nu
\right\}.
\]
Thus, all perturbed similarities remain nonnegative and bounded in $[0,1]$.

The following lemma clarifies when the perturbed objective preserves the structural properties required by our analysis.

\begin{lemma}[Structural preservation under admissible perturbations]
\label{lem:attacked-structure}
Under the perturbation set $\mathcal P_p(\epsilon_p)$, the perturbed similarity matrix remains entrywise nonnegative. Moreover, if the perturbed discrete utility $f(S,\Omega(\mathbf v))$ remains submodular and satisfies $\rho_{\Omega(\mathbf v)}<1$, then its multilinear extension $F(\mathbf{x},\Omega(\mathbf v))$ is non-negative, DR-submodular with respect to $\mathbf{x}$, and $m_{\Omega(\mathbf v)}$-weakly monotone, where $
m_{\Omega(\mathbf v)} = 1-\rho_{\Omega(\mathbf v)}$.
\end{lemma}

Accordingly, the theoretical results in the following sections are stated for perturbations under which the perturbed objective remains within the same non-negative DR-submodular and weakly monotone structural class.

\subsection{Threat Model and Attack Surface}
\label{subsec:threat_model}

We consider a data-processing pipeline in which raw data are first converted into feature representations, a similarity structure is then constructed from these representations, and a continuous summarization module selects or weights representative items before the summarized data are used by downstream learning or decision modules. 
In this pipeline, the summarization module is an upstream component: it determines which information is retained and therefore can affect the reliability of subsequent tasks.

\subsubsection{Attack surface}

This paper focuses on similarity-level attacks against the summarization module. 
This perturbation model abstracts attacks on the representation or similarity-construction stage of the summarization pipeline. 
As formalized in Sec.~\ref{subsec:model}, the adversary applies an admissible perturbation $\mathbf v\in\mathcal P_p(\epsilon_p)$ to the upstream similarity structure used by the summarization objective, while all perturbed similarities remain valid scores in $[0,1]$. Note that the adversary does not directly choose the final summary, modify the summarization decision variable $\mathbf x$, or edit the rounded summary $S_k(\mathbf{x})$. 
Instead, the attack affects the summarizer indirectly by changing the objective on which the summary is optimized. The summarization algorithm is then executed on the perturbed objective, and the resulting summary is evaluated under the original clean objective to measure the quality loss caused by the upstream perturbation.

\subsubsection{Adversary's knowledge and capability}
We mainly study a white-box, gradient-access adversary. 
The adversary knows the summarization objective, the feasible set, and the perturbation budget, and can query the target summarization models to obtain function-value and gradient information. 
This setting is used to evaluate worst-case vulnerability of optimization-based summarization systems and provides an upper bound on the damage that a structure-aware adversary may cause. 
In the multi-target setting, the adversary aims to construct a single perturbation that degrades multiple target summarization models simultaneously, rather than generating a separate perturbation for each model.

\subsubsection{Defender's knowledge and objective}
The defender controls the summarization algorithm and seeks a solution that remains useful under possible perturbations of the similarity structure. 
Since the defender may not know in advance which type of perturbation will occur, we consider a mixed-attack setting. 
Each attack type $j\in[J]$ is associated with a feasible perturbation set $\mathcal P_j$. 
The defender therefore optimizes the summarization decision against the worst-case mixture of these perturbations. 
The goal is not to eliminate all possible performance loss, but to reduce attack-induced utility degradation while maintaining high clean-data summarization quality.

\subsubsection{Scope and limitations}

The proposed threat model is most appropriate for systems in which the similarity matrix, or the representation module used to construct it, is accessible or indirectly manipulable. 
We do not assume that every admissible entrywise similarity perturbation can be realized by imperceptible pixel-level changes to raw inputs. 
Instead, the perturbation set provides a controlled abstraction for analyzing how bounded changes to the upstream similarity structure affect the summarization solution and its clean-objective utility.

Fig.~\ref{fig:threat_model_attack_surface} summarizes the clean summarization pipeline, the similarity-level attack surface, the defender's robust protection mechanism, and the downstream evaluation protocol. \begin{figure*}[t]
\centering
\begin{tikzpicture}[
    font=\footnotesize,
    every node/.style={inner sep=4pt},
    pipebox/.style={
        rectangle,
        draw=black!70,
        line width=0.7pt,
        rounded corners=2pt,
        align=center,
        minimum height=0.82cm,
        text width=1.75cm,
        fill=gray!8
    },
    corebox/.style={
        rectangle,
        draw=black!80,
        line width=0.85pt,
        rounded corners=2pt,
        align=center,
        minimum height=0.92cm,
        text width=2.05cm,
        fill=gray!12
    },
    attackbox/.style={
        rectangle,
        draw=red!75!black,
        line width=0.8pt,
        rounded corners=2pt,
        align=center,
        minimum height=0.88cm,
        text width=2.05cm,
        fill=red!7
    },
    defensebox/.style={
        rectangle,
        draw=blue!70!black,
        line width=0.8pt,
        rounded corners=2pt,
        align=center,
        minimum height=0.88cm,
        text width=2.20cm,
        fill=blue!7
    },
    evalbox/.style={
        rectangle,
        draw=green!45!black,
        line width=0.75pt,
        rounded corners=2pt,
        align=center,
        minimum height=0.82cm,
        text width=2.15cm,
        fill=green!7
    },
    notebox/.style={
        rectangle,
        draw=black!45,
        line width=0.6pt,
        dashed,
        rounded corners=2pt,
        align=center,
        font=\scriptsize,
        minimum height=0.72cm,
        text width=2.75cm,
        fill=white
    },
    dataarrow/.style={
        -{Latex[length=3.2mm,width=2.2mm]},
        line width=0.9pt,
        draw=black!80,
    },
    attackarrow/.style={
        -{Latex[length=3.2mm,width=2.2mm]},
        line width=0.9pt,
        draw=red!75!black,
        dashed,
    },
    defensearrow/.style={
        -{Latex[length=3.2mm,width=2.2mm]},
        line width=0.9pt,
        draw=blue!70!black,
        dashed,
    },
    evalarrow/.style={
        -{Latex[length=3.2mm,width=2.2mm]},
        line width=0.9pt,
        draw=green!45!black,
        dashed,
    }
]

\node[pipebox] (raw) at (0,0) {Raw data\\$\Omega$};
\node[pipebox] (feat) at (2.55,0) {Feature\\vectors};
\node[pipebox] (sim) at (5.10,0) {Clean\\similarity\\$S$};
\node[corebox] (summ) at (8.00,0) {Continuous\\summarizer\\decision $\mathbf{x}$};
\node[pipebox] (summary) at (10.7,0) {Selected\\summary};
\node[evalbox] (downstream) at (13.45,0) {Downstream\\task};

\draw[dataarrow] (raw.east) -- (feat.west);
\draw[dataarrow] (feat.east) -- (sim.west);
\draw[dataarrow] (sim.east) -- (summ.west);
\draw[dataarrow] (summ.east) -- (summary.west);
\draw[dataarrow] (summary.east) -- (downstream.west);

\node[attackbox] (adv) at (0,2.35) {White-box\\adversary};
\node[attackbox] (query) at (3.15,2.35) {Queries\\values/gradients};
\node[attackbox, text width=2.35cm] (pert) at (6.10,2.35) {Admissible\\perturbation\\$\mathbf{v}\in\mathcal{P}_{p}(\epsilon_p)$};
\node[attackbox] (pertsim) at (9.05,2.35) {Perturbed\\similarity\\$S+\mathbf{v}$};

\node[notebox, text width=2.45cm] (valid) at (12.00,2.8) {Validity\\$0\le S+\mathbf{v}\le 1$};
\node[notebox, text width=2.45cm] (nodirect) at (12.00,1.6) {No direct edit of\\$\mathbf{x}$ or summary};

\draw[attackarrow] (adv.east) -- (query.west);
\draw[attackarrow] (query.east) -- (pert.west);
\draw[attackarrow] (pert.east) -- (pertsim.west);
\draw[attackarrow] (valid.south) -- (nodirect.north);

\draw[attackarrow] 
    (pertsim.south) 
   -- ++(0,-0.70)
    -| 
    (summ.north);
    
\draw[attackarrow] 
     (pertsim.north) 
     -- ++(0,0.7)
     -|
     (valid.north);

\draw[attackarrow] 
    (query.south) 
    -- ++(0,-0.7)
    -| 
    (summ.north);

\node[notebox, text width=2.95cm] (multi) at (1.45,-3.45)
{Multi-model setting:\\one perturbation affects\\multiple target summarizers};

\draw[attackarrow]
    (pert.south)
    -- ++(0,-0.35)
    -| 
    (multi.north);

\node[defensebox, text width=2.30cm] (sets) at (3.10,-2.15) {Mixed attack\\sets\\$\{\mathcal{P}_{j}\}_{j=1}^{J}$};
\node[defensebox, text width=2.55cm] (robust) at (6.25,-2.15) {Max--min\\robust solver\\$\max_{\mathbf{x}}\min_{\mathbf{w},\mathbf{v}^{j}}G$};
\node[defensebox, text width=2.30cm] (robx) at (9.4,-2.15) {Robust\\summary decision\\$\mathbf{x}^{\rm rob}$};

\draw[defensearrow] (sets.east) -- (robust.west);
\draw[defensearrow] (robust.east) -- (robx.west);

\draw[defensearrow] 
    (robx.north) 
    -- ++(0,0.65)
    -| 
    (summ.south);

\node[evalbox, text width=2.25cm] (clean_eval) at (10.5,-3.55) {Clean-objective\\evaluation};
\node[evalbox, text width=2.25cm] (task_eval) at (13.45,-3.55) {Task-level\\reliability};

\draw[evalarrow] (downstream.south) -- (task_eval.north);
\draw[evalarrow] (clean_eval.east) -- (task_eval.west);

\draw[evalarrow] 
    (summary.south) 
    -- ++(0,-0.20)
    -- ++(0.85,0)
    -- ++(0,-2.2)
    -- ++(-0.85,0)
    -- 
    (clean_eval.north);
\end{tikzpicture}
\caption{Threat model and attack surface for similarity-based continuous summarization. 
The clean pipeline constructs a similarity matrix \(S\), optimizes a continuous summarization decision \(\mathbf{x}\), and passes the selected summary to downstream tasks. 
The adversary has white-box query access and perturbs the upstream similarity structure through an admissible perturbation \(\mathbf{v}\in\mathcal{P}_{p}(\epsilon_p)\), yielding \(S+\mathbf{v}\), but does not directly edit \(\mathbf{x}\) or the final summary. 
The defender solves a max--min robust summarization problem over mixed perturbation sets \(\{\mathcal{P}_j\}_{j=1}^{J}\). 
Both summarization utility and downstream task reliability are evaluated to connect optimization-level degradation with trustworthy-AI pipeline reliability.}
\label{fig:threat_model_attack_surface}
\end{figure*}
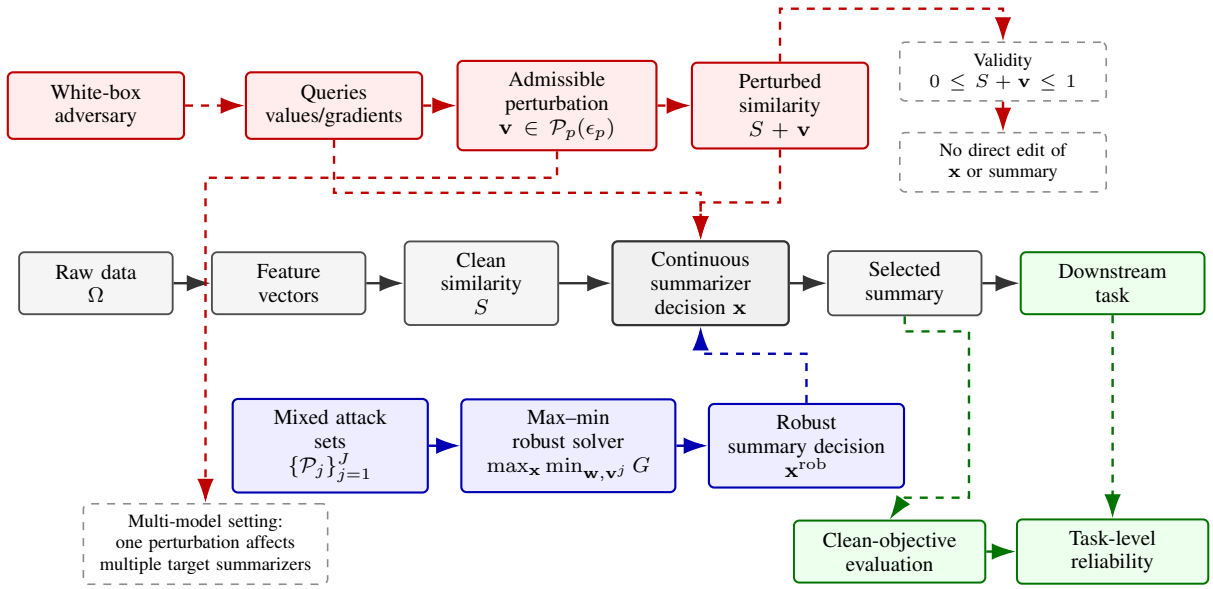

\subsection{Attack and Defense Formulations}
\label{subsec:attack_defense_formulations}

Based on the threat model above, we now formalize the attack and defense problems. On the attack side, the adversary searches for a bounded perturbation of the similarity structure that causes the summarization algorithm to output lower-quality solutions under the clean objective. On the defense side, the model owner seeks a summarization solution that remains stable and useful under mixed perturbation types.

\subsubsection{Multi-Target Attack Formulation}

Suppose there are $I$ target summarization models. To generate a single perturbation that degrades all target models simultaneously, we consider
\[
\min_{\mathbf{v}\in\mathcal{P},\,t} t
\qquad
\text{s.t.}\quad
\max_{\mathbf{x}^{i}\in\mathcal{X}_{i}}F_{i}(\mathbf{v},\mathbf{x}^{i})\le t,
\quad \forall i\in[I],
\]
where $\mathbf{v}\in\mathcal{P}$ is the perturbation applied to the similarity structure and $\mathbf{x}^{i}\in\mathcal{X}_{i}$ is the summarization decision for the $i$-th model. To obtain a tractable formulation, we introduce the weight vector $\mathbf{w}$ and define the min-max convex-submodular problem as
\begin{equation}
\label{eq:attck_problem}
\min_{\mathbf{v}\in\mathcal{P}}
\max_{\mathbf{w}\in\mathcal{W},\,\mathbf{x}^{i}\in\mathcal{X}_{i}}
\phi(\mathbf{v},\mathbf{w},\{\mathbf{x}^{i}\}),
\end{equation}
where $\phi(\mathbf{v},\mathbf{w},\{\mathbf{x}^{i}\})=
\sum_{i=1}^{I}\mathbf{w}_{i}F_{i}(\mathbf{v},\mathbf{x}^{i})$, and $\mathcal{W}=
\left\{
\mathbf{w}\ \middle|\ \mathbf{1}^{T}\mathbf{w}=1,\ \mathbf{w}_{i}\in[0,1],\ \forall i
\right\}$. Here, $F_i(\mathbf{v},\mathbf{x}^{i})$ denotes the objective value of the $i$-th attacked summarization model defined on the perturbed dataset $\Omega_i(\mathbf{v})$. We use the following approximation notion for the attack problem. We say that $\bar{\mathbf v}$ is an $(\alpha,\epsilon)$-approximation min-max solution of~\eqref{eq:attck_problem} if
\[
\alpha
\max_{\mathbf{w}\in\mathcal{W},\,\mathbf{x}^{i}\in\mathcal{X}_{i}}
\phi(\bar{\mathbf{v}},\mathbf{w},\{\mathbf{x}^{i}\})
\le
{\rm OPT}_{\rm minmax}+\epsilon,
\]
where ${\rm OPT}_{\rm minmax}=
\min_{\mathbf{v}\in\mathcal{P}}
\max_{\mathbf{w}\in\mathcal{W},\,\mathbf{x}^{i}\in\mathcal{X}_{i}}
\phi(\mathbf{v},\mathbf{w},\{\mathbf{x}^{i}\}).$

\subsubsection{Robust Defense Formulation}

We now turn to the defender's perspective. Since the model owner may face multiple possible attack types and may not know in advance which one will occur, we seek a robust solution that performs well against mixed attacks. Let $\mathbf{v}^{j}\in\mathcal{P}_{j}$ denote the perturbation associated with the $j$-th attack type, where $j\in[J]$. We then consider the following regularized max-min submodular-convex formulation:
\begin{equation}
\label{eq:robust_problem}
\max_{\mathbf{x}\in\mathcal{X}}
\min_{\mathbf{w}\in\mathcal{W},\,\mathbf{v}^{j}\in\mathcal{P}_{j}}
G(\mathbf{x},\mathbf{w},\{\mathbf{v}^{j}\}),
\end{equation}
where
\begin{eqnarray*}
G(\mathbf{x},\mathbf{w},\{\mathbf{v}^j\})
&=&
\sum_{j=1}^J w_j F(\mathbf{x},\mathbf{v}^j)\\
&&+
\frac{\lambda_v}{2}\sum_{j=1}^J \|\mathbf{v}^j\|_2^2
+
\frac{\gamma}{2}\left\|\mathbf{w}-\frac{1}{J}\mathbf{1}\right\|_2^2.
\end{eqnarray*}

Here, for perturbations under which the perturbed objective remains in the same structural class, $F(\mathbf{x},\mathbf{v}^{j})$ is non-negative, $m$-weakly monotone, and DR-submodular with respect to $\mathbf{x}$, while $\mathbf{w}$ lies in the simplex set
\[
\mathcal{W}
=
\left\{
\mathbf{w}\ \middle|\ \mathbf{1}^{T}\mathbf{w}=1,\ w_{j}\in[0,1],\ \forall j\in[J]
\right\}.
\]
The regularization parameter $\lambda_v>0$ is introduced to stabilize the inner minimization with respect to the perturbation variables and to induce strong convexity in $\mathbf{v}^{j}$. The regularization parameter $\gamma>0$ is introduced to avoid over-concentration on any single attack type and to promote balanced robustness across mixed attacks.

We use the following approximation notion for the robust defense problem. We say that $\bar{\mathbf{x}}$ is a $(\beta,\epsilon)$-approximation max-min solution of~\eqref{eq:robust_problem} if
\[
\min_{\mathbf{w}\in\mathcal{W},\,\mathbf{v}^{j}\in\mathcal{P}_{j}}
G(\bar{\mathbf{x}},\mathbf{w},\{\mathbf{v}^{j}\})
\ge
\beta\,{\rm OPT}_{\rm maxmin}-\epsilon,
\]
where ${\rm OPT}_{\rm maxmin}=
\max_{\mathbf{x}\in\mathcal{X}}
\min_{\mathbf{w}\in\mathcal{W},\,\mathbf{v}^{j}\in\mathcal{P}_{j}}
G(\mathbf{x},\mathbf{w},\{\mathbf{v}^{j}\})$.

\section{Adversarial Attack Generation against Multiple Summarization Models}\label{sec:attack_gener}
In this section, we study multi-target adversarial attack generation for continuous summarization models, with the goal of characterizing adversarial vulnerability in an upstream data summarization component of trustworthy AI pipelines. We first present a greedy maximization subroutine for $m$-weakly monotone DR-submodular objectives, and then build a multi-target attack algorithm on top of it. Detailed proofs are deferred to the Appendix.
\subsection{Greedy Maximization Subroutine}\label{subsec:appro}

For a fixed perturbation $\mathbf{v}$, we consider the problem
\[
\max_{\mathbf{x}\in \mathcal{X}}F(\mathbf{v},\mathbf{x}),
\]
where $\mathcal{X}$ is a down-closed convex set with diameter $D$. To solve this problem, we introduce a continuous greedy-type procedure, denoted by $\mathcal{M}_{\rm greedy}$. Since $F(\mathbf{v},\mathbf{x})$ is only $m$-weakly monotone rather than fully monotone, negative gradient coordinates may decrease the objective value. We therefore use the clipped gradient
\[
g(\mathbf{x}_{t}) := \nabla_{\mathbf{x}}F(\mathbf{v},\mathbf{x}_{t})\vee \mathbf{0},
\]
which removes descent directions. As a result, the update direction only keeps coordinates corresponding to nonnegative partial derivatives, and the following identity holds:
\begin{equation}
\label{eq:pre_equ}
\left\langle \nabla_{\mathbf{x}} F(\mathbf{v},\mathbf{x}_{t}),\mathbf{x}_{t+1}-\mathbf{x}_{t}\right\rangle
=
\left\langle g(\mathbf{x}_{t}),\mathbf{x}_{t+1}-\mathbf{x}_{t}\right\rangle.
\end{equation}
This relation is central to the approximation analysis.

\begin{algorithm}[t!]
\small
\caption{General Continuous Greedy Algorithm $\mathcal{M}_{\rm greedy}$}
\begin{algorithmic}[1] \label{alg:Mgreedy}
    \renewcommand{\algorithmicrequire}{\textbf{Input:}}
    \REQUIRE $F(\mathbf{v},\mathbf{x}), \mathcal{X}, T$, and $\mathbf{x}_{0}=\mathbf{0}$
    \FOR {$t=0,\ldots,T-1$}
        \STATE Compute
        \[
        \bar{\mathbf{d}}_{t}
        =
        \arg \max_{\mathbf{d}\in \mathcal{X}}
        \langle g(\mathbf{x}_{t}),\mathbf{d} \rangle,
        \qquad
        g(\mathbf{x}_{t})=\nabla_{\mathbf{x}}F(\mathbf{v},\mathbf{x}_{t})\vee \mathbf{0}
        \]
        \STATE Direction clipping:
        \[
        [\mathbf{d}_{t}]_{s}
        =
        \begin{cases}
        [\bar{\mathbf{d}}_{t}]_{s}, & s\in \{s:[\nabla F(\mathbf{v},\mathbf{x}_{t})]_{s}\ge 0\},\\
        0, & \text{otherwise}
        \end{cases}
        \]
        \STATE $\mathbf{x}_{t+1}=\mathbf{x}_{t}+\frac{1}{T}\mathbf{d}_{t}$
    \ENDFOR
    \STATE Return $\mathbf{x}_{T}$
\end{algorithmic}
\end{algorithm}

We can now state the approximation guarantee of $\mathcal{M}_{\rm greedy}$.

\begin{theorem} [Proof in Sec.~\ref{supp:proofs}]
\label{thm:approximation}
Let $F(\mathbf{v},\mathbf{x})$ be a nonnegative, $L$-smooth, DR-submodular function that is also $m$-weakly monotone. Then the output $\mathbf{x}_{T}$ of Alg.~\ref{alg:Mgreedy} satisfies
\[
F(\mathbf{v},\mathbf{x}_{T})
\ge
m(1-1/e)\max_{\mathbf{x}\in\mathcal{X}}F(\mathbf{v},\mathbf{x})-\epsilon
\]
after $\mathcal{O}(LD^{2}/\epsilon)$ iterations.
\end{theorem}

Assuming the availability of gradient evaluations and a linear maximization oracle (LMO), we obtain the following oracle complexity.

\begin{corollary}[Oracle complexity of $\mathcal{M}_{\rm greedy}$]
\label{cor:complexity-MGreedy}
Under the same assumptions as Theorem~\ref{thm:approximation}, Alg.~\ref{alg:Mgreedy} achieves an $(m(1-e^{-1}),\epsilon)$-approximation using $\mathcal{O}(\epsilon^{-1})$ gradient evaluations and $\mathcal{O}(\epsilon^{-1})$ LMO calls.
\end{corollary}

\subsection{Multi-Target Attack Algorithm}

We now build a multi-target attack algorithm on top of $\mathcal{M}_{\rm greedy}$. At each outer iteration, the algorithm approximately solves the inner maximization over $\{\mathbf{x}^i\}$ using $\mathcal{M}_{\rm greedy}$, updates the model weights $\mathbf{w}$ by maximizing the current aggregated objective, and then performs a projected gradient step on the perturbation variable $\mathbf{v}$.

\begin{algorithm}[t!]
\small
\caption{Minmax Convex-Submodular Approximation Algorithm}
\begin{algorithmic}[1]\label{alg:attack}
\renewcommand{\algorithmicrequire}{\textbf{Input:}}
\REQUIRE $\mathbf{v}_{0}, \mathbf{w}_{0}$, iteration number $K$, subroutine $\mathcal{M}_{\rm greedy}$, and constraint sets $\mathcal{P},\mathcal{W},\mathcal{X}_{i}$
    \FOR {$k=0, \ldots, K-1$}
        \STATE For $i=1,\ldots,I$, run
        \[
        \mathbf{x}^{i}_{k+1}
        =
        \mathcal{M}_{\rm greedy}(F_{i}(\mathbf{v}_{k},\mathbf{x}^{i}),\mathcal{X}_{i},T)
        \]
        \STATE Update \ 
        $
        \mathbf{w}_{k+1}
        =
        \arg\max_{\mathbf{w}\in\mathcal{W}}
        \phi(\mathbf{v}_{k},\mathbf{w},\{\mathbf{x}^{i}_{k+1}\})
        $
        \STATE Update
        \[
        \mathbf{v}_{k+1}
        =
        {\rm proj}_{\mathcal{P}}
        \Bigl(
        \mathbf{v}_{k}
        -
        \eta \nabla_{\mathbf{v}}
        \phi(\mathbf{v}_{k},\mathbf{w}_{k+1},\{\mathbf{x}^{i}_{k+1}\})
        \Bigr)
        \]
    \ENDFOR
    \STATE Return $\bar{\mathbf{v}}=\frac{1}{K}\sum_{k=1}^{K}\mathbf{v}_{k}$
\end{algorithmic}
\end{algorithm}

For each iterate $\mathbf{v}_{k}$, Alg.~\ref{alg:attack} first computes approximate maximizers $\{\mathbf{x}^{i}_{k+1}\}$ for all target models, then updates the weight vector $\mathbf{w}_{k+1}$ to emphasize the currently worst-performing models, and finally updates the perturbation variable $\mathbf{v}_{k+1}$ by a projected gradient step over $\mathcal{P}$. Here the Euclidean projection operator is defined by
\[
{\rm proj}_{\mathcal{P}}(\mathbf{a})
=
\arg\min_{\mathbf{x}\in\mathcal{P}}\|\mathbf{x}-\mathbf{a}\|_{2}^{2}.
\]

We impose the following regularity conditions for the attack analysis.
\begin{assumption}
\label{ass:attack}
For each target model, $F_i(\mathbf{v},\mathbf{x}^i)$ is differentiable, $M$-Lipschitz continuous in $(\mathbf{v},\mathbf{x}^i)$, $L$-smooth with respect to $\mathbf{x}^i$, and convex with respect to $\mathbf{v}$. Moreover, each $\mathcal{X}_i$ is a down-closed convex set with diameter $D_{\mathcal{X}_i}$, and $\mathcal{P}$ is a compact convex set with diameter $D_{\mathcal{P}}$.
\end{assumption}

\subsection{Approximation Guarantee and Oracle Complexity}

The following guarantee is conditional on convexity of the attacked objective with respect to the variable $\mathbf{v}$ for Alg.~\ref{alg:attack}.

\begin{theorem} [Proof in Sec.~\ref{supp:proofs}]
\label{thm:attack}
For adversarial sample generation across multiple smooth DR-submodular maximization models that are $m$-weakly monotone, under Assumption~\ref{ass:attack} and with $\eta = 1/\sqrt{K}$, the output $\bar{\mathbf{v}}$ of Alg.~\ref{alg:attack} with subroutine $\mathcal{M}_{\rm greedy}$ satisfies
\[
m(1-1/e)
\max_{\mathbf{w}\in\mathcal{W},\,\mathbf{x}^{i}\in\mathcal{X}_{i}}
\phi(\bar{\mathbf{v}},\mathbf{w},\{\mathbf{x}^{i}\})
\le
{\rm OPT}_{\rm minmax}+\epsilon
\]
after $\mathcal{O}(D_{\mathcal{P}}/\epsilon^{2})$ outer iterations in Alg.~\ref{alg:attack} and $\mathcal{O}(LD^{2}/\epsilon)$ inner iterations in $\mathcal{M}_{\rm greedy}$.
\end{theorem}

Assuming the availability of gradient evaluations, an LMO, and a projection oracle over convex sets, we obtain the following complexity bound.

\begin{corollary}[Oracle complexity]
\label{cor:attack}
Under the same conditions as Theorem~\ref{thm:attack}, Alg.~\ref{alg:attack} achieves an $(m(1-1/e),\epsilon)$-approximation min-max solution with $\mathcal{O}(I\epsilon^{-3}+\epsilon^{-2})$ gradient evaluations, $\mathcal{O}(I\epsilon^{-3}+\epsilon^{-2})$ LMO calls, and $\mathcal{O}(\epsilon^{-2})$ projection-oracle calls.
\end{corollary}

The guarantee simplifies in the single-model monotone case.

\begin{remark}
When $I=1$ and $m=1$, the weight variable disappears and the formulation reduces to
\[
\min_{\mathbf{v}\in\mathcal{P}}\max_{\mathbf{x}\in\mathcal{X}}F(\mathbf{v},\mathbf{x}).
\]
In this special case, our method yields a $((1-1/e),\epsilon)$-approximation guarantee with $\mathcal{O}(\epsilon^{-3})$ gradient evaluations, $\mathcal{O}(\epsilon^{-3})$ LMO calls, and $\mathcal{O}(\epsilon^{-2})$ projection-oracle calls.
\end{remark}

\subsection{Proof Sketch}

The guarantee of Theorem~\ref{thm:attack} follows from combining an approximate inner maximization step with a projected outer update on the perturbation variable. First, by Theorem~\ref{thm:approximation}, the subroutine $\mathcal{M}_{\rm greedy}$ provides an $m(1-1/e)$-approximate solution for each inner DR-submodular maximization problem. Second, the projected gradient update on $\mathbf{v}$ yields a descent-type inequality for the min-max objective through the convexity of $\phi(\cdot,\mathbf{w},\{\mathbf{x}^{i}\})$ and the projection step.

Specifically, the gradient of $\phi$ with respect to $\mathbf{v}$ is
\[
\nabla_{\mathbf{v}}\phi(\mathbf{v},\mathbf{w},\{\mathbf{x}^{i}\})
=
\sum_{i=1}^{I}\mathbf{w}_{i}\nabla_{\mathbf{v}}F_{i}(\mathbf{v},\mathbf{x}^{i}),
\]
and Assumption~\ref{ass:attack} implies the bound
\begin{equation}
\label{eq:attack_bound_phi}
\|\nabla_{\mathbf{v}}\phi(\mathbf{v},\mathbf{w},\{\mathbf{x}^{i}\})\|^{2}
\le
M^{2},
\qquad
\forall\,\mathbf{w}\in\mathcal{W}.
\end{equation}

To quantify one-step progress, define
\[
\tau_{\mathbf{v}}
=
\phi(\mathbf{v},\mathbf{w}_{k+1},\{\mathbf{x}^{i}_{k+1}\})
-
\phi(\mathbf{v}_{k},\mathbf{w}_{k+1},\{\mathbf{x}^{i}_{k+1}\}).
\]
The key step is the following inequality.

\begin{lemma}
\label{lem:attack_key}
Under Assumption~\ref{ass:attack},
\begin{equation}
\label{eq:attack_lem}
\tau_{\mathbf{v}}
\ge
\frac{\|\mathbf{v}_{k+1}-\mathbf{v}\|^{2}-\|\mathbf{v}_{k}-\mathbf{v}\|^{2}}{2\eta}
-
\frac{\eta M^{2}}{2}.
\end{equation}
\end{lemma}

Summing~\eqref{eq:attack_lem} over the outer iterations, combining it with the approximation guarantee of $\mathcal{M}_{\rm greedy}$ for the inner maximization, and then using the convexity of $\phi(\cdot,\mathbf{w},\{\mathbf{x}^{i}\})$ yields Theorem~\ref{thm:attack}. The full proofs of Lemma~\ref{lem:attack_key} and Theorem~\ref{thm:attack} are deferred to the Appendix.

\section{Robust Defense under Mixed Attacks}\label{sec:defense}

In this section, we study robust protection against adversarial perturbations in continuous data summarization under mixed attack types. Our goal is to compute a summarization solution that preserves utility while remaining stable against multiple possible attack constraints. We formulate this protection task as a robust submodular maximization problem and develop a corresponding approximation algorithm. Detailed proofs are deferred to the Appendix.

\subsection{Robust Continuous Greedy Algorithm}

We propose a robust continuous greedy algorithm for maximizing an $m$-weakly monotone DR-submodular objective under mixed attacks. At each outer iteration, the algorithm performs three steps: it first approximately updates each adversarial variable $\mathbf{v}^{j}$ by projected gradient descent, then computes the worst-case mixture weights $\mathbf{w}$ by solving the inner minimization over $\mathcal{W}$, and finally updates the primal variable $\mathbf{x}$ through a clipped-gradient continuous greedy step.

More specifically, for a fixed iterate $\mathbf{x}_{k}$, we compute $
\mathbf{v}^{j}_{T}(\mathbf{x}_{k})$
by projected gradient descent on $G(\mathbf{x}_{k},\mathbf{w},\cdot)$ for each attack type $j\in[J]$, the gradient with respect to $\mathbf v^j$ is
\[
\nabla_{\mathbf{v}^{j}}G
=
w_{t,j}\nabla_{\mathbf{v}^{j}}F(\mathbf{x}_{k},\mathbf{v}^{j}_{t})
+
\lambda_{\mathbf v}\mathbf{v}^{j}_{t}.
\]
and define
\[
\mathbf{w}^{*}(\{\mathbf{v}^{j}_{T}\}_{k})
=
\arg\min_{\mathbf{w}\in \mathcal{W}}
G(\mathbf{x}_{k},\mathbf{w},\{\mathbf{v}^{j}_{T}(\mathbf{x}_{k})\}).
\]
We then use the clipped gradient
\[
[\nabla_{\mathbf{x}}G(\mathbf{x}_{k})]_{+}
=
\nabla_{\mathbf{x}}G\!\left(
\mathbf{x}_{k},
\mathbf{w}^{*}(\{\mathbf{v}^{j}_{T}\}_{k}),
\{\mathbf{v}^{j}_{T}(\mathbf{x}_{k})\}
\right)\vee \mathbf{0},
\]
to define the ascent direction. As in the attack-side greedy routine, clipping removes directions corresponding to negative gradients. Consequently,
\begin{equation}
\label{eq:pre_defense}
\left\langle
\nabla_{\mathbf{x}}G\!\left(
\mathbf{x}_{k},
\mathbf{w}^{*}(\{\mathbf{v}^{j}_{T}\}_{k}),
\{\mathbf{v}^{j}_{T}(\mathbf{x}_{k})\}
\right),
\mathbf{d}_{k}
\right\rangle
=
\left\langle
[\nabla_{\mathbf{x}}G(\mathbf{x}_{k})]_{+},
\mathbf{d}_{k}
\right\rangle.
\end{equation}

The complete procedure is given below.

\begin{algorithm}[t!]
\small
\caption{Robust Continuous Greedy Algorithm}
\begin{algorithmic}[1]\label{alg:robust}
    \renewcommand{\algorithmicrequire}{\textbf{Input:}}
    \REQUIRE $\mathbf{v}_{0}, \mathbf{w}_{0}, \eta, K, T, \mathcal{P}_{j}, \mathcal{X}$
    \FOR {$k=0,\ldots,K-1$}
        \STATE Initialize $\mathbf{w}_{0}(\mathbf{v}_{k})=\mathbf{w}_{0}$
        \FOR {$t=0,\ldots,T-1$}
            \STATE For each \(j\in[J]\), compute
            \begin{eqnarray*}
            \mathbf{v}^{j}_{t+1}(\mathbf{x}_{k})&=&{\rm proj}_{\mathcal{P}_{j}}\left(\mathbf{v}^{j}_{t}(\mathbf{x}_{k})\right.\\
            &&-\left.\eta\nabla_{\mathbf{v}^{j}}G\left(\mathbf{x}_{k},\mathbf{w}_{t},\{\mathbf{v}^{\ell}_{t}(\mathbf{x}_{k})\}_{\ell=1}^{J}\right)\right).
            \end{eqnarray*}
        \ENDFOR
        \STATE Compute
        \[
        \mathbf{w}^{*}(\{\mathbf{v}^{j}_{T}\}_{k})
        =
        \arg\min_{\mathbf{w}\in \mathcal{W}}
        G(\mathbf{x}_{k}, \mathbf{w}, \{\mathbf{v}^{j}_{T}(\mathbf{x}_{k})\})
        \]
        and
        \[
        \bar{\mathbf{d}}_{k}
        =
        \arg\max_{\mathbf{d}\in \mathcal{X}}
        \left\langle
        [\nabla_{\mathbf{x}}G(\mathbf{x}_{k})]_{+},
        \mathbf{d}
        \right\rangle
        \]
        \STATE Direction clipping:
        \[
        [\mathbf{d}_{k}]_{s}
        =
        \begin{cases}
        [\bar{\mathbf{d}}_{k}]_{s}, & s\in \mathcal{S},\\
        0, & \text{otherwise},
        \end{cases}
        \]
        where
        \[
        \mathcal{S}
        =
        \left\{
        s:
        \left[
        \nabla_{\mathbf{x}}G\!\left(
        \mathbf{x}_{k},
        \mathbf{w}^{*}(\{\mathbf{v}^{j}_{T}\}_{k}),
        \{\mathbf{v}^{j}_{T}(\mathbf{x}_{k})\}
        \right)
        \right]_{s}\ge 0
        \right\}
        \]
        \STATE Set $\mathbf{x}_{k+1}=\mathbf{x}_{k}+\frac{1}{K}\mathbf{d}_{k}$
    \ENDFOR
    \STATE Return $\mathbf{x}_{K}$
\end{algorithmic}
\end{algorithm}

We impose the following regularity conditions for the robust defense analysis.

\begin{assumption}
\label{ass:robust}
Consider the regularized robust model~\eqref{eq:robust_problem}. For perturbations under which the perturbed objective remains in the same structural class, each function $F(\mathbf{x},\mathbf{v}^j)$ is non-negative, $m$-weakly monotone, and DR-submodular with respect to $\mathbf{x}$, and is jointly $L$-smooth with respect to $(\mathbf{x},\mathbf{v}^j)$. Moreover, each perturbation set $\mathcal{P}_j$ is compact and convex, and $\mathcal{X}\subseteq \mathbb{R}_+^n$ is a down-closed convex set with diameter $D_{\mathcal X}$.

For each fixed $\mathbf{x}$ and $\mathbf{w}$, the inner objective
\[
\sum_{j=1}^J w_j F(\mathbf{x},\mathbf{v}^j)
+
\frac{\lambda_v}{2}\sum_{j=1}^J \|\mathbf{v}^j\|_2^2
\]
is $\mu$-strongly convex with respect to $\{\mathbf{v}^j\}_{j=1}^J$, where the strong convexity is induced by the quadratic regularization term.
\end{assumption}

Under Assumption~\ref{ass:robust}, we can establish the following approximation guarantee and oracle complexity for Alg.~\ref{alg:robust}.

\subsection{Approximation Guarantee and Oracle Complexity}

We now state the main guarantee of Alg.~\ref{alg:robust}.

\begin{theorem}[Proof in Appendix]
\label{thm:robust}
Under Assumption~\ref{ass:robust}, let the inner projected-gradient step size be $\eta=1/L$ and define $\rho=1-\mu/L \le 1$. For any $\epsilon\in(0,1)$, if
\[
K\ge \mathcal{O}(\epsilon^{-1})
\qquad \text{and} \qquad
T\ge \mathcal{O}(\log_{\rho}(\epsilon^{-1})),
\]
then the output $\mathbf{x}_{K}$ of Alg.~\ref{alg:robust} is an $\left(m(1-1/e),\epsilon\right)$-approximation max-min solution to~\eqref{eq:robust_problem}. In particular,
\[
\min_{\mathbf{w}\in \mathcal{W},\,\mathbf{v}^{j}\in \mathcal{P}_{j}}
G(\mathbf{x}_{K},\mathbf{w},\{\mathbf{v}^{j}\})
\ge
m(1-1/e)\,{\rm OPT}_{\rm maxmin}-\epsilon,
\]
where 
$
{\rm OPT}_{\rm maxmin}=
\max_{\mathbf{x}\in \mathcal{X}}
\min_{\mathbf{w}\in \mathcal{W},\,\mathbf{v}^{j}\in \mathcal{P}_{j}}
G(\mathbf{x},\mathbf{w},\{\mathbf{v}^{j}\}).
$
\end{theorem}

Assuming the availability of gradient evaluations, an LMO, and a projection oracle over convex sets, we obtain the following complexity bound.

\begin{corollary}[Oracle complexity]
\label{cor:robust}
Under the same conditions as Theorem~\ref{thm:robust}, Alg.~\ref{alg:robust} produces an $(m(1-1/e),\epsilon)$-approximation max-min solution using
\[
\mathcal{O}\!\left(J\epsilon^{-1}\log_{\rho}\epsilon^{-1}+\epsilon^{-1}\right)
\]
gradient evaluations, $\mathcal{O}(\epsilon^{-1})$ LMO calls, and
\[
\mathcal{O}\!\left(J\epsilon^{-1}\log_{\rho}\epsilon^{-1}+\epsilon^{-1}\right)
\]
projection-oracle calls.
\end{corollary}

\begin{remark}
For the special case $J=1$ and $m=1$, our method provides a $(1-1/e)$-approximation guarantee for robust monotone DR-submodular maximization. Thus, it recovers the best-known approximation ratio in the single-attack setting, while additionally accommodating mixed attack types in the more general framework studied here.
\end{remark}

\subsection{Proof Sketch}

We briefly outline why Alg.~\ref{alg:robust} achieves a robust approximation guarantee under mixed attacks.

First, under Assumption~\ref{ass:robust}, the inner projected-gradient updates on $\mathbf{v}^{j}$ contract toward the corresponding minimizers of the regularized inner problem. In particular, if $\rho=1-\mu/L$,
then
$$
\|\mathbf{v}^{j}_{T}(\mathbf{x}_{k})-{\mathbf{v}^{j}}^{*}(\mathbf{x}_{k})\|
\le
\rho^{T/2}D_{\mathcal{P}_{j}}.
$$
Second, define the value function
\begin{equation}
\label{eq:defense_phi}
\Phi(\mathbf{x})
:=
\min_{\mathbf{w}\in \mathcal{W},\,\mathbf{v}^{j}\in \mathcal{P}_{j}}
G(\mathbf{x},\mathbf{w},\{\mathbf{v}^{j}\}).
\end{equation}
Using standard sensitivity arguments for strongly convex inner problems, one can show that $\Phi$ is smooth and that the minimizer $\mathbf{w}^{*}$ varies Lipschitz-continuously with the adversarial variables. Third, the continuous greedy update on $\mathbf{x}$ yields a one-step improvement inequality of the form
\begin{eqnarray*}
&&\Phi(\mathbf{x}_{k+1})-\Phi(\mathbf{x}_{k})\\
&&\ge
\frac{1}{K}
\left(
m\Phi(\mathbf{x}^{*})-\Phi(\mathbf{x}_{k})
-2D_{\mathcal{X}}\|\Delta_{k}\|
-\frac{L_{\Phi}}{2K}D_{\mathcal{X}}^{2}
\right),
\end{eqnarray*}
where $\Delta_{k}=\nabla \Phi(\mathbf{x}_{k})-
\nabla_{\mathbf{x}}G\!\left(
\mathbf{x}_{k},
\mathbf{w}^{*}(\{\mathbf{v}^{j}_{T}\}_{k}),
\{\mathbf{v}^{j}_{T}(\mathbf{x}_{k})\}
\right)$. Bounding $\|\Delta_{k}\|$ by the contraction error of the inner $\mathbf{v}^{j}$-updates and telescoping the above inequality over $k=0,\dots,K-1$ yields Theorem~\ref{thm:robust}. The detailed proofs of the auxiliary lemmas and the complete proof of Theorem~\ref{thm:robust} are deferred to the Appendix.

\section{Experimental Evaluation}\label{sec:experi}

We organize the experimental evaluation around two complementary theory-aligned settings. 
First, we evaluate the proposed attack and defense algorithms on the multilinear-extension summarization model constructed from real image datasets, which is directly consistent with the DR-submodular analysis in Sec.~\ref{sec:attack_gener} and Sec.~\ref{sec:defense}. 
Second, we introduce a controlled clustered multi-target benchmark to isolate the attack-defense mechanism under a known representative structure. 
The first setting is used to evaluate utility degradation and robustness on real data, while the second is used to verify whether a single shared perturbation can consistently disrupt representative selection across multiple related victim models. 
For completeness, additional empirical results on the direct continuous-score objective are reported in Appendix~\ref{supp:Experimental Evaluation}, since that objective is not the primary model covered by our theoretical guarantees.

\subsection{Experimental Setup}

\subsubsection{Real-data multilinear-extension summarization}

We conduct experiments on CIFAR-10~\cite{alex2009}, MNIST~\cite{lecun2021mnist}, and Fashion-MNIST~\cite{xiao2017fashionmnist}. 
Unless otherwise specified, the main numerical results are reported on CIFAR-10, while MNIST and Fashion-MNIST are used for additional validation. 
In each experiment, we construct $I=10$ victim summarization models, and each model contains $n=50$ images sampled from the raw dataset. 
To evaluate the effect of model size, we also report additional results for $n\in\{100,200\}$ in the Appendix.

\subsubsection{Controlled clustered multi-target benchmark}

To further isolate the attack-defense mechanism, we construct a controlled clustered multi-target benchmark. 
This benchmark is designed to capture a multi-resolution image summarization scenario in which visually similar images form semantic clusters and each cluster contains several near-tie representative images. 
Unless otherwise specified, we construct $I=10$ victim models, each containing $n=50$ items organized into $5$ clusters with $10$ items per cluster. 
The similarity matrix of each victim model is normalized to $[0,1]$. 
Inter-cluster similarities are sampled from a low range, while intra-cluster similarities are sampled from a moderate range. 
Within each cluster, two representative candidates are assigned near-tie similarity scores to the remaining items. 
In the main benchmark, inter-cluster similarities are sampled from $[0.05,0.12]$, ordinary intra-cluster similarities from $[0.44,0.56]$, hub-representative similarities from $[0.76,0.81]$, and runner-up representative similarities from $[0.75,0.80]$.

All victim models share the same clustered representative structure, while small symmetric noise is added to the similarity matrices across models. 
This shared-but-nonidentical construction is designed to reflect the multi-target attack setting: the attacker must generate a single perturbation that degrades several related summarization models simultaneously. 
Moreover, the near-tie representative design prevents the attack from being dominated by a single obvious hub item, making the benchmark suitable for evaluating structure-aware attacks under the multilinear-extension objective.

\subsubsection{Feasible set and summary budgets}

Following the multilinear-extension formulation, each coordinate $x_\nu\in[0,1]$ is interpreted as the probability or soft weight of selecting an item. 
Thus, we use the cardinality-budget relaxation
\[
\mathcal{X}_k
=
\left\{
\mathbf{x}\in[0,1]^n:
\sum_{\nu=1}^{n}x_\nu\le k
\right\},
\]
where $k$ controls the expected summary size. 
For the real-data setting, we consider $k\in\{5,10,15\}$, corresponding to summary ratios of $10\%$, $20\%$, and $30\%$ when $n=50$. 
For the controlled clustered benchmark, we use $k=5$, corresponding to selecting one representative item from each cluster on average. 
After optimization, a deterministic summary is obtained by selecting the top-$k$ coordinates of $\mathbf{x}$.

\subsubsection{Evaluation protocol and metrics}

We evaluate each method from both continuous and discrete perspectives. 
For the continuous evaluation, we report the multilinear-extension utility \(F(\mathbf{x},\Omega)\). 
For the discrete evaluation, we round the continuous solution by selecting the top-\(k\) entries of \(\mathbf{x}\) and then evaluate the resulting summary under the discrete utility \(f(S_k(\mathbf{x}),\Omega)\). 
All attacked and robust solutions are evaluated under the original clean similarity matrix, so that the reported degradation reflects the quality loss caused by the perturbation rather than a change in the evaluation objective itself. 
For the controlled clustered benchmark, this protocol additionally allows us to interpret attack effectiveness in terms of whether the selected summary loses cluster coverage or fails to preserve representative items.

For attack evaluation, we report three metrics: {\bf Avg. Degradation}, {\bf Success Ratio}, and {\bf Attack Intensity}. 
Avg. Degradation measures the average decrease in clean-objective value caused by the attack. 
Success Ratio measures the fraction of target models whose clean-objective values are reduced by the generated perturbation. 
Attack Intensity is the normalized average degradation. For defense evaluation, we report {\bf Loss}, {\bf Robustness}, and {\bf Mitigation}. 
Loss measures the clean-data utility gap between the robust solution and the greedy solution. 
Robustness measures the stability of the robust solution when comparing clean and attacked settings. 
Mitigation measures how much the defense reduces the attack effect relative to the attacked greedy solution. 

For the controlled clustered benchmark, we further evaluate the structural quality of the rounded summary using three indicators: {\bf Coverage}, {\bf Hit}, and {\bf Redundancy}. 
Coverage measures how many ground-truth clusters are represented in the selected summary. 
Hit measures whether the summary successfully selects the designated hub representatives. 
Redundancy measures the extent to which the selected summary contains repeated information from only a few clusters. 
Higher coverage and hit indicate better structural quality, whereas lower redundancy indicates a more diverse and representative summary. 

Finally, to connect summarization-level degradation with downstream reliability, we evaluate downstream nearest-neighbor classification on the controlled clustered benchmark. 
Each cluster is treated as one class. 
For each method, we round the continuous solution by selecting the top-\(k\) items and use the selected summary as a small training set. 
Each clean test sample is then assigned the label of its most similar selected summary item under the clean similarity matrix. 
We report nearest-neighbor classification accuracy ({\bf NN Acc.}) and downstream {\bf Recovery}. 
NN Acc. measures the usefulness of the selected summary for downstream classification, while Recovery measures the fraction of downstream accuracy loss recovered relative to the attacked summary.


\subsection{Attack Performance under $l_p$-Norm Constraints}

We organize the attack evaluation into two complementary cases. 
The first case considers the real-data multilinear-extension summarization model, which is directly aligned with our theoretical formulation and is used to examine whether structure-aware similarity perturbations can produce measurable degradation beyond gradient-based and random baselines. 
The second case uses a controlled clustered multi-target benchmark with explicit representative structure, which allows us to isolate the multi-target attack mechanism and interpret degradation through cluster coverage, representative hits, and redundancy. 
We report representative norm settings in the main text and defer additional norm- and budget-specific results to the Appendix.

\subsubsection{Real-data multilinear-extension summarization}

We first evaluate the attack on the real-data multilinear-extension model. 
This setting is directly aligned with the theoretical objective, but the multilinear extension also smooths the discrete utility over randomized subsets. 
Therefore, the purpose of this case is not to show a large collapse of the real-data summarization model, but to test whether a structure-aware perturbation can produce measurable degradation beyond gradient-based and random baselines.

{\bf Choice of baselines.}
The representative methods summarized in Table~\ref{tab:related_positioning} differ substantially in target setting and structural assumptions. 
In particular, the methods in~\cite{adibi2022} do not explicitly address the continuous multi-target setting considered here. 
We therefore compare the proposed method with two directly applicable baselines: a representative gradient-based baseline for general non-convex min-max optimization~\cite{wang2021}, and a random perturbation baseline sampled from the same norm-constrained perturbation set. 
The random baseline serves as a sanity check for whether utility degradation is caused by optimized structure-aware perturbations rather than arbitrary bounded similarity noise.

\begin{table}[t]
\centering
\caption{Attack comparison on the multilinear-extension summarization model under representative norm constraints. We use \(I=10\), \(n=50\), \(k=5\), \(\lambda=1\), \(K_{\rm attack}=30\), and \(T_{\rm attack}=20\). Random perturbation is averaged over 20 independent trials.}
\label{tab:me_attack_baseline}
\scriptsize
\setlength{\tabcolsep}{2pt}
\resizebox{\columnwidth}{!}{
\begin{tabular}{l c c c}
\toprule
\textbf{Method}
& \textbf{Avg. Deg.} $\uparrow$
& \textbf{Succ. Ratio} $\uparrow$
& \textbf{Intensity} $\uparrow$ \\
\midrule
\multicolumn{4}{c}{\textbf{\(\ell_1\)-norm constraint, \(\epsilon=2.0\)}} \\
\midrule
Proposed attack          &  0.0329 & 0.60 &  0.000869 \\
PGD baseline         &  0.0097 & 0.70 &  0.000255 \\
Random perturbation  & -0.0355 & 0.10 & -0.000938 \\
\midrule
\multicolumn{4}{c}{\textbf{\(\ell_2\)-norm constraint, \(\epsilon=2.0\)}} \\
\midrule
Proposed attack         &  0.0548 & 0.70 &  0.001448 \\
PGD baseline         &  0.0001 & 0.70 &  0.000002 \\
Random perturbation  & -0.0355 & 0.40 & -0.000508 \\
\bottomrule
\end{tabular}
}
\end{table}
As shown in Table~\ref{tab:me_attack_baseline}, the proposed attack achieves the largest average degradation and attack intensity under both representative \(\ell_1\)- and \(\ell_2\)-norm constraints. Although PGD attains comparable or slightly higher success ratios in some cases, its degradation magnitude is much smaller. The random perturbation baseline yields negative degradation in both representative settings, indicating that arbitrary feasible perturbations do not reliably reduce the clean summarization utility. At the same time, the absolute degradation values are small, suggesting that the real-data multilinear-extension objective is relatively stable under bounded similarity-level perturbations.

\begin{figure*}[t!]
\centering
\includegraphics[width=1.7\columnwidth]{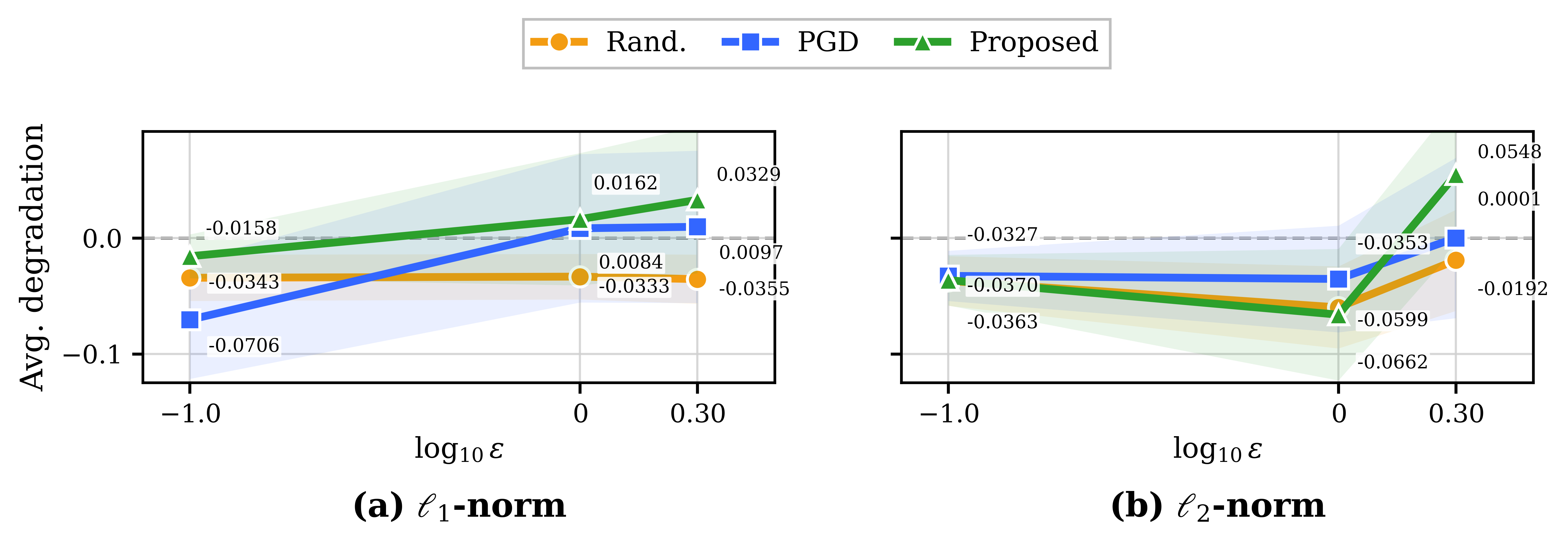}

\caption{Real-data budget sensitivity under $\ell_1$- and $\ell_2$-norm constraints. The $x$-axis reports $\log_{10}\epsilon$, and the $y$-axis reports average degradation. Shaded bands indicate the standard error across the 10 victim models. Node labels show the mean average-degradation values. Random perturbation is averaged over 20 independent trials.}
\label{fig:realdata_l1_l2_avg_degradation_sensitivity}
\end{figure*}

We further examine budget sensitivity in Fig.~\ref{fig:realdata_l1_l2_avg_degradation_sensitivity}. 
The proposed attack algorithm becomes effective in moderate-budget regimes, while random perturbations remain weak or negative in the most informative moderate-budget region. 
At very small budgets, all attacks may be weak, suggesting that the multilinear-extension objective is relatively stable to small similarity perturbations. 
This observation motivates the controlled benchmark below, where the representative structure is known and the attack mechanism can be more directly examined. We also tested $\ell_\infty$-bounded perturbations. In the multilinear-extension setting, the resulting degradation was weak and sometimes negative under the tested budgets, suggesting that the smoothed objective is relatively stable against small coordinate-wise similarity perturbations. We therefore report the detailed \(\ell_\infty\) results in the appendix.

\subsubsection{Controlled clustered multi-target benchmark}

We next evaluate the attacks on the controlled clustered benchmark. 
This case is designed to isolate the multi-target attack mechanism under a known representative structure. 
Each victim model contains explicit clusters and near-tie hub representatives, while the victim models share the same structural pattern with small model-specific noise. 
Therefore, this benchmark allows us to examine not only whether the objective value decreases, but also whether the selected summary loses cluster coverage, misses representative hubs, or becomes redundant.

{\bf Utility degradation under \(\ell_\infty\) perturbations.} For the controlled clustered benchmark, the strongest and most stable effects are observed under the \(\ell_\infty\)-norm constraint. 
We additionally test \(\ell_1\) and \(\ell_2\) attacks on a near-tie clustered variant, where the attack becomes effective after reducing the redundancy weight to \(\lambda=0.1\) and increasing \(T_{\rm attack}\) to 10.

Table~\ref{tab:cluster_attack_baseline} reports the attack performance under \(\ell_\infty\)-bounded perturbations. 
For low-to-moderate budgets, the proposed attack produces larger degradation than the PGD baseline. 
At \(\epsilon=0.3\), the proposed attack improves the average degradation from \(5.1324\) to \(6.5695\), and at \(\epsilon=0.5\), the gap increases from \(2.8914\) to \(6.8044\). 
Both methods achieve a success ratio of \(1.0\), so the advantage of the proposed attack comes from larger utility drops rather than more frequent successes. 
When the budget increases to \(\epsilon=0.7\), PGD becomes stronger, indicating that the advantage of the proposed attack is most pronounced when the perturbation budget is limited and the attack must exploit the shared representative structure.

\begin{table}[t]
\centering
\caption{Attack comparison on the controlled clustered multilinear-extension summarization benchmark. 
We use \(I=10\), \(n=50\), \(k=5\), \(\lambda=0.5\), \(K_{\rm attack}=30\), and \(T_{\rm attack}=5\).}
\label{tab:cluster_attack_baseline}
\scriptsize
\setlength{\tabcolsep}{2.8pt}
\resizebox{\columnwidth}{!}{
\begin{tabular}{c l c c c}
\toprule
\(\boldsymbol{\epsilon}\) 
& \textbf{Method}
& \textbf{Avg. Deg.} $\uparrow$
& \textbf{Succ. Ratio} $\uparrow$
& \textbf{Intensity} $\uparrow$ \\
\midrule
\multicolumn{5}{c}{\textbf{\(\ell_\infty\)-norm constraint}} \\
\midrule
\multirow{2}{*}{0.3}
& Proposed attack  & 6.5695 & 1.00 & 0.1714 \\
& PGD baseline & 5.1324 & 1.00 & 0.1339 \\
\cmidrule(lr){1-5}
\multirow{2}{*}{0.5}
& Proposed attack  & 6.8044 & 1.00 & 0.1776 \\
& PGD baseline & 2.8914 & 1.00 & 0.0755 \\
\cmidrule(lr){1-5}
\multirow{2}{*}{0.7}
& Proposed attack  & 10.5072 & 1.00 & 0.2742 \\
& PGD baseline & 14.0013 & 1.00 & 0.3654 \\
\bottomrule
\end{tabular}
}
\end{table}

{\bf Additional \(\ell_1\) and \(\ell_2\) results.} Table~\ref{tab:cluster_l1_l2_attack_formal} reports the selected cluster-data results under \(\ell_1\) and \(\ell_2\) constraints. 
Compared with the original tight-gap clustered setting used for the \(\ell_\infty\) study, this near-tie clustered variant creates more competitive candidate summaries and therefore better exposes the effect of global-budget perturbations. 
Under the \(\ell_1\)-norm constraint, the proposed attack achieves positive degradation and attack intensity, while random perturbation has nearly zero effect, indicating that arbitrary feasible perturbations are insufficient to explain the observed performance drop. 
Under the \(\ell_2\)-norm constraint, both optimized attacks substantially reduce the clean utility, with the proposed attack achieving the largest average degradation and intensity. 
Although random perturbation also succeeds under this larger \(\ell_2\) budget, its intensity is much smaller than that of the optimized attacks, showing that the structure-aware perturbation is considerably more damaging.

\begin{table}[t]
\centering
\caption{Attack comparison on the controlled clustered multilinear-extension summarization benchmark. We use the near-tie clustered setting with \(I=10\), \(n=50\), \(k=5\), \(\lambda=0.1\), \(K_{\rm attack}=30\), and \(T_{\rm attack}=10\).}
\label{tab:cluster_l1_l2_attack_formal}
\scriptsize
\setlength{\tabcolsep}{4 pt}
\resizebox{\columnwidth}{!}{
\begin{tabular}{c c l c c c}
\toprule
\textbf{Norm} & \(\epsilon\) & \textbf{Method} & \textbf{Avg. Deg.} \(\uparrow\) & \textbf{Succ.} \(\uparrow\) & \textbf{Int.} \(\uparrow\) \\
\midrule
\(\ell_1\) & 2 & Proposed attack & 1.6064 & 0.40 & 0.0390 \\
 &  & PGD baseline & 1.6064 & 0.40 & 0.0390 \\
 &  & Random perturbation & -0.0065 & 0.00 & -0.0002 \\
\cmidrule(lr){1-6}
\(\ell_2\) & 5 & Proposed attack & 11.1272 & 1.00 & 0.2705 \\
 &  & PGD baseline & 9.1592 & 1.00 & 0.2226 \\
 &  & Random perturbation & 1.4105 & 1.00 & 0.0343 \\
\bottomrule
\end{tabular}
}
\end{table}

{\bf Structural effects on rounded summaries.}
Table~\ref{tab:cluster_structural_metrics} further explains the utility degradation through the rounded summary structure. 
Under both representative budgets, the proposed attack causes larger coverage loss, stronger reduction in representative-hit ratio, and higher redundancy than the PGD baseline. 
For example, when \(\epsilon=0.5\), Proposed attack reduces coverage from \(1.000\) to \(0.820\) and hit ratio from \(0.580\) to \(0.320\), while increasing redundancy to \(0.180\). 
These results show that the attack not only lowers continuous objective value, but also damages discrete summary by removing hub representatives and concentrating selections in fewer clusters.

Overall, the two cases provide complementary evidence. 
The real-data multilinear-extension setting shows measurable but budget-sensitive vulnerability under optimized similarity perturbations. 
The controlled clustered benchmark further reveals when this vulnerability becomes more pronounced: when multiple victim models share a near-tie representative structure, a single bounded perturbation can systematically reduce utility and damage the rounded summary. 
Thus, the proposed attack is best interpreted as a structure-aware multi-target attack that is most informative in low-to-moderate budget regimes, rather than as a uniformly strongest attack across all norms and budgets.

\begin{table}[t]
\centering
\caption{Structural degradation of the selected summaries on the controlled clustered benchmark. 
Coverage and representative hit ratio are higher for better summaries, while redundancy is lower.}
\label{tab:cluster_structural_metrics}
\small
\setlength{\tabcolsep}{4.5pt}
\begin{tabular}{c l c c c}
\toprule
\(\boldsymbol{\epsilon}\) 
& \textbf{Metric}
& \textbf{Clean}
& \textbf{Proposed attack}
& \textbf{PGD attack} \\
\midrule
\multirow{6}{*}{0.3}
& Cov. $\uparrow$        & 1.000 & 0.880 & 0.980 \\
& \(\Delta\)Cov. $\uparrow$ & 0.000 & 0.120 & 0.020 \\
& Hit $\uparrow$             & 0.580 & 0.420 & 0.500 \\
& \(\Delta\)Hit $\uparrow$      & 0.000 & 0.160 & 0.080 \\
& Red. $\downarrow$    & 0.000 & 0.120 & 0.020 \\
& \(\Delta\)Red. $\uparrow$     & 0.000 & 0.120 & 0.020 \\
\midrule
\multirow{6}{*}{0.5}
& Cov. $\uparrow$        & 1.000 & 0.820 & 0.860 \\
& \(\Delta\)Cov. $\uparrow$ & 0.000 & 0.180 & 0.140 \\
& Hit $\uparrow$             & 0.580 & 0.320 & 0.600 \\
& \(\Delta\)Hit $\uparrow$      & 0.000 & 0.260 & -0.020 \\
& Red. $\downarrow$    & 0.000 & 0.180 & 0.140 \\
& \(\Delta\)Red. $\uparrow$     & 0.000 & 0.180 & 0.140 \\
\bottomrule
\end{tabular}
\vspace{-6pt}
\end{table}

\subsection{Robustness of the Defense Algorithm}
\label{results:robust}

We evaluate whether the proposed robust algorithm can reduce adversarial effects while preserving clean summarization quality. 
The controlled clustered benchmark is used as the main defense testbed, since its known representative structure makes robustness and mitigation effects directly interpretable. 
We also report MovieLens results as a real-data sanity check for the multilinear-extension summarization model.

\subsubsection{Controlled clustered benchmark}

We compare the proposed robust method with a representative PGD-based robust baseline under the same attack and evaluation protocol. 
The goal is to examine whether robust optimization can reduce clean-attacked variation, improve mitigation over the attacked greedy solution, and preserve clean-data utility.

\begin{table}[t]
\centering
\caption{Defense comparison on the controlled clustered benchmark. 
We use the \(\ell_\infty\)-bounded setting with \(\epsilon_{\rm att}=0.5\), \(\epsilon_{\rm def}=0.02\), \(\gamma=0.1\), \(L=10\), \(K_{\rm robust}=30\), and \(T_{\rm robust}=5\).}
\label{tab:defense_cluster_baseline}
\scriptsize
\setlength{\tabcolsep}{4pt}
\resizebox{\columnwidth}{!}{
\begin{tabular}{l c c c c}
\toprule
\textbf{Method}
& \textbf{Loss} \(\downarrow\)
& \textbf{Robust.} \(\downarrow\)
& \textbf{Mitig.} \(\uparrow\)
& \textbf{Time} \(\downarrow\) \\
\midrule
Proposed robust      & -0.0092 & 0.1029 & 0.0826 & 2.99 \\
PGD robust baseline  & -0.0526 & 0.1464 & 0.0757 & 4.80 \\
\bottomrule
\end{tabular}
}
\vspace{-6pt}
\end{table}

Table~\ref{tab:defense_cluster_baseline} shows that the proposed robust method provides a more favorable defense trade-off on the controlled clustered benchmark. 
It achieves lower clean-attacked variation, slightly higher mitigation, and shorter runtime than the PGD robust baseline. 
Both methods have slightly negative loss values, indicating that robust optimization does not incur observable clean-utility sacrifice in this controlled setting.

\subsubsection{Real-data sanity check}

We also evaluate the defense methods on the MovieLens real-data multilinear-extension model. 
Since this setting does not provide an explicit representative structure, the defense effect is weaker and more sensitive to the perturbation geometry; we therefore treat it as a sanity check rather than the main defense evidence.

\begin{table}[t]
\centering
\caption{Real-data defense comparison on MovieLens under representative norm constraints.}
\label{tab:realdata_defense}
\scriptsize
\setlength{\tabcolsep}{2pt}
\resizebox{\columnwidth}{!}{
\begin{tabular}{c c c l c c c c}
\toprule
\textbf{Norm} & \(\epsilon_{\rm att}\) & \(\epsilon_{\rm def}\)
& \textbf{Method}
& \textbf{Loss} \(\downarrow\)
& \textbf{Robust.} \(\downarrow\)
& \textbf{Mitig.} \(\uparrow\)
& \textbf{Time} \(\downarrow\) \\
\midrule
\multirow{2}{*}{\(\ell_1\)}
& \multirow{2}{*}{2}
& \multirow{2}{*}{2}
& Proposed robust & 0.0005 & 0.0041 & -0.0011 & 2.27 \\
& & & PGD robust baseline & 0.0184 & 0.0012 & -0.0180 & 3.03 \\
\midrule
\multirow{2}{*}{\(\ell_2\)}
& \multirow{2}{*}{2}
& \multirow{2}{*}{2}
& Proposed robust & 0.0005 & 0.0044 & 0.0009 & 1.24 \\
& & & PGD robust baseline & 0.0234 & 0.0012 & -0.0222 & 2.00 \\
\midrule
\multirow{2}{*}{\(\ell_\infty\)}
& \multirow{2}{*}{0.1}
& \multirow{2}{*}{0.1}
& Proposed robust & 0.0205 & 0.0105 & -0.0223 & 1.28 \\
& & & PGD robust baseline & 0.0255 & 0.0007 & -0.0283 & 1.86 \\
\bottomrule
\end{tabular}
}
\vspace{-6pt}
\end{table}

Table~\ref{tab:realdata_defense} shows that the proposed robust method generally incurs smaller clean-data loss than the PGD robust baseline on MovieLens. 
However, the mitigation values are close to zero and depend on the norm constraint, confirming that real-data robust protection is more parameter-sensitive than in the controlled clustered benchmark. 
These results support the use of the controlled benchmark as the main defense evaluation while providing additional evidence that the proposed method can preserve clean utility in the real-data multilinear-extension setting.

\subsubsection{Defense Sensitivity under Different Norm Constraints}

Table~\ref{tab:defense_all_norms} reports an additional defense sensitivity study under different norm constraints on the controlled clustered benchmark and the MovieLens real-data setting. 
This experiment is intended as a cross-norm sanity check rather than the main defense comparison, since the effectiveness of robust optimization depends on the geometry of the perturbation set, the attack budget, and the underlying similarity structure.

On the controlled clustered benchmark, the proposed robust method performs favorably under the \(\ell_1\)-norm constraint: it incurs almost no clean-data loss and avoids the large negative mitigation observed for the PGD robust baseline. 
Under the \(\ell_2\)-norm constraint, however, both methods obtain negative mitigation, indicating that this setting is difficult for the tested defense algorithms and that robust optimization does not effectively improve over the attacked greedy solution. 
Under the \(\ell_\infty\)-norm constraint, both methods achieve positive mitigation; the proposed method has smaller clean-data loss and slightly lower robustness variation, whereas the PGD robust baseline attains larger mitigation. 
These results suggest that defense performance on the clustered benchmark is sensitive to the norm geometry and should be interpreted as a robustness--mitigation trade-off rather than a uniform dominance result.

On the MovieLens real-data setting, the robustness values are small for both methods across all tested norms, suggesting that the real-data objective is relatively stable under the considered bounded perturbations. 
The proposed robust method generally incurs smaller clean-data loss than the PGD robust baseline, but the mitigation values remain close to zero and can be negative. 
This indicates that, in the real-data setting, robust protection is more parameter-sensitive and the benefit of defense is weaker than in the structured clustered benchmark. Overall, Table~\ref{tab:defense_all_norms} shows that the proposed defense can provide favorable trade-offs in representative settings, but its effectiveness is not uniform across all norm constraints. 
These results support our main-text focus on the controlled \(\ell_\infty\)-bounded clustered benchmark, where the attack induces clear structural degradation and the defense effect is more interpretable.
\begin{table}[t]
\centering
\caption{Defense sensitivity under different norm constraints on the controlled clustered benchmark and MovieLens.}
\label{tab:defense_all_norms}
\scriptsize
\setlength{\tabcolsep}{4pt}
\resizebox{\columnwidth}{!}{
\begin{tabular}{c c l c c c}
\toprule
\textbf{Norm} 
& \(\boldsymbol{\epsilon}\) 
& \textbf{Method} 
& \textbf{Loss} \(\downarrow\) 
& \textbf{Robust.} \(\downarrow\) 
& \textbf{Mitig.} \(\uparrow\) \\
\midrule
\multicolumn{6}{c}{\textbf{Controlled clustered benchmark}} \\
\midrule
\(\ell_1\) & 2 
& Proposed robust & -0.0025 & 0.0472 & 0.0000 \\
&  & PGD robust      &  0.1723 & 0.0375 & -0.1910 \\
\cmidrule(lr){1-6}
\(\ell_2\) & 2 
& Proposed robust & 0.2352 & 0.1960 & -0.4377 \\
&  & PGD robust      & 0.1949 & 0.2251 & -0.4286 \\
\cmidrule(lr){1-6}
\(\ell_\infty\) & 0.5 
& Proposed robust &  0.0163 & 0.1445 & 0.0187 \\
&  & PGD robust      & -0.0526 & 0.1549 & 0.0668 \\
\midrule
\multicolumn{6}{c}{\textbf{MovieLens real-data setting}} \\
\midrule
\(\ell_1\) & 2 
& Proposed robust & 0.0005 & 0.0041 & -0.0011 \\
&  & PGD robust      & 0.0184 & 0.0012 & -0.0180 \\
\cmidrule(lr){1-6}
\(\ell_2\) & 2 
& Proposed robust & 0.0005 & 0.0044 & 0.0009 \\
&  & PGD robust      & 0.0234 & 0.0012 & -0.0222 \\
\cmidrule(lr){1-6}
\(\ell_\infty\) & 0.1 
& Proposed robust & 0.0205 & 0.0105 & -0.0223 \\
&  & PGD robust      & 0.0255 & 0.0007 & -0.0283 \\
\bottomrule
\end{tabular}
}
\end{table}

\subsection{Downstream Evaluation on the Controlled Benchmark}
\label{results:downstream}

The previous metrics evaluate summarization utility and structural quality. 
We now examine whether such degradation also affects downstream task reliability on the controlled clustered benchmark. 
Following the downstream protocol described in the experimental setup, we use the rounded top-\(k\) summaries as small training sets and evaluate nearest-neighbor classification on clean test samples. 
An auxiliary real-data downstream evaluation on MovieLens is provided in the Appendix, where the downstream labels are less directly aligned with the summarization objective.

\begin{table}[t]
\centering
\caption{Downstream classification performance on the controlled clustered benchmark.}
\label{tab:downstream_cluster}
\scalebox{1.1}{
\begin{tabular}{l c c c}
\toprule
\textbf{Method} 
& \textbf{Coverage} $\uparrow$
& \textbf{NN Acc.} $\uparrow$
& \textbf{Recovery} $\uparrow$ \\
\midrule
Clean summary        & 1.000 & 1.000 & 1.000 \\
Proposed attack           & 0.820 & 0.820 & 0.000 \\
PGD attack           & 0.860 & 0.860 & 0.222 \\
Proposed robust      & 1.000 & 1.000 & 1.000 \\
PGD robust baseline  & 1.000 & 1.000 & 1.000 \\
\bottomrule
\end{tabular}
}
\vspace{-6pt}
\end{table}

Table~\ref{tab:downstream_cluster} shows that structural degradation can translate into downstream performance loss. 
The proposed attack reduces cluster coverage from \(1.000\) to \(0.820\), and NN Acc. drops accordingly from \(1.000\) to \(0.820\). 
The PGD attack also reduces downstream accuracy, but the drop is smaller. 
After robust optimization, both robust methods recover full cluster coverage and restore NN Acc. to \(1.000\). 
These results indicate that preserving class/cluster coverage is critical for the task-level reliability of the selected summary.

\begin{figure}[t]
\centering
\includegraphics[width=0.82\columnwidth]{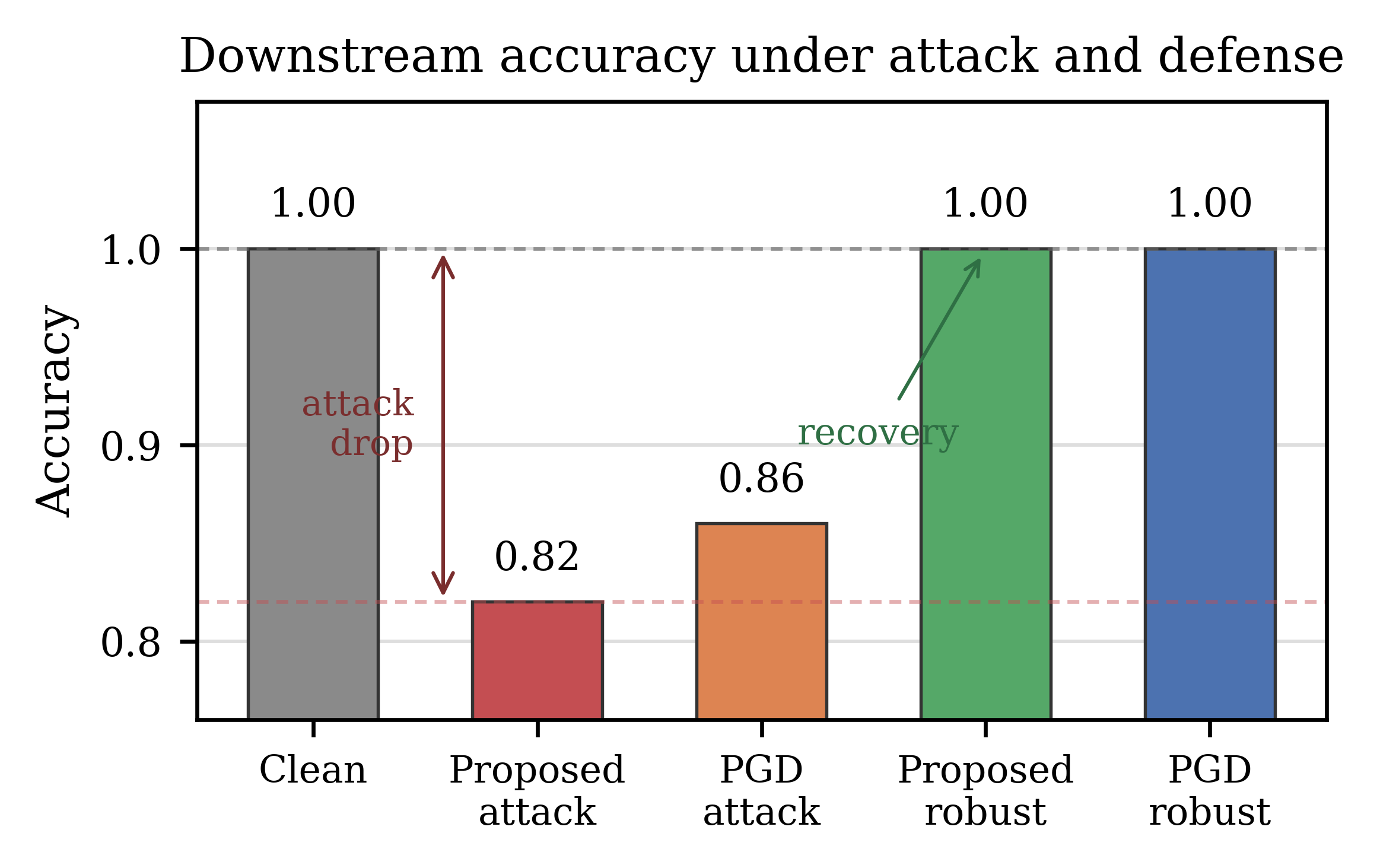}
\caption{Downstream NN accuracy under attack and defense on the controlled clustered benchmark.}
\label{fig:cluster_downstream_accuracy}
\end{figure}

Fig.~\ref{fig:cluster_downstream_accuracy} visually summarizes the downstream effect: attacks reduce NN accuracy, whereas robust summaries recover full task-level accuracy by restoring cluster coverage.

\section{Conclusion}
\label{sec:conclu}

This paper investigated multi-target robustness evaluation for continuous data summarization under similarity-level adversarial perturbations. 
We established a DR-submodular formulation for a class of similarity-based summarization objectives via multilinear extension and \(m\)-weak monotonicity, and used this structure to develop approximation algorithms for multi-target attack generation and robust defense. Empirical results on real-data multilinear-extension summarization and a controlled clustered benchmark show measurable, budget-sensitive vulnerabilities under optimized similarity perturbations. 
The controlled benchmark further reveals that attacks can degrade rounded summary structure and reduce downstream task performance, while robust summaries can recover task-level reliability by restoring class/cluster coverage. 
These findings highlight data summarization as a security-relevant upstream component in trustworthy AI pipelines. 
They also show that robust protection remains sensitive to data structure, perturbation geometry, and parameter choices, leaving practical robust summarization as an important direction for future work.

\appendix

\subsection{\bf Proofs for Sec.~\ref{sec:preli}}\label{supp:proofs-pre}
\subsubsection{Several properties in Sec.~\ref{sec:preli}}
For a set function $f:2^{\Omega}\to\mathbb{R}$, its multilinear extension $F:[0,1]^n\to\mathbb{R}$ is defined as
\[
F(\mathbf{x})
=
\mathbb{E}_{S\sim \mathbf{x}}[f(S)]
=
\sum_{S\subseteq\Omega}
f(S)
\prod_{i\in S}x_i
\prod_{i\notin S}(1-x_i),
\]
where each element $i\in\Omega$ is independently included in $S$ with probability $x_i$.
Since $F$ is a multilinear polynomial, it is continuously differentiable and twice differentiable on $[0,1]^n$. 
Moreover, if $f$ is submodular, then its multilinear extension $F$ is DR-submodular. 
Indeed, for any distinct $i,j$,
\begin{eqnarray*}
\frac{\partial^2 F(\mathbf{x})}{\partial x_i\partial x_j}&=&
\mathbb{E}_{S\sim\mathbf{x}_{-\{i,j\}}}
\left[
f(S\cup\{i,j\})-f(S\cup\{i\})\right.\\
&&-\left.f(S\cup\{j\})+f(S)
\right]\\
&\le& 0.
\end{eqnarray*}
In addition, if $f$ is non-negative, then $F$ is non-negative. 
If $f$ is $m$-weakly monotone, namely
\[
f(A\cup B)\ge m f(A),\qquad \forall A,B\subseteq\Omega,
\]
then its multilinear extension is also $m$-weakly monotone:
\[
F(\mathbf{x}\vee\mathbf{y})\ge mF(\mathbf{x}),
\qquad \forall \mathbf{x},\mathbf{y}\in[0,1]^n.
\]
Suppose $f$ is DR-submodular on $\mathcal{D}$. Then the following statements hold ~\cite{bian2017guaranteed}.
\begin{enumerate}[leftmargin=12pt]
\item \textbf{Gradient monotonicity.} If $f$ is differentiable, then
\[
\nabla f(\mathbf{x})\le \nabla f(\mathbf{y}), \qquad \forall \mathbf{x}\ge \mathbf{y}.
\]

\item \textbf{Hessian negativity.} If $f$ is twice differentiable, then
\[
\frac{\partial^{2}f(\mathbf{x})}{\partial \mathbf{x}_{i}\partial \mathbf{x}_{j}}\le 0,
\qquad \forall \mathbf{x}\in\mathcal{D},\ \forall i,j.
\]

\item \textbf{First-order condition.} If $f$ is continuously differentiable, then
\begin{equation}\label{lem:DR_equal}
    \langle \nabla f(\mathbf{x}),\mathbf{y}-\mathbf{x}\rangle
\ge
f(\mathbf{x}\vee \mathbf{y})+f(\mathbf{x}\wedge \mathbf{y})-2f(\mathbf{x}),
\qquad \forall \mathbf{x},\mathbf{y}\in\mathcal{D}.
\end{equation}
\end{enumerate}

In order to prove the smoothness of the model, we provide the definition as follows firstly.
\begin{definition}[~\cite{thekumparampil2019,xu2023}] A function $f(\mathbf{x},\mathbf{y})$ is said to be $L$-smoothness if there exists a constant $L=\max\{L_{\mathbf{x}},L_{\mathbf{y}}\}$ such that
\begin{align*}
    &\lVert \nabla_{\mathbf{x}}f(\mathbf{x}_{1},\mathbf{y}_{1})-\nabla_{\mathbf{x}}f(\mathbf{x}_{2},\mathbf{y}_{2})\rVert \leq L_{\mathbf{x}}(\lVert\mathbf{x}_{1}-\mathbf{x}_{2}\rVert+ \lVert \mathbf{y}_{1}-\mathbf{y}_{2}\rVert),\\
    & \lVert \nabla_{\mathbf{y}}f(\mathbf{x}_{1},\mathbf{y}_{1})-\nabla_{\mathbf{y}}f(\mathbf{x}_{2},\mathbf{y}_{2})\rVert \leq L_{\mathbf{y}}(\lVert \mathbf{x}_{1}-\mathbf{x}_{2}\rVert+ \lVert \mathbf{y}_{1}-\mathbf{y}_{2}\rVert).
\end{align*}
where $\lVert \cdot \rVert$ denotes the Euclidean norm ($l_{2}$-norm).
\end{definition}

\begin{lemma}
Let $F$ be the multilinear extension of a bounded set function $f:2^\Omega\to\mathbb{R}$. 
Then $F$ is continuously differentiable and $L$-smooth on $[0,1]^n$ with
\[
L
=
\sup_{\mathbf{x}\in[0,1]^n}
\|\nabla^2 F(\mathbf{x})\|_2
<\infty.
\]
\end{lemma}

In particular, the first-order condition implies that if $\mathbf{y}\ge \mathbf{x}$, then
\[
\langle \nabla f(\mathbf{x}),\mathbf{y}-\mathbf{x}\rangle \ge f(\mathbf{y})-f(\mathbf{x}),
\]
which shows that DR-submodular functions are up-concave, i.e., concave along any non-negative or non-positive direction~\cite{lian2024zeroth}. We evaluate constrained DR-submodular maximization algorithms using the standard notion of approximation.

\subsubsection{ Proof of Lem.~\ref{lem:multi-reso}}
\label{supp:proofs-pre}

\begin{proof}
We first prove the properties of the discrete utility
\[
f(S,\Omega)=r(S,\Omega)-\frac{\lambda}{n}q(S,\Omega),
\]
where
\[
r(S,\Omega)=\sum_{\mu\in\Omega}\max_{\nu\in S}s_{\mu,\nu},
\quad
q(S,\Omega)=\sum_{\mu\in S}\sum_{\nu\in S,\nu\neq\mu}s_{\mu,\nu}.
\]

First, by the definition of $\rho_\Omega$, for any nonempty $S\subseteq\Omega$,
\[
\frac{\lambda}{n}q(S,\Omega)
\le
\rho_\Omega r(S,\Omega).
\]
Since $0\le \rho_\Omega<1$, we have
\[
f(S,\Omega)
\ge
(1-\rho_\Omega)r(S,\Omega)
\ge 0.
\]
Also, $f(\emptyset,\Omega)=0$ by convention. Hence $f$ is non-negative.

Second, the facility-location term $r(S,\Omega)$ is submodular. 
Indeed, for any $A\subseteq B\subseteq\Omega$ and $e\notin B$,
\[
r(A\cup\{e\},\Omega)-r(A,\Omega)
=
\sum_{\mu\in\Omega}
\left[
s_{\mu,e}-\max_{\nu\in A}s_{\mu,\nu}
\right]_+,
\]
which is no smaller than
\[
\sum_{\mu\in\Omega}
\left[
s_{\mu,e}-\max_{\nu\in B}s_{\mu,\nu}
\right]_+
=
r(B\cup\{e\},\Omega)-r(B,\Omega),
\]
because $A\subseteq B$. 
Thus $r$ is submodular.

On the other hand, the redundancy term $q(S,\Omega)$ is supermodular. 
For any $A\subseteq B\subseteq\Omega$ and $e\notin B$,
\begin{eqnarray*}
&&q(A\cup\{e\},\Omega)-q(A,\Omega)\\
&=&
\sum_{\nu\in A}(s_{e,\nu}+s_{\nu,e})\\
&\le&
\sum_{\nu\in B}(s_{e,\nu}+s_{\nu,e})
=
q(B\cup\{e\},\Omega)-q(B,\Omega),
\end{eqnarray*}
where the inequality follows from $s_{\mu,\nu}\ge0$. 
Therefore, $-\frac{\lambda}{n}q$ is submodular, and hence $f$ is submodular.

Third, we prove weak monotonicity. 
For any $A,B\subseteq\Omega$, let $C=A\cup B$. 
Since $r$ is monotone, $r(C,\Omega)\ge r(A,\Omega)$. 
Using the definition of $\rho_\Omega$, we obtain
\[
f(C,\Omega)
\ge
(1-\rho_\Omega)r(C,\Omega)
\ge
(1-\rho_\Omega)r(A,\Omega).
\]
Moreover, $f(A,\Omega)\le r(A,\Omega)$. 
Therefore,
\[
f(A\cup B,\Omega)
\ge
(1-\rho_\Omega)f(A,\Omega).
\]
Thus, $f$ is $m_\Omega$-weakly monotone with $m_\Omega=1-\rho_\Omega$.

Finally, since $F(\mathbf{x},\Omega)$ is the multilinear extension of $f(S,\Omega)$, the standard properties of multilinear extensions imply that non-negativity, submodularity, and $m_\Omega$-weak monotonicity of $f$ respectively lead to non-negativity, DR-submodularity, and $m_\Omega$-weak monotonicity of $F$. 
This completes the proof.
\end{proof}

\subsection{\bf Proofs for Sec.~\ref{subsec:appro}}\label{supp:proofs}

\subsubsection{Proof of Lemmas}

Proof of Lem. \ref{lem:attack_key}
\begin{proof}
According to the iteration rule of Alg.~\ref{alg:attack}, i.e., 
$$
\mathbf{v}_{k+1} = \arg \min_{\mathbf{v}\in \mathcal{P}}\lVert \mathbf{v}-(\mathbf{v}_{k}- \eta \nabla_{\mathbf{v}}\phi(\mathbf{v}_{k},\mathbf{w}_{k},\{\mathbf{x}^{i}_{T}(\mathbf{v}_{k},\mathbf{w}_{k})\}))\rVert^{2},
$$
we obtain the following inequality for any $\mathbf{v}\in \mathcal{P}$:
\begin{align*}
    &\lVert \mathbf{v}_{k+1}-\mathbf{v}\rVert_{2}^{2}\\
    \leq &\lVert \mathbf{v}_{k}-\eta\nabla_{\mathbf{v}}\phi(\mathbf{v}_{k},\mathbf{w}_{k},\{\mathbf{x}_{T}^{i}(\mathbf{v}_{k},\mathbf{w}_{k})\})-\mathbf{v}\rVert^{2}\\
    = &\lVert \mathbf{v}_{k}-\mathbf{v} \rVert^{2}-2\eta \underbrace{ \langle \nabla_{\mathbf{v}}\phi(\mathbf{v}_{k},\mathbf{w}_{k},\{\mathbf{x}_{T}^{i}(\mathbf{v}_{k},\mathbf{w}_{k})\}),\mathbf{v}_{k}-\mathbf{v}\rangle}_{:={\bf A}} \\
    &+\eta^{2}\lVert \nabla_{\mathbf{v}}\phi(\mathbf{v}_{k},\mathbf{w}_{k},\{\mathbf{x}_{T}^{i}(\mathbf{v}_{k},\mathbf{w}_{k})\})\rVert^{2},
\end{align*}
which can be expressed as 
\begin{align*}
    {\bf A}\geq & \frac{\lVert \mathbf{v}_{k+1}-\mathbf{v}\rVert_{2}^{2}- \lVert \mathbf{v}_{k}-\mathbf{v} \rVert^{2}}{2\eta}\\
    &-\frac{-\eta\lVert \nabla_{\mathbf{v}}\phi(\mathbf{v}_{k},\mathbf{w}_{k},\{\mathbf{x}_{T}^{i}(\mathbf{v}_{k},\mathbf{w}_{k})\})\rVert^{2}}{2}.
\end{align*}
Combining with the convexity of the function $\phi(\mathbf{v},\mathbf{w},\{\mathbf{x}^{i}\})$ w.r.t. $\mathbf{v}$, we have
\begin{align*}
    &\phi(\mathbf{v},\mathbf{w}_{k},\{\mathbf{x}_{T}^{i}(\mathbf{v}_{k},\mathbf{w}_{k})\})- \phi(\mathbf{v}_{k},\mathbf{w}_{k},\{\mathbf{x}_{T}^{i}(\mathbf{v}_{k},\mathbf{w}_{k})\})\\
    \geq& \underbrace{\left \langle \nabla_{\mathbf{v}}\phi(\mathbf{v}_{k},\mathbf{w}_{k},\{\mathbf{x}_{T}^{i}(\mathbf{v}_{k},\mathbf{w}_{k})\}),\mathbf{v}-\mathbf{v}_{k}\right \rangle}_{{\bf A}}\\
    \geq& \frac{\lVert \mathbf{v}_{k+1}-\mathbf{v}\rVert^{2}- \lVert \mathbf{v}_{k}-\mathbf{v} \rVert^{2}}{2\eta}\\
    &-\frac{\eta}{2}\lVert \nabla_{\mathbf{v}}\phi(\mathbf{v}_{k},\mathbf{w}_{k},\{\mathbf{x}_{T}^{i}(\mathbf{v}_{k},\mathbf{w}_{k})\})\rVert^{2}\\
    \geq& \frac{\lVert \mathbf{v}_{k+1}-\mathbf{v}\rVert^{2}- \lVert \mathbf{v}_{k}-\mathbf{v} \rVert^{2}}{2\eta}-\frac{\eta M^{2}}{2}.
\end{align*}
where the last inequality follows from the boundedness of $\nabla_{\mathbf{v}}\phi$ as descried in Eq.~\eqref{eq:attack_bound_phi}. 
\end{proof}
\subsubsection{ Proof of Theorems}

Proof of Thm. \ref{thm:approximation}
\begin{proof}
Firstly, we show that the output $\mathbf{x}_{T}$ is feasible. Due to $\mathbf{d}_{t}\leq \bar{\mathbf{d}}_{t}$ and the set $\mathcal{X}$ is down-closed convex, we have $\mathbf{d}_{t}\in \mathcal{X}$. Combined that $\mathbf{x}_{T}$ is the combination of all iterations $\mathbf{d}_{t}, t=1,\ldots T-1$. According to the smoothness of function $F(\mathbf{v},\mathbf{x})$ and Eq.~\eqref{eq:pre_equ}, we have the following inequality for all $\mathbf{x}\in \mathcal{X}$   
\begin{align*}
    &F(\mathbf{v},\mathbf{x}_{t+1})- F(\mathbf{v},\mathbf{x}_{t})\\
    \geq& \langle \nabla F(\mathbf{v}, \mathbf{x}_{t}), \mathbf{x}_{t+1}-\mathbf{x}_{t}\rangle -\frac{L}{2}\lVert \mathbf{x}_{t+1}-\mathbf{x}_{t}\rVert^{2}\\
    \overset{\eqref{eq:pre_equ}}{\geq}& \langle g(\mathbf{x}_{t}), \frac{1}{T}\mathbf{d}_{t}\rangle -\frac{L}{2T^{2}}\lVert \mathbf{d}_{t}\rVert^{2}\\
    \overset{(a)}{\geq}& \frac{1}{T} \langle g(\mathbf{x}_{t}),\mathbf{x}\rangle -\frac{L}{2T^{2}}D^{2}\\
     \overset{(b)}{\geq}& \frac{1}{T} \langle g(\mathbf{x}_{t}), \mathbf{x}\vee \mathbf{x}_{t}-\mathbf{x}_{t} \rangle -\frac{L}{2T^{2}}D^{2}\\
     \overset{(c)}{\geq}& \frac{1}{T} \langle \nabla F(\mathbf{v}, \mathbf{x}_{t}), \mathbf{x}\vee \mathbf{x}_{t}-\mathbf{x}_{t}\rangle -\frac{L}{2T^{2}}D^{2}\\
    \overset{(d)}{\geq}& \frac{1}{T}\left[F(\mathbf{v}, \mathbf{x}\vee \mathbf{x}_{t})-F(\mathbf{v}, \mathbf{x}_{t}) \right]-\frac{L}{2T^{2}}D^{2}\\
    \overset{(f)}{\geq}& \frac{1}{T}\left[m F(\mathbf{v}, \mathbf{x})-F(\mathbf{v}, \mathbf{x}_{t}) \right]-\frac{L}{2T^{2}}D^{2}
\end{align*}
where $(a)$ is according to the Step 4 in Alg.~\ref{alg:Mgreedy}, $(b)$ follows from the fact that $\mathbf{x}+\mathbf{x}_{t}\geq \mathbf{x}\vee \mathbf{x}_{t}$ for all $\mathbf{x},\mathbf{x}_{t} \in \mathcal{X}$, and $(c)$ is a consequence of the definition $g(\mathbf{x}_{t})= \nabla F(\mathbf{v}, \mathbf{x}_{t})\vee \mathbf{0}$, $(d)$ and $(f)$ are obtained from DR-submodularity and monotonicity, respectively. Thus, we have
\begin{align*}
    &F(\mathbf{v}, \mathbf{x}_{t+1})- m F(\mathbf{v}, \mathbf{x})\\
    \geq &(1-\frac{1}{T})\left(F(\mathbf{v}, \mathbf{x}_{t})- m F(\mathbf{v},\mathbf{x})\right)-\frac{L}{2T^{2}}D^{2}
\end{align*}
Summing up the above inequalities over $t=0,\ldots, T-1$, we obtain
\begin{align*}
    &F(\mathbf{v}, \mathbf{x}_{T})- m F(\mathbf{v}, \mathbf{x})\\
    \geq &(1-\frac{1}{T})^{T}\left(F(\mathbf{v}, \mathbf{x}_{0})- m F(\mathbf{v},\mathbf{x}^{*})\right)-\frac{L}{2T}D^{2}.
\end{align*}
Furthermore, by choosing $\mathbf{x}=\mathbf{x}^{*}=\arg \max_{\mathbf{x}\in \mathcal{X}}F(\mathbf{v},\mathbf{x})$, and combining with the non-negativity of $F$, i.e., $F(\mathbf{v}, \mathbf{x}_{0})\geq 0$, we have
\begin{align*}
    F(\mathbf{v}, \mathbf{x}_{T}) &\geq m\left(1-(1-\frac{1}{T})^{T}\right)F(\mathbf{v},\mathbf{x})-\frac{L}{2T}D^{2}\\
    &\geq m(1-1/e) F(\mathbf{v},\mathbf{x}^{*})-\frac{L}{2T}D^{2},
\end{align*}
where the last inequality follows from $(1-1/T)^{T}\geq 1/e$. Therefore, by setting $T=\mathcal{O}(LD^{2}/ \epsilon)$, we obtain the desired result.
\end{proof}

Proof of Thm.~\ref{thm:attack}
\begin{proof}
Firstly, we sum up all the inequalities in Eq.~\eqref{eq:attack_lem} over $k$, yielding the following
\begin{align} \notag
&\sum^{K}_{k=1}\left[\phi(\mathbf{v},\mathbf{w}_{k+1},\{\mathbf{x}_{k+1}^{i}\})- \phi(\mathbf{v}_{k},\mathbf{w}_{k+1},\{\mathbf{x}_{k+1}^{i}\}\right]\\ \notag
\geq & \frac{1}{2\eta}(\lVert \mathbf{v}_{K}-\mathbf{v}\rVert^{2}- \lVert \mathbf{v}_{0}-\mathbf{v} \rVert^{2})-\frac{K\eta M^{2}}{2}\\ \label{eq:attck_thm}
\geq & -\frac{D_{\mathcal{P}}}{2\eta} -\frac{K\eta M^{2}}{2},
\end{align}
where the last inequality comes from the bound of $\mathcal{P}$, i.e., $\lVert \mathbf{v}_{0}-\mathbf{v} \rVert^{2}\leq D_{\mathcal{P}}$ for $\mathbf{v}\in \mathcal{P}$. 
Combining the constraint set of $\mathbf{w}$ with  
$$
\phi(\mathbf{v},\mathbf{w},\{\mathbf{x}^{i}\})=\sum^{I}_{i=1}\mathbf{w}_{i}F(\mathbf{v},\mathbf{x}^{i}),
$$
we get the $m(1-1/e)$-approximate optimality of $F(\mathbf{v},\mathbf{x}^{i})$ as demonstrated in Thm.~\ref{thm:approximation}. It implies that $\{\mathbf{x}^{i}\}$ satisfies $\forall \mathbf{w}$
\begin{equation}\label{eq:appro_equa}
    \phi(\mathbf{v},\mathbf{w},\{\mathbf{x}_{k+1}^{i}\})\geq m(1-1/e) \max_{\mathbf{x}^{i}\in \mathcal{X}_{i}}\phi(\mathbf{v},\mathbf{w},\{\mathbf{x}^{i}\})-\mathcal{O}(\frac{LD^{2}}{T}).
\end{equation}
Following Step 3 in Alg.~\ref{alg:attack} and the inequality in Eq.~\eqref{eq:appro_equa}, we have 
\begin{align*}
&\phi(\mathbf{v}_{k},\mathbf{w}_{k+1},\{\mathbf{x}_{k+1}^{i}\})
    =\max_{\mathbf{w}\in \mathcal{W}}\phi(\mathbf{v}_{k},\mathbf{w},\{\mathbf{x}_{k+1}^{i}\})\\
    \geq& m(1-1/e) \max_{\mathbf{x}^{i}\in \mathcal{X}_{i}, \mathbf{w}\in \mathcal{W}}\phi(\mathbf{v}_{k},\mathbf{w},\{\mathbf{x}^{i}\})-\mathcal{O}(\frac{LD^{2}}{T}).
\end{align*}
It follows from Eq.~\eqref{eq:attck_thm} that
\begin{align*}
    &\sum^{K-1}_{k=0}\phi(\mathbf{v},\mathbf{w}_{k+1},\{\mathbf{x}_{k+1}^{i}\} \\
    &- \sum^{K-1}_{k=0}\left[m(1-1/e)\max_{\mathbf{x}^{i}\in \mathcal{X}_{i},\mathbf{w}\in\mathcal{W}}\phi(\mathbf{v}_{k},\mathbf{w},\{\mathbf{x}^{i}\})\right]\\
    \geq &-\frac{D_{\mathcal{P}}}{2\eta}-\frac{K\eta M^{2}}{2}-K\times \mathcal{O}(\frac{LD^{2}}{T}),
\end{align*}
which holds for all $\mathbf{v}\in \mathcal{P}$. In the following, we choose $\eta =\frac{1}{\sqrt{K}}$ and divide both sides by $K$, 
\begin{align*}
    &\frac{1}{K}\sum^{K-1}_{k=0}\phi(\mathbf{v},\mathbf{w}_{k+1},\{\mathbf{x}_{k+1}^{i}\})\\
    &- m(1-1/e) \frac{1}{K}\sum^{K-1}_{k=0}\left[\max_{\mathbf{x}^{i}\in \mathcal{X}_{i},\mathbf{w}\in\mathcal{W}}\phi(\mathbf{v}_{k},\mathbf{w},\{\mathbf{x}^{i}\})\right]\\
    \geq & -\frac{D_{\mathcal{P}}}{\sqrt{K}}+\frac{M^{2}}{\sqrt{K}}-\mathcal{O}(\frac{LD^{2}}{T}).
\end{align*} 
Then, combining with the convexity of $\phi(\cdot,\mathbf{w},\{\mathbf{x}^{i}\})$, we have 
\begin{align*}
    & \frac{1}{K}\sum^{K-1}_{k=0}\phi(\mathbf{v}_{k},\mathbf{w},\{\mathbf{x}^{i}\})\geq \phi(\frac{1}{K}\sum^{K-1}_{k=0}\mathbf{v}_{k},\mathbf{w},\{\mathbf{x}^{i}\}).
\end{align*}
So we get 
\begin{align*}
    &\frac{1}{K}\sum^{K-1}_{k=0}\phi(\mathbf{v},\mathbf{w}_{k+1},\{\mathbf{x}_{k+1}^{i}\})\\
    &- m(1-1/e) \max_{\mathbf{x}^{i}\in \mathcal{X}_{i},\mathbf{w}\in\mathcal{W}} \phi(\frac{1}{K}\sum^{K-1}_{k=0} \mathbf{v}_{k},\mathbf{w},\{\mathbf{x}^{i}\})\\
   \geq &-\frac{D_{\mathcal{P}}}{\sqrt{K}}+\frac{M^{2}}{\sqrt{K}}-\mathcal{O}(\frac{LD^{2}}{T}),
\end{align*}
which holds for all $\mathbf{v}\in \mathcal{P}$. 
Observing 
$$
\max_{\mathbf{w}\in \mathcal{W}, \mathbf{x}^{i}\in \mathcal{X}_{i}}\phi(\mathbf{v},\mathbf{w},\{\mathbf{x}^{i}\})\geq \phi(\mathbf{v},\mathbf{w}_{k+1},\{\mathbf{x}_{k+1}^{i}\}),
$$
and choosing $\mathbf{v}=\arg\min_{\mathbf{v}}\phi(\mathbf{v},\mathbf{w},\{\mathbf{x}^{i}\})$, we obtain 
\begin{align*}
&\min_{\mathbf{v}\in\mathcal{P}}\max_{\mathbf{w}\in \mathcal{W}, \mathbf{x}^{i}\in \mathcal{X}_{i}}\phi(\mathbf{v},\mathbf{w},\{\mathbf{x}^{i}\})\\
&-m(1-1/e) \max_{\mathbf{w}\in \mathcal{W},\mathbf{x}^{i}\in \mathcal{X}_{i}}\phi(\bar{\mathbf{v}},\mathbf{w},\{\mathbf{x}^{i}\})\\
\geq &-\frac{D_{\mathcal{P}}}{\sqrt{K}}+\frac{M^{2}}{\sqrt{K}}-\mathcal{O}(\frac{LD^{2}}{T}),
\end{align*}
where $\bar{\mathbf{v}} = \frac{1}{K}\sum^{K-1}_{k=0}\mathbf{v}_{k}$. Furthermore, with 
$$
K=\mathcal{O}(D_{\mathcal{P}}/\epsilon^{2}), \qquad  T=\mathcal{O}(LD^{2}/{\epsilon}),
$$ 
the output of Alg.~\ref{alg:attack}, $\bar{\mathbf{v}}$ satisfies
\begin{equation*}
    m(1-1/e) \max_{w\in \mathcal{W}, \mathbf{x}^{i}\in \mathcal{X}_{i}}\phi(\bar{\mathbf{v}},\mathbf{w},\{\mathbf{x}^{i}\})\leq  {\rm OPT}_{minmax} + \epsilon,
\end{equation*}
where 
$$
{\rm OPT}_{\rm minmax} = \min_{\mathbf{v}\in\mathcal{P}}\max_{\mathbf{w}\in \mathcal{W}, \mathbf{x}^{i}\in \mathcal{X}_{i}}\phi(\mathbf{v},\mathbf{w},\{\mathbf{x}^{i}\}).
$$
This means the output $\bar{\mathbf{v}}$ is an $(\alpha,\epsilon)$-approximation minmax solution with $\alpha = m(1-1/e)$ for adversarial sample generation across multiple submodular models.
\end{proof}

\subsection{\bf Proofs for Sec.\ref{sec:defense}}
\subsubsection{ Auxiliary Properties of $G$ and $\Phi$}
Under Assumption~\ref{ass:robust}, the function $G(\cdot, \mathbf{w},\{\mathbf{v}^{j}\})$ is $m$-weakly monotone and DR-submodular. The gradients of $G(\mathbf{x},\mathbf{w},\{\mathbf{v}^{j}\})$ w.r.t. $\mathbf{w}$ and $\mathbf{x}$ can be computed as follows:
\begin{eqnarray*}
    &\nabla_{w} G(\mathbf{x},\mathbf{w},\{\mathbf{v}^{j}\})=\sum^{J}_{j=1}e_{j}F(\mathbf{x},\mathbf{v}^{j})+\gamma(\mathbf{w}-\frac{\mathbf{1}}{J}),\\
    & \nabla_{\mathbf{x}} G(\mathbf{x},\mathbf{w},\{\mathbf{v}^{j}\})=\sum^{J}_{j=1}\mathbf{w}_{j}\nabla_{\mathbf{x}} F(\mathbf{x},\mathbf{v}^{j}).
\end{eqnarray*}
This implies that $G(\mathbf{x},\mathbf{w},\{\mathbf{v}^{j}\})$ is $L$-smooth and hence the following inequality holds:
\begin{align}
\notag
    & \lVert \nabla_{x} G(\mathbf{x},\mathbf{w},\{\mathbf{v}_{1}^{j}\})- \nabla_{\mathbf{x}} G(\mathbf{x},\mathbf{w},\{\mathbf{v}_{2}^{j}\}) \rVert\\ 
    \notag
    =& \sum^{J}_{j=1}w_{j}\lVert \nabla_{x} F(\mathbf{x},\mathbf{v}_{1}^{j})-\nabla_{\mathbf{x}} F(\mathbf{x},\mathbf{v}_{2}^{j})\rVert\\ \notag
    \leq& \sum^{J}_{j=1}\mathbf{w}_{j}\lVert \mathbf{v}^{j}_{1}- \mathbf{v}^{j}_{2}\rVert\\  
    \label{eq:smooth_G}
     \leq & \max_{j\in [J]}\lVert \mathbf{v}^{j}_{1}- \mathbf{v}^{j}_{2}\rVert,
\end{align}
where the last inequality follows from the fact that the sum of $\mathbf{w}_{j}$ over $j$ is equal to $1$, as required by the constraint set $\mathcal{W}$.

 To begin with, we define
\begin{align}\label{eq:defense_phi}
    \Phi(\mathbf{x})\overset{\mathrm{def}}{=} \min_{\mathbf{w}\in \mathcal{W},\mathbf{v}^{j}\in \mathcal{P}_{j}}G(\mathbf{x}, \mathbf{w},\{\mathbf{v}^{j}\})
\end{align}
and let $(\mathbf{w}^{*}(\mathbf{\mathbf{x}}),\{{\mathbf{v}^{j}}^{*}(\mathbf{x})\}) =\arg\min_{\mathbf{w}\in \mathcal{W},\mathbf{v}^{j}\in \mathcal{P}_{j}}\Phi(\mathbf{x})$. We can easily obtain the following properties regarding $\Phi(\mathbf{x})$. Since the proof is similar to that of Lemma 4.3 in~\cite{Lin2020a}, we omit it here.

\begin{lemma}\label{lem:robust_phi}
Under Assumption~\ref{ass:robust}, we have:
\begin{itemize}[leftmargin=10pt]
    \item the gradient of $\Phi$ is 
    $$
    \nabla \Phi (\mathbf{x}) = \nabla_{\mathbf{x}} G(\mathbf{x},\mathbf{w}^{*}(\mathbf{x}),\{{\mathbf{v}^{j}}^{*}(\mathbf{x})\}),
    $$
    and the function $\Phi(\mathbf{x})$ is $L_{\Phi}$-smooth, where $L_{\Phi}=(L+\frac{L^{2}}{\mu})$.
    \item Define $\mathbf{w}^{*}(\mathbf{v})=\arg\min_{\mathbf{w}\in \mathcal{W}}G(\mathbf{x}, \mathbf{w}, \{\mathbf{v}^{j}\})$ with fixed $\mathbf{x}$, the optimal solution $\mathbf{w}^{*}(\mathbf{v})$ is $\frac{L}{\mu}$-Lipschitz continuous: 
    $$
    \lVert \mathbf{w}^{*}(\mathbf{v}_{1})- \mathbf{w}^{*}(\mathbf{v}_{2}) \rVert\leq \frac{L}{\mu}\max_{j\in [J]}\lVert \mathbf{v}^{j}_{1}-\mathbf{v}^{j}_{2}\rVert,
    $$
    where $\mathbf{v}$ denotes the set of variables $\{\mathbf{v}^{j}\}$.
\end{itemize}

\end{lemma}
    
Under Assumption~\ref{ass:robust} and with the step size $\eta = 1/L$ in Alg.~\ref{alg:robust}, we have the following inequality according to the strongly convex with respect to $\mathbf{v}^{j}$, which states that the iteration $\mathbf{v}^{j}_{T}(\mathbf{x}_{k})$ satisfies
\begin{equation}\label{strongly convex}
\lVert \mathbf{v}^{j}_{T}(\mathbf{x}_{k}) - {\mathbf{v}^{j}}^{*}(\mathbf{x}_{k})\rVert \leq (1-\frac{\mu}{L})^{T/2}D_{\mathcal{P}_{j}},
\end{equation}
where $\mathbf{v}^{j^{*}}(\mathbf{x}_{k})=\arg\min_{\mathbf{v}\in \mathcal{P}_{j}}F(\mathbf{x}_{k}, \mathbf{v}^{j}(\mathbf{x}_{k}))$.
    
Similarly, since $G(\mathbf{x},\mathbf{w},\{\mathbf{v}^{j}\})$ is $\gamma$-strongly convex and $\gamma$-smooth regarding $\mathbf{w}$, the optimal vector $\mathbf{w}^{*}(\mathbf{x})$ can be determined exactly at each iterate $k$. The following lemma characterizes the relationship between the function values at two consecutive iteration points. 

\begin{lemma}\label{lem:robust_difference}
Let 
$$ 
\Delta_{k}=\nabla \Phi(\mathbf{x}_{k})-\nabla G_{\mathbf{x}}(\mathbf{x}_{k},  \mathbf{w}^{*}(\{\mathbf{v}^{j}_{T}\}_{k}),\{\mathbf{v}_{T}^{j}(\mathbf{x}_{k})\}),
$$ 
where $\Phi$ is defined as in Eq.~\eqref{eq:defense_phi}. Then, under Assumption~\ref{ass:robust}, 
the sequence $\{\mathbf{x}_{k}\}_{k\geq1}$ generated by Alg.~\ref{alg:robust} satisfies the following inequality:
\begin{align*}
&\Phi(\mathbf{x}_{k+1})- \Phi(\mathbf{x}_{k})\\
\geq &\frac{1}{K}\left(m\Phi(\mathbf{x}^{*})-\Phi(\mathbf{x}_{k}) - 2D_{\mathcal{X}}\lVert\Delta_{k}\rVert -\frac{L_{\Phi}}{2K}D_{\mathcal{X}}^{2}\right),
\end{align*}
where $\mathbf{x}^{*}$ is an optimal solution, $L_{\phi}$ is the smoothness parameter of $\phi(\mathbf{x})$, and $D_{\mathcal{X}}$ is the diameter of the constraint set $\mathcal{X}$.
\end{lemma}
\begin{proof}
We begin the proof by leveraging the $L_{\Phi}$-smoothness of $\Phi(\mathbf{x})$ and   
$$
\nabla \Phi(\mathbf{x}_{k}) = \nabla G_{\mathbf{x}}(\mathbf{x}_{k}, \mathbf{w}^{*}(\{\mathbf{v}^{j}_{T}\}_{k}),\{{\mathbf{v}^{j}}^{*}(\mathbf{x}_{k})\})
$$
as stated in Lem.~\ref{lem:robust_phi}, it holds
\begin{align}
    \notag
    & \Phi(\mathbf{x}_{k+1})-\Phi(\mathbf{x}_{k})\\ 
    \notag
    \geq & \left \langle \nabla \Phi(\mathbf{x}_{k}), \mathbf{x}_{k+1}-\mathbf{x}_{k}\right \rangle-\frac{L_{\Phi}}{2}\lVert \mathbf{x}_{k+1}-\mathbf{x}_{k}\rVert^{2}_{2}\\ 
    \notag
    \geq & \left \langle \nabla \Phi(\mathbf{x}_{k}), \mathbf{x}_{k+1}-\mathbf{x}_{k}\right \rangle-\frac{L_{\Phi}}{2}\lVert\mathbf{x}_{k+1}-\mathbf{x}_{k}\rVert^{2}_{2} \\
    \label{eq:defense_diff_step8}
    = & \frac{1}{K}\left \langle\nabla \Phi(\mathbf{x}_{k}), \mathbf{d}_{k}\right \rangle-\frac{L_{\Phi}}{2K^{2}}\lVert \mathbf{d}_{k}\rVert^{2}_{2} \\
    \notag
    = & \frac{1}{K}\left \langle \nabla G_{\mathbf{x}}(\mathbf{x}_{k},  \mathbf{w}^{*}(\{\mathbf{v}^{j}_{T}\}_{k}),\{\mathbf{v}_{T}^{j}(\mathbf{x}_{k})\}), \mathbf{d}_{k}\right \rangle \\
    \notag
    & + \frac{1}{K}\left \langle \Delta_{k}, \mathbf{d}_{k}\right \rangle -\frac{L_{\Phi}}{2K^{2}}\lVert \mathbf{d}_{k}\rVert^{2}_{2}\\ 
    \notag
    = & \frac{1}{K}\left \langle [\nabla G_{\mathbf{x}}(\mathbf{x}_{k})]_{+}, \mathbf{d}_{k}\right \rangle + \frac{1}{K}\left \langle \Delta_{k}, \mathbf{d}_{k}\right \rangle -\frac{L_{\Phi}}{2K^{2}}\lVert \mathbf{d}_{k}\rVert^{2}_{2} \\ 
    \label{eq:defense_lem_diff_1}
    \geq & \frac{1}{K}\left \langle [\nabla G_{\mathbf{x}}(\mathbf{x}_{k})]_{+},  \mathbf{x}^{*}\right \rangle + \frac{1}{K}\left \langle \Delta_{k}, \mathbf{d}_{k}\right \rangle -\frac{L_{\Phi}}{2K^{2}}\lVert \mathbf{d}_{k}\rVert^{2}_{2} \\
    \notag
    = & \frac{1}{K}\left \langle [\nabla \Phi(\mathbf{x}_{k})]_{+}, \mathbf{x}^{*}\right \rangle + \frac{1}{K}\left(\langle \Delta_{k}, \mathbf{d}_{k}\rangle -\langle \bar{ \Delta}_{k}, \mathbf{x}^{*}\rangle\right)\\  \notag
    &-\frac{L_{\Phi}}{2K^{2}}\lVert\mathbf{d}_{k}\rVert^{2}_{2},
\end{align}
where Eq.~\eqref{eq:defense_diff_step8} holds following Step 7 of Alg.~\ref{alg:robust}, Eq.~\eqref{eq:defense_lem_diff_1} follows  Step 6 of Alg.~\ref{alg:robust}, the second equality holds since $\Delta_{k}=\nabla \Phi(\mathbf{x}_{k})-\nabla G_{\mathbf{x}}(\mathbf{x}_{k}, \mathbf{w}^{*}(\{\mathbf{v}^{j}_{T}\}_{k}),\{\mathbf{v}_{T}^{j}(\mathbf{x}_{k})\})$, and the third equality is due to Eq.~\eqref{eq:pre_defense}. Moreover, the second inequality comes from the linear maximization in Step 6, i.e.,
$$
\langle [\nabla G_{\mathbf{x}}(\mathbf{x}_{k})]_{+}, \mathbf{d}_{k} \rangle\geq \langle [\nabla G_{\mathbf{x}}(\mathbf{x}_{k})]_{+}, \mathbf{x}^{*}\rangle,
$$ 
where 
$$
\mathbf{x}^{*}\overset{\mathrm{def}}{=}\arg\max_{\mathbf{x}\in\mathcal{X}}\Phi(\mathbf{x})=\arg\max_{\mathbf{x}\in \mathcal{X}}\min_{\mathbf{w}\in \mathcal{W},\mathbf{v}^{j}\in \mathcal{P}_{j}}G(\mathbf{x},\mathbf{w},\{\mathbf{v}^{j}\}).
$$
The last inequality is from $\mathbf{x}^{*} \geq \mathbf{0}$, and the  $\bar{\Delta}_{k}$ in the last equality is defined as
$$
\bar{\Delta}_{k} = [\nabla \Phi(\mathbf{x}_{k})]_{+} - [\nabla G_{\mathbf{x}}(\mathbf{x}_{k}, \mathbf{w}^{*}(\{\mathbf{v}^{j}_{T}\}_{k}),\{\mathbf{v}_{T}^{j}(\mathbf{x}_{k})\})]_{+}.
$$
Then, focusing on the first term on the right-hand side of Inequality~\eqref{eq:defense_lem_diff_1}, we obtain from the DR-submodular of $F(\mathbf{x},{\mathbf{v}^{j}}^{*}(\mathbf{x}))$
\begin{align} 
\notag
    &\langle [\nabla \Phi(\mathbf{x}_{k})]_{+}, \mathbf{x}^{*}\rangle \\ \notag
    \geq &\langle [\nabla \Phi(\mathbf{x}_{k})]_{+}, \mathbf{x}^{*} \vee \mathbf{x}_{k}- \mathbf{x}_{k}\rangle \\ 
    \notag
    \geq &\langle \nabla G_{\mathbf{x}}(\mathbf{x}_{k}, \mathbf{w}^{*}(\mathbf{x}_{k}),\{{\mathbf{v}^{j}}^{*}(\mathbf{x}_{k})\}), \mathbf{x}^{*} \vee \mathbf{x}_{k}-\mathbf{x}_{k}\rangle\\ 
    \notag
    \geq & G(\mathbf{x}^{*}\vee \mathbf{x}_{k}, \mathbf{w}^{*}(\mathbf{x}_{k}),\{{\mathbf{v}^{j}}^{*}(\mathbf{x}_{k})\})-G(\mathbf{x}_{k},\mathbf{w}^{*}(\mathbf{x}_{k}), \{{\mathbf{v}^{j}}^{*}(\mathbf{x}_{k})\})\\ 
    \notag
    \geq & m G(\mathbf{x}^{*}, \mathbf{w}^{*}(\mathbf{x}_{k}),\{{\mathbf{v}^{j}}^{*}(\mathbf{x}_{k})\})-G(\mathbf{x}_{k},\mathbf{w}^{*}(\mathbf{x}_{k}), \{{\mathbf{v}^{j}}^{*}(\mathbf{x}_{k})\})\\   
    \label{eq:defense_lem_diff_2}
    \geq &m \Phi(\mathbf{x}^{*})-\Phi(\mathbf{v}_{k}).
\end{align}
Here, the first inequality follows from  $\mathbf{x}^{*} \vee \mathbf{x}_{k}\leq  \mathbf{x}^{*} + \mathbf{x}_{k}$ for $\mathbf{x}^{*}, \mathbf{x}_{k}\in \mathcal{X}$, the second one comes from the definition of $\Phi(\mathbf{x})$ in Lem.~\ref{lem:robust_phi} and $\mathbf{x}^{*} \vee \mathbf{x}_{k}-\mathbf{x}_{k} \geq \mathbf{0}$, the third inequality is obtained using the DR-submodularity of $G(\cdot, \mathbf{w},\{\mathbf{v}^{j}\})$ and Lem.~\ref{lem:DR_equal} (iii) with $\mathbf{x}^{*} \vee \mathbf{x}_{k} \geq \mathbf{x}_{k}$, and the fourth follows from the $m$-weak monotonicity of $G(\cdot, \mathbf{w},\{\mathbf{v}^{j}\})$ as along with the fact 
$$
G(\mathbf{x}^{*}, \mathbf{w}^{*}(\mathbf{x}_{k}),\{{\mathbf{v}^{j}}^{*}(\mathbf{x}_{k})\})\geq \Phi(\mathbf{x}^{*}),  
$$ 
where $\Phi(\mathbf{x}^{*}) = \max_{\mathbf{x}\in \mathcal{X}}\min_{\mathbf{w}\in \mathcal{W},\mathbf{v}^{j}\in \mathcal{P}_{j}} G(\mathbf{x},\mathbf{w},\{\mathbf{v}^{j}\})$ in the last inequality. Combining Eqs.~\eqref{eq:defense_lem_diff_1} and \eqref{eq:defense_lem_diff_2}, we obtain 
\begin{align*}
    \notag
    &\Phi(\mathbf{x}_{k+1})- \Phi(\mathbf{x}_{k})\\ 
    \geq & \frac{1}{K}(m \Phi(\mathbf{x}^{*})-\Phi(\mathbf{x}_{k})) + \frac{1}{K}\left(\langle \Delta_{k}, \mathbf{d}_{k}\rangle -\langle \bar{ \Delta}_{k}, \mathbf{x}^{*}\rangle\right) \\
    &-\frac{L_{\Phi}}{2K^{2}}\lVert \mathbf{d}_{k}\rVert^{2}_{2}\\ 
    \geq & \frac{1}{K}(m\Phi(\mathbf{x}^{*})-\Phi(\mathbf{x}_{k})) - \frac{2D_{\mathcal{X}}}{K}\lVert\Delta_{k}\rVert -\frac{L_{\Phi}}{2K^{2}}D_{\mathcal{X}}^{2},
\end{align*}
where the second inequality holds by Cauchy-Schwartz inequality, i.e., 
$$
\lVert a \rVert \lVert b \rVert \geq \langle a, b\rangle\geq - \lVert a \rVert \lVert b \rVert,
$$ 
$\mathbf{d}_{k},\mathbf{x}^{*}\in \mathcal{X}$ and $\lVert \bar{\Delta}_{k}\rVert \leq \lVert \Delta_{k}\rVert$. 
   
In fact, for two sets  $A=\{i: [\nabla \Phi(\mathbf{x}_{k})]_{i}\geq 0\}$ and $B=\{i: [\nabla G_{\mathbf{x}}(\mathbf{x}_{k}, \mathbf{w}^{*}(\{\mathbf{v}^{j}_{T}\}_{k}),\{\mathbf{v}_{T}^{j}(\mathbf{x}_{k})\})]_{i}\geq 0\}$, it is obvious that 
$$
\lvert [\bar{\Delta}_{k}]_{i}\rvert
\begin{cases}
     =\lvert [\Delta_{k}]_{i}\rvert, & i\in A\cap B\\
    \leq \lvert [\Delta_{k}]_{i}\rvert, &i\in A\cup B-A\cap B\\
    = 0\leq [\Delta_{k}]_{i}, &\text{Otherwise}\\
\end{cases}
$$
Thus, we have $\lVert \bar{\Delta}_{k}\rVert \leq \lVert \Delta_{k}\rVert$ and complete the proof.
 \end{proof}    
\subsubsection{ Proof of Theorem~\ref{thm:robust}}

Based on Lem.~\ref{lem:robust_difference}, we present the proof of Thm.~\ref{thm:robust} as follows:
\begin{proof}
First, we bound $\lVert \Delta_{k} \rVert$ on the right hand side as follows
\begin{align*}
     \lVert \Delta_{k} \rVert=& \lVert\nabla \Phi(\mathbf{x}_{k}) -\nabla G_{\mathbf{x}}(\mathbf{x}_{k},  \mathbf{w}^{*}(\{\mathbf{v}^{j}_{T}\}_{k}),\{\mathbf{v}_{T}^{j}(\mathbf{x}_{k})\})\rVert\\
    \leq & \lVert \nabla_{\mathbf{x}} G(\mathbf{x}_{k}, \mathbf{w}^{*}(\mathbf{x}_{k}),\{\mathbf{v}^{j^{*}}(\mathbf{x}_{k})\})\\
    &-\nabla G_{\mathbf{x}}(\mathbf{x}_{k}, \mathbf{w}^{*}(\mathbf{x}_{k}),\{\mathbf{v}_{T}^{j}(\mathbf{x}_{k})\})\rVert \\
    & +\lVert\nabla G_{\mathbf{x}}(\mathbf{x}_{k}, \mathbf{w}^{*}(\mathbf{x}_{k}),\{\mathbf{v}_{T}^{j}(\mathbf{x}_{k})\})\\
    &-\nabla G_{\mathbf{x}}(\mathbf{x}_{k}, \mathbf{w}^{*}(\{\mathbf{v}^{j}_{T}\}_{k}),\{\mathbf{v}_{T}^{j}(\mathbf{x}_{k})\})\rVert\\
    \leq &L \max_{j\in [J]}\lVert {\mathbf{v}^{j}}^{*}(\mathbf{x}_{k}) - \mathbf{v}_{T}^{j}(\mathbf{x}_{k}))\rVert \\
    & + \gamma \lVert \mathbf{w}^{*}(x_{k}) - \mathbf{w}^{*}(\{\mathbf{v}^{j}_{T}\}_{k}) \rVert \\
    \leq &(L+\gamma \frac{L}{\mu}) \max_{j\in [J]}\lVert {\mathbf{v}^{j}}^{*}(\mathbf{x}_{k}) - \mathbf{v}_{T}^{j}(\mathbf{x}_{k}))\rVert  \\
    \leq & (L+\gamma \frac{L}{\mu})(1-\frac{\mu}{L})^{T/2}D_{\mathcal{P}},
\end{align*}
where the second inequality is derived based on combining the inequality in Eq.~\eqref{eq:smooth_G} and $\gamma$-smoothness of $G(\mathbf{x},\mathbf{w},\{\mathbf{v}^{j}\})$ w.r.t. $\mathbf{w}$, the third inequality holds by combining $\frac{L}{\mu}$-Lipschitz continuous of $\mathbf{w}^{*}$ in Lem.~\ref{lem:robust_phi} with 
$$
\mathbf{w}^{*}(\mathbf{x}_{k})=\arg\max_{\mathbf{w}\in \mathcal{W}}G(\mathbf{x}_{k},\mathbf{w},\{\mathbf{v}^{j^{*}}(\mathbf{x}_{k})\}),
$$
and the last inequality follows from \eqref{strongly convex} and $D_{\mathcal{P}} = \max_{j}\{D_{\mathcal{P}_{j}}\}$.
        
Therefore, by setting $Q=(L+\gamma \frac{L}{\mu})$ and rearranging the inequality in Lem.~\ref{lem:robust_difference} through subtracting $m \Phi(\mathbf{x}^{*})$ from both sides, we obtain 
\begin{align}\label{eq:defense_thm}
\Phi(\mathbf{x}_{k+1})-m\Phi(\mathbf{x}^{*})\geq& (1-\frac{1}{K})(\Phi(\mathbf{x}_{k})-m\Phi(\mathbf{x}^{*})) \\ \notag
&-\frac{2Q D_{\mathcal{P}}D_{\mathcal{X}}}{K}(1-\frac{\mu}{L})^{T/2}-\frac{L_{\Phi}^{2}D^{2}_{\mathcal{X}}}{2K^{2}}.
\end{align}
Now, we telescope Eq.~\eqref{eq:defense_thm} over all the iterations for $k=0,\ldots,K-1$, which yields 
\begin{align*}
    \Phi(\mathbf{x}_{K})-m\Phi(\mathbf{x}^{*})
    \geq& (1-\frac{1}{K})^{K}(\Phi(\mathbf{x}_{0})-m\Phi(\mathbf{x}^{*})) \\
    &-\sum^{K-1}_{k=0}\frac{2Q D_{\mathcal{P}}D_{\mathcal{X}}}{K}\rho^{T/2} -\frac{L^{2}D^{2}_{\mathcal{X}}}{2K}\\
    \geq &(1-\frac{1}{K})^{K}(\Phi(\mathbf{x}_{0})-m\Phi(\mathbf{x}^{*}))\\
    &-2QD_{\mathcal{X}} D_{\mathcal{P}}\rho ^{T/2}-\frac{L_{\Phi}^{2}D^{2}_{\mathcal{X}}}{2K},
    \end{align*}
where $\rho = 1-\frac{\mu}{L}$. Furthermore, when 
$$
K\geq \frac{L_{\Phi}^{2}D^{2}_{\mathcal{X}}}{2\epsilon} \qquad \text{and} \qquad T\geq 2\log_{\rho}\frac{\epsilon}{2QD_{\mathcal{X}} D_{\mathcal{P}}},
$$
it holds that 
\begin{align*}
    \Phi(\mathbf{x}_{K})\geq &m\left(1-(1-\frac{1}{K})^{K}\right)\Phi(\mathbf{x}^{*}) - \epsilon\\
    \geq & m\left(1-1/e\right)\Phi(\mathbf{x}^{*})-\epsilon,
\end{align*}
where the second inequality is because of  $\Phi(\mathbf{x}_{0})\geq 0$ by definition and $(1-\frac{1}{K})^{K}\leq 1/e$ by simple calculation.
\end{proof}

\subsection{\bf Additional Experimental Details and Results}
\subsubsection{Detailed Definitions of Evaluation Metrics}
\label{app:metrics}
In this section, we provide the detailed definitions of the performance metrics used in the experiments. Unless otherwise specified, all attacked solutions are evaluated under the original clean objective, following the shared evaluation protocol adopted in the main text. For notational simplicity, we omit the explicit dependence of attacked solutions on $(p,\epsilon_p,D)$ when no confusion arises.

\paragraph{Attack Metrics}
Let $\mathbf{x}_{T,i}$ denote the clean-reference solution produced by Alg.~\ref{alg:Mgreedy} for model $i$ under the original similarity structure $\Omega_i$. For an attack type $p\in\{1,2,\infty\}$, attack budget $\epsilon_p$, and feasible-set upper bound $D$, let $\mathbf{x}^{\rm att}_{T,i}$ denote the attacked solution obtained by applying Alg.~\ref{alg:Mgreedy} to the perturbed data. We evaluate the attack effect under the clean objective $F(\cdot,\Omega_i)$.

{\bf Success Ratio of Attack.}
We first define the degradation of model $i$ under attack as
\begin{equation*}
\Delta^{\rm att}_{i}(p,\epsilon_{p},D)
=
F(\mathbf{x}_{T,i},\Omega_{i})
-
F(\mathbf{x}^{\rm att}_{T,i},\Omega_{i}).
\end{equation*}

A target model is regarded as successfully attacked if
\[
\Delta^{\rm att}_{i}(p,\epsilon_{p},D)\ge 0,
\]
that is, the attacked solution yields no larger clean-objective value than the clean-reference solution. Accordingly, the success ratio is defined as
\begin{equation*}
\textbf{Ratio}^{\rm att}(p,\epsilon_{p},D)
=
\frac{1}{I}\sum_{i=1}^{I}
\mathbf{1}\!\left\{
\Delta^{\rm att}_{i}(p,\epsilon_{p},D)\ge 0
\right\}.
\end{equation*}
A larger value of $\textbf{Ratio}^{\rm att}(p,\epsilon_p,D)$ indicates that the generated perturbation successfully degrades a larger fraction of target models.

{\bf Average Degradation and Attack Intensity.}
To measure the average attack effect over all target models, we define the average degradation as
\begin{equation*}
\bar{\Delta}^{\rm att}(p,\epsilon_{p},D)
=
\frac{1}{I}\sum_{i=1}^{I}
\left(
F(\mathbf{x}_{T,i},\Omega_{i})
-
F(\mathbf{x}^{\rm att}_{T,i},\Omega_{i})
\right).
\end{equation*}

Based on this quantity, we define the normalized attack intensity as
\begin{equation*}
\textbf{Intensity}^{\rm att}(p,\epsilon_{p},D)
=
\frac{\bar{\Delta}^{\rm att}(p,\epsilon_{p},D)}
{\frac{1}{I}\sum_{i=1}^{I}F(\mathbf{x}_{T,i},\Omega_{i})}.
\end{equation*}

A larger value of $\bar{\Delta}^{\rm att}(p,\epsilon_{p},D)$ or $\textbf{Intensity}^{\rm att}(p,\epsilon_{p},D)$ indicates a stronger attack and hence greater vulnerability of the clean-reference summarization solver.

\paragraph{Defense Metrics}

We use three metrics to evaluate the robust algorithm: {\bf Loss}, {\bf Robustness}, and {\bf Mitigation}. Let $\mathbf{x}^{\rm rob}_{T,i}$ denote the solution produced by Alg.~\ref{alg:robust} under the clean setting, and let $\mathbf{x}^{\rm rob,att}_{T,i}$ denote the robust solution obtained under attack. We continue to evaluate all quantities below under the clean objective $F(\cdot,\Omega_i)$.

{\bf Loss.}
This metric measures the clean-data utility loss caused by using the robust algorithm instead of the plain greedy algorithm. It is defined as
\begin{equation*}
{\bf Loss}
=
\frac{1}{I}\sum_{i=1}^{I}
\frac{
F(\mathbf{x}_{T,i},\Omega_i)
-
F(\mathbf{x}^{\rm rob}_{T,i},\Omega_i)
}{
F(\mathbf{x}_{T,i},\Omega_i)
}.
\end{equation*}
A smaller value of {\bf Loss} indicates that the robust algorithm preserves the clean-data utility more effectively.

{\bf Robustness.}
This metric measures the sensitivity of the robust solution to attacks by comparing the robust algorithm under clean and attacked settings. It is defined as
\begin{equation*}
{\bf Robustness}
=
\frac{1}{I}\sum_{i=1}^{I}
\frac{
\left|
F(\mathbf{x}^{\rm rob}_{T,i},\Omega_i)
-
F(\mathbf{x}^{\rm rob,att}_{T,i},\Omega_i)
\right|
}{
F(\mathbf{x}^{\rm rob}_{T,i},\Omega_i)
}.
\end{equation*}
A smaller value of {\bf Robustness} indicates that the robust algorithm is more stable under attacks.

{\bf Mitigation.}
This metric measures how much the robust algorithm reduces the degradation caused by attacks, compared with the attacked greedy solution. It is defined as
\begin{equation*}
{\bf Mitigation}
=
\frac{1}{I}\sum_{i=1}^{I}
\frac{
F(\mathbf{x}^{\rm rob,att}_{T,i},\Omega_i)
-
F(\mathbf{x}^{\rm att}_{T,i},\Omega_i)
}{
F(\mathbf{x}_{T,i},\Omega_i)
}.
\end{equation*}
A larger value of {\bf Mitigation} indicates that the robust algorithm more effectively offsets the damage caused by attacks. In particular, a positive value means that, on average, the robust algorithm achieves a better clean-objective value under attack than the attacked greedy solution.

\paragraph{Structural Metrics}

For the controlled clustered benchmark, we further evaluate the structural quality of the rounded summary. 
Let \(\mathcal{C}=\{C_1,\ldots,C_m\}\) be the set of ground-truth clusters, and let \(H\) denote the set of designated hub representatives, with one hub representative for each cluster. 
Given a rounded top-\(k\) summary \(S_k\), we report three structural metrics: {\bf Coverage}, {\bf Hit}, and {\bf Redundancy}.

{\bf Coverage.}
Coverage measures the fraction of ground-truth clusters represented by the selected summary. 
It is defined as
\[
\mathrm{Cov}(S_k)
=
\frac{
|\{c:S_k\cap C_c\neq\emptyset\}|
}{m}.
\]
A larger value of \(\mathrm{Cov}(S_k)\) indicates that the selected summary covers more clusters and is therefore more representative at the cluster level.

{\bf Hit.}
Hit measures whether the selected summary contains the designated hub representatives. 
It is defined as
\[
\mathrm{Hit}(S_k)
=
\frac{|S_k\cap H|}{m}.
\]
A larger value of \(\mathrm{Hit}(S_k)\) indicates that the summary selects more true hub representatives.

{\bf Redundancy.}
Redundancy measures repeated selection within already covered clusters. 
It is defined as
\[
\mathrm{Red}(S_k)
=
1-
\frac{
|\{c:S_k\cap C_c\neq\emptyset\}|
}{k}.
\]
A smaller value of \(\mathrm{Red}(S_k)\) indicates that the selected summary contains less repeated information from the same clusters. 
Therefore, higher Coverage and Hit values and lower Redundancy values correspond to better structural summary quality.

\paragraph{Downstream Metrics}

For the controlled clustered benchmark, we further evaluate whether the rounded summary remains useful for downstream classification. 
Each cluster is treated as one class, and the selected top-\(k\) summary \(S_k\) is used as a small training set. 
All downstream evaluation is performed on clean test samples under the clean similarity matrix, so that the reported performance reflects the usefulness of the selected summary rather than a change in the test distribution. 
We report two downstream metrics: {\bf NN Acc.} and {\bf Recovery}.

{\bf NN Acc.}
NN Acc. denotes nearest-neighbor classification accuracy. 
For each clean test sample, we assign the label of the most similar selected summary item under the clean similarity matrix. 
The resulting classification accuracy is reported as NN Acc. 
A larger value of NN Acc. indicates that the selected summary is more useful for downstream classification.

{\bf Recovery.}
Recovery measures the fraction of downstream accuracy loss recovered relative to the attacked summary. 
It is defined as
\[
\mathrm{Recovery}
=
\frac{
\mathrm{Acc}_{\mathrm{method}}-\mathrm{Acc}_{\mathrm{attack}}
}{
\mathrm{Acc}_{\mathrm{clean}}-\mathrm{Acc}_{\mathrm{attack}}
}.
\]
Here, \(\mathrm{Acc}_{\mathrm{clean}}\) is the accuracy obtained by the clean summary, and \(\mathrm{Acc}_{\mathrm{attack}}\) is the accuracy obtained by the attacked summary. 
Thus, \(\mathrm{Recovery}=0\) corresponds to the attacked summary, while \(\mathrm{Recovery}=1\) corresponds to full recovery to the clean-summary accuracy. 
A larger value of Recovery indicates stronger recovery of downstream task performance.

\subsubsection{Additional Empirical Results on multilinear extension summarization}

\paragraph{Real-data sensitivity on the multilinear-extension summarization model}

\begin{table}[t]
\centering
\caption{Real-data sensitivity to the attack budget \(\epsilon\) on the multilinear-extension summarization model. 
We use \(I=10\), \(n=50\), \(k=5\), \(\lambda=1\), \(K_{\rm attack}=30\), and \(T_{\rm attack}=20\). 
Random perturbation is averaged over 20 independent trials.}
\label{tab:realdata_epsilon_sensitivity}
\scriptsize
\setlength{\tabcolsep}{2.4pt}
\resizebox{\columnwidth}{!}{
\scalebox{0.9}{
\begin{tabular}{c c l c c c}
\toprule
\textbf{Norm} 
& \(\boldsymbol{\epsilon}\)
& \textbf{Method}
& \textbf{Avg. Deg.} $\uparrow$
& \textbf{Succ.} $\uparrow$
& \textbf{Int.} $\uparrow$ \\
\midrule
\multirow{9}{*}{\(\ell_1\)}
& \multirow{3}{*}{0.1}
& Proposed attack    & -0.0158 & 0.20 & -0.000417 \\
& & PGD   & -0.0706 & 0.20 & -0.001864 \\
& & Rand. & -0.0343 & 0.00 & -0.000907 \\
\cmidrule(lr){2-6}
& \multirow{3}{*}{1}
& Proposed attack    &  0.0162 & 0.70 &  0.000428 \\
& & PGD   &  0.0084 & 0.70 &  0.000221 \\
& & Rand. & -0.0333 & 0.00 & -0.000881 \\
\cmidrule(lr){2-6}
& \multirow{3}{*}{2}
& Proposed attack    &  0.0329 & 0.60 &  0.000869 \\
& & PGD   &  0.0097 & 0.70 &  0.000255 \\
& & Rand. & -0.0355 & 0.10 & -0.000938 \\
\midrule
\multirow{12}{*}{\(\ell_2\)}
& \multirow{3}{*}{0.1}
& Proposed attack    & -0.0363 & 0.10 & -0.000960 \\
& & PGD   & -0.0327 & 0.00 & -0.000863 \\
& & Rand. & -0.0370 & 0.20 & -0.000978 \\
\cmidrule(lr){2-6}
& \multirow{3}{*}{1}
& Proposed attack    & -0.0662 & 0.30 & -0.001748 \\
& & PGD   & -0.0353 & 0.30 & -0.000934 \\
& & Rand. & -0.0599 & 0.40 & -0.001582 \\
\cmidrule(lr){2-6}
& \multirow{3}{*}{2}
& Proposed attack    &  0.0548 & 0.70 &  0.001448 \\
& & PGD   &  0.0001 & 0.70 &  0.000002 \\
& & Rand. & -0.0192 & 0.40 & -0.000508 \\
\cmidrule(lr){2-6}
& \multirow{3}{*}{5}
& Proposed attack    &  0.0548 & 0.70 &  0.001448 \\
& & PGD   &  0.0001 & 0.70 &  0.000002 \\
& & Rand. &  0.1407 & 0.70 &  0.003716 \\
\midrule
\multirow{12}{*}{\(\ell_\infty\)}
& \multirow{3}{*}{0.02}
& Proposed attack    & -0.0747 & 0.00 & -0.001973 \\
& & PGD   & -0.0307 & 0.10 & -0.000812 \\
& & Rand. & -0.0505 & 0.40 & -0.001333 \\
\cmidrule(lr){2-6}
& \multirow{3}{*}{0.05}
& Proposed attack    & -0.0728 & 0.20 & -0.001923 \\
& & PGD   & -0.0799 & 0.00 & -0.002110 \\
& & Rand. & -0.0292 & 0.40 & -0.000770 \\
\cmidrule(lr){2-6}
& \multirow{3}{*}{0.08}
& Proposed attack    & -0.0618 & 0.40 & -0.001632 \\
& & PGD   & -0.0529 & 0.20 & -0.001396 \\
& & Rand. & -0.0045 & 0.60 & -0.000119 \\
\cmidrule(lr){2-6}
& \multirow{3}{*}{0.1}
& Proposed attack    & -0.0972 & 0.40 & -0.002567 \\
& & PGD   & -0.0284 & 0.20 & -0.000750 \\
& & Rand. &  0.0093 & 0.60 &  0.000246 \\
\bottomrule
\end{tabular}
}
}
\vspace{-6pt}
\end{table}

Table~\ref{tab:realdata_epsilon_sensitivity} reports the sensitivity of real-data attacks to the perturbation budget \(\epsilon\) under three norm constraints. 
Under the \(\ell_1\)-norm constraint, the proposed attack becomes effective when \(\epsilon\) increases from \(0.1\) to \(1\), and achieves the largest average degradation and attack intensity at \(\epsilon=2\). 
Although PGD obtains comparable success ratios at \(\epsilon=1\) and \(\epsilon=2\), its degradation magnitude and intensity are much smaller. 
Random perturbation remains negative across all tested \(\ell_1\) budgets, indicating that arbitrary feasible perturbations do not reliably reduce the clean summarization utility.

Under the \(\ell_2\)-norm constraint, all methods are weak at small budgets, while the proposed attack becomes effective at \(\epsilon=2\). 
At the larger budget \(\epsilon=5\), random perturbation also produces positive degradation, suggesting that sufficiently large unstructured changes can disrupt the similarity model. 
However, this effect is not stable in smaller-budget regimes and does not provide a controlled multi-target attack mechanism.

The \(\ell_\infty\)-norm results show a different pattern. 
Across the tested budgets, the proposed attack and PGD attacks do not yield consistently positive average degradation, and their attack intensities remain small or negative. 
Random perturbation becomes positive only at \(\epsilon=0.1\), but the magnitude is also small. 
This suggests that coordinate-wise bounded perturbations are less effective for the MovieLens similarity structure than the \(\ell_1\)- and \(\ell_2\)-budgeted attacks. 
Therefore, we treat the \(\ell_\infty\) case as an additional sensitivity analysis rather than the main real-data attack setting.

Overall, the results indicate that the proposed structure-aware attack is most meaningful in low-to-moderate \(\ell_1\)- and \(\ell_2\)-budget regimes, where optimized perturbations outperform arbitrary feasible perturbations in both degradation magnitude and attack intensity.

\paragraph{Per-model real-data attack results}
Table~\ref{tab:me_attack_per_model} reports the per-model attack effects on the real-data MovieLens instance under representative \(\ell_1\)- and \(\ell_2\)-norm constraints. 
The results show that the degradation is not uniformly distributed across all target models, which is expected in the multi-target setting because each victim model has a different similarity structure. 
Nevertheless, the proposed attack produces clear positive degradation on several individual models and achieves larger average degradation than the PGD baseline under both representative constraints. 
In contrast, random perturbation often yields near-zero or negative degradation, indicating that simply perturbing the similarity matrix does not reliably reduce the clean-objective value. 
These per-model results therefore support the aggregate results in Table~\ref{tab:me_attack_baseline}: the observed degradation is primarily due to the structured attack optimization rather than arbitrary perturbations.

A negative degradation means that the attacked solution obtains a slightly larger clean-objective value than the clean reference on that individual model. This can occur because the perturbation is shared across all target models and is optimized for the aggregate multi-target objective rather than for each model independently.
\begin{table*}[t]
\centering
\caption{Per-model real-data attack results under representative norm constraints. \(F_{\rm clean}\) denotes the clean-reference objective; each method column reports attacked objective / degradation.}
\label{tab:me_attack_per_model}
\scalebox{1.2}{
\begin{tabular}{c c c c c c}
\hline
Norm & Model & \(F_{\rm clean}\) & Proposed attack & PGD baseline & Random perturbation \\
\hline
\multirow{10}{*}{$\ell_1$} & 1 & 41.8625 & 41.5888 / 0.2737 & 41.8251 / 0.0373 & 41.8688 / -0.0064 \\
 & 2 & 40.8258 & 41.2047 / -0.3788 & 41.1852 / -0.3594 & 40.8614 / -0.0356 \\
 & 3 & 38.7825 & 38.7469 / 0.0357 & 38.7701 / 0.0124 & 38.7823 / 0.0002 \\
 & 4 & 37.1577 & 37.0388 / 0.1190 & 37.4117 / -0.2540 & 37.2135 / -0.0557 \\
 & 5 & 35.0018 & 34.9913 / 0.0104 & 34.9269 / 0.0749 & 35.0040 / -0.0023 \\
 & 6 & 36.9518 & 36.6008 / 0.3509 & 36.6484 / 0.3033 & 37.1697 / -0.2179 \\
 & 7 & 36.2921 & 36.3042 / -0.0120 & 36.1945 / 0.0976 & 36.3252 / -0.0330 \\
 & 8 & 39.9055 & 40.0042 / -0.0987 & 39.6352 / 0.2702 & 39.9055 / 0.0000 \\
 & 9 & 36.9493 & 36.9793 / -0.0301 & 37.0650 / -0.1158 & 36.9580 / -0.0087 \\
 & 10 & 34.7983 & 34.7394 / 0.0589 & 34.7684 / 0.0298 & 34.7938 / 0.0045 \\
\hline
\multirow{10}{*}{$\ell_2$} & 1 & 41.8625 & 41.5888 / 0.2737 & 41.8470 / 0.0154 & 42.1490 / -0.2865 \\
 & 2 & 40.8258 & 41.2047 / -0.3788 & 41.1852 / -0.3594 & 40.9139 / -0.0880 \\
 & 3 & 38.7825 & 38.7469 / 0.0357 & 38.7701 / 0.0124 & 38.8155 / -0.0330 \\
 & 4 & 37.1577 & 37.0388 / 0.1189 & 37.4117 / -0.2540 & 37.2220 / -0.0643 \\
 & 5 & 35.0018 & 34.9913 / 0.0104 & 34.9269 / 0.0749 & 34.9039 / 0.0978 \\
 & 6 & 36.9708 & 36.6008 / 0.3700 & 36.6484 / 0.3224 & 36.9617 / 0.0091 \\
 & 7 & 36.2921 & 36.3042 / -0.0121 & 36.1945 / 0.0976 & 36.3021 / -0.0100 \\
 & 8 & 39.9055 & 39.8089 / 0.0966 & 39.6352 / 0.2702 & 39.6811 / 0.2244 \\
 & 9 & 36.9493 & 36.9743 / -0.0250 & 37.1580 / -0.2087 & 37.0678 / -0.1185 \\
 & 10 & 34.7983 & 34.7394 / 0.0589 & 34.7684 / 0.0298 & 34.7216 / 0.0767 \\
\hline
\end{tabular}
}
\end{table*}

\paragraph{Auxiliary Real-Data Downstream Evaluation on MovieLens}

We additionally conduct a lightweight downstream evaluation on the MovieLens real-data setting. 
This experiment is intended as an auxiliary sanity check rather than the main downstream evidence, since the downstream genre-classification task is not directly optimized by the summarization objective. 
For each method, we use the rounded top-\(k\) summary as the training set for a 1-nearest-neighbor genre classifier, and use the remaining movies in the same clean fold as held-out test items. 
The classifier assigns each test item the genre label of the most similar selected summary item under the clean similarity matrix. 
We report classification accuracy, the accuracy drop relative to the clean summary, and class coverage of the selected summary.

\begin{table}[t]
\centering
\caption{Auxiliary real-data downstream evaluation on MovieLens. 
Each top-\(k\) summary is used as the training set for a 1-NN genre classifier, and the remaining movies in the same clean fold are used as held-out test items.}
\label{tab:realdata_downstream_movie_appendix}
\scriptsize
\setlength{\tabcolsep}{3pt}
\resizebox{\columnwidth}{!}{
\begin{tabular}{c l c c c}
\toprule
\textbf{Norm} 
& \textbf{Method} 
& \textbf{Acc.} \(\uparrow\) 
& \(\boldsymbol{\Delta}\)\textbf{Acc.} \(\downarrow\) 
& \textbf{Class cov.} \(\uparrow\) \\
\midrule
\multirow{6}{*}{\(\ell_1\)}
& Clean summary        & 0.413 &  0.000 & 0.373 \\
& Proposed attack           & 0.418 & -0.004 & 0.388 \\
& PGD baseline         & 0.402 &  0.011 & 0.388 \\
& Random perturbation  & 0.411 &  0.002 & 0.368 \\
& Proposed robust      & 0.382 &  0.031 & 0.373 \\
& PGD robust baseline  & 0.362 &  0.051 & 0.415 \\
\midrule
\multirow{6}{*}{\(\ell_2\)}
& Clean summary        & 0.413 &  0.000 & 0.373 \\
& Proposed attack           & 0.431 & -0.018 & 0.387 \\
& PGD baseline         & 0.387 &  0.027 & 0.352 \\
& Random perturbation  & 0.412 &  0.001 & 0.374 \\
& Proposed robust      & 0.387 &  0.027 & 0.331 \\
& PGD robust baseline  & 0.398 &  0.016 & 0.331 \\
\bottomrule
\end{tabular}
}
\vspace{-6pt}
\end{table}

Table~\ref{tab:realdata_downstream_movie_appendix} shows that the MovieLens downstream results do not exhibit a consistent attack--defense recovery pattern. 
Under the \(\ell_1\)- and \(\ell_2\)-norm settings, some attacked summaries achieve accuracy comparable to or slightly higher than the clean summary, while robust summaries do not consistently improve downstream accuracy. 
This behavior suggests that the lightweight genre-classification task is strongly affected by class coverage and label distribution in the selected top-\(k\) movies, whereas the robust summarization objective is designed to preserve similarity-based summarization utility under perturbations rather than directly optimize genre classification. 
Therefore, we use this real-data downstream experiment only as an auxiliary sanity check. 
The controlled clustered benchmark in the main text provides a clearer downstream evaluation because its class structure is aligned with the representative-cluster structure used in the summarization model.

\subsection{Additional Empirical Results on the Direct Continuous-Score Summarization Objective}\label{supp:Experimental Evaluation} In this section, we also report empirical results on the direct continuous-score summarization objective used in earlier continuous summarization formulations. 
This objective uses the term 
$$
\sum_{\mu}\max_{\nu}x_\nu s_{\mu,\nu},
$$
which provides an intuitive soft-importance interpretation but is not the primary model covered by the DR-submodular theoretical guarantee. 
We include these results only as supplementary empirical evidence that the proposed attack-defense framework can be applied beyond the main theoretical model.

We conduct experiments on CIFAR-10~\cite{alex2009}, MNIST~\cite{lecun2021mnist}, and Fashion-MNIST~\cite{xiao2017fashionmnist}. Unless otherwise specified, the main numerical results are reported on CIFAR-10, while MNIST and Fashion-MNIST are used for additional visual validation. In each experiment, we construct $I=10$ victim summarization models, and each model contains $n=50$ images sampled from the raw dataset. Table~\ref{tab:exp_metrics} summarizes the main mathematical notations used in the experiments. Specifically, we consider three attack types, corresponding to $l_1$-, $l_2$-, and $l_\infty$-norm constraints. For $p\in\{1,2\}$, the attack budget is chosen from $\epsilon_p\in[0.01,10]$, and for $p=\infty$, we use $\epsilon_\infty\in[0.1,1]$. The feasible set of summarization solutions is $\mathcal{X}\in[0,D]^n$ with $D\in\{0.1,1,10\}$. Additionally, i) For the proposed attack algorithm, we initialize $\mathbf{v}_0=\mathbf{0}$ and $\mathbf{w}_0=\frac{1}{I}\mathbf{1}$, and set $K_{\rm attack}=100$ and $T_{\rm attack}=30$. ii) For the robust defense algorithm, we initialize $\mathbf{v}_0=\mathbf{0}$ and $\mathbf{w}_0=\frac{1}{3}\mathbf{1}$, since three attack types are considered, and use $K_{\rm robust}=100$, $T_{\rm robust}=30$, and $\gamma=0.1$.

\begin{table}[t]
\centering
\caption{Key experimental parameters and evaluation metrics.}
\label{tab:exp_metrics}
\scalebox{0.9}{
\begin{tabular}{c|l}
\hline
\textbf{Symbol / Metric} & \textbf{Meaning} \\
\hline\hline
$\epsilon_p$ & Attack budget under the $\ell_p$-norm constraint \\
$K_{\rm attack}$ & Outer iterations of Alg.~\ref{alg:attack} \\
$T_{\rm attack}$ & Inner iterations of $\mathcal{M}_{\rm greedy}$ in Alg.~\ref{alg:attack} \\
$K_{\rm robust}$ & Outer iterations of Alg.~\ref{alg:robust} \\
$T_{\rm robust}$ & Inner iterations of Alg.~\ref{alg:robust} \\
Avg. Deg. & Average degradation under attack \\
Succ. Ratio & Fraction of victim models successfully degraded \\
Intensity & Normalized attack intensity \\
Loss & Relative clean-utility change under robust optimization \\
Robustness & Relative clean-attacked variation of the robust solution \\
Mitigation & Improvement over attacked greedy under the clean objective \\
\hline
\end{tabular}
}
\end{table}

\paragraph{Attack Performance under $l_p$-Norm Constraints}

We first compare the proposed attack generator with a representative gradient-based baseline, and then study the effect of the attack budget and feasible-set size on attack performance.

{\bf Comparison with a Gradient-Based Baseline.} The representative methods summarized in Table~\ref{tab:related_positioning} differ substantially in target setting and structural assumptions. In particular, the methods in~\cite{adibi2022} do not explicitly address the continuous multi-target setting considered here. We therefore compare our method with a representative gradient-based baseline for general non-convex min-max optimization~\cite{wang2021}, which is directly applicable to our attack setting.

As shown in Table~\ref{tab:attack_baseline_all}, the proposed method consistently achieves larger average degradation and attack intensity across all three attack constraints. Under the $l_1$-norm setting, it also attains a higher success ratio than the gradient-based baseline. Under the $l_2$- and $l_\infty$-norm settings, both methods achieve the same success ratio, while the proposed method still yields slightly stronger degradation and intensity. These results show that the proposed attack generator is empirically competitive with, and in these representative settings slightly stronger than, the gradient-based baseline.

\begin{table}[t]
\centering
\caption{Comparison with the gradient-based baseline under different attack constraints ($D=1, \epsilon=1,1,0.5$ for $l_1,l_2,l_{\infty}$).}
\label{tab:attack_baseline_all}
\scalebox{1.1}{
\begin{tabular}{llccc}
\hline
\textbf{Metric} & \textbf{Method} & $\boldsymbol{l_1}$ & $\boldsymbol{l_2}$ & $\boldsymbol{l_{\infty}}$ \\
\hline
\multirow{2}{*}{Success Ratio (\%) $\uparrow$}
& Proposed & 100 & 100 & 90 \\
& Baseline & 90 & 100 & 90 \\
\hline
\multirow{2}{*}{Avg. Degradation $\uparrow$}
& Proposed & 0.9583 & 0.9907 & 0.8292 \\
& Baseline & 0.7933 & 0.9644 & 0.6948 \\
\hline
\multirow{2}{*}{Attack Intensity (\%) $\uparrow$}
& Proposed & 6.78 & 6.93 & 5.94 \\
& Baseline & 5.61 & 6.75 & 4.98 \\
\hline
\end{tabular}
}
\end{table}

{\bf Effect of the Attack Budget $\epsilon_p$ and Feasible-Set Size $D$}
We next evaluate how attack performance changes with the attack budget and the feasible-set size. Figure~\ref{fig:attack_intensity} shows the average attack intensity under different $l_p$-norm constraints.

Overall, the attack intensity increases as $\epsilon_p$ becomes larger. Moreover, the effect of $D$ depends on the attack type. Under $l_1$- and $l_2$-norm constraints, the influence of $D$ is relatively limited. For example, when $\epsilon_p=10$, the average attack intensity lies in a narrow range under both $l_1$- and $l_2$-norm constraints. In contrast, under the $l_\infty$-norm constraint, the attack intensity is much more sensitive to $D$. For instance, when $\epsilon_\infty=1.0$, the intensity grows substantially from the case $D=0.1$ to the case $D=10$. This suggests that the feasible-set scale plays a more important role under coordinate-wise bounded perturbations.

{\bf Attack Success Ratio under Different Settings.}
Besides the average attack intensity, we also report the success ratio of attacks in Table~\ref{tab:attack_success}. The results show that the proposed attack achieves a high success probability across a wide range of parameter settings. In particular, the success ratio approaches or reaches $100\%$ once the attack budget exceeds a moderate threshold. For example, under the $l_2$-norm setting, all attacks succeed when $\epsilon_2\geq 0.05$.

\begin{table}[t!]
\centering
\normalsize
\caption{Attack success ratio of Alg.~\ref{alg:attack} under $l_{p}$-norm attacks ($p=1,2,\infty$) across different $D$ ($D=0.1,1,10$) and $\epsilon$ values. Parameters: $\eta$ values from Tab.~\ref{tab:eta}.} 
\scalebox{0.9}{
\begin{tabular}{p{0.08\columnwidth}|p{0.08\columnwidth}p{0.08\columnwidth}p{0.08\columnwidth}p{0.08\columnwidth}p{0.08\columnwidth}p{0.08\columnwidth}p{0.08\columnwidth}}
\hline
\multirow{2}{*}{$D$} & \multicolumn{7}{c}{$\epsilon_p(p=1,2)$}\\
\cline{2-8}
& 0.01 & 0.05 & 0.1 & 0.5 & 1.0 & 5.0 & 10.0 \\ 
\hline
\multicolumn{8}{c}{$\textbf{Ratio}^{\text{att}}(p=1,\epsilon_{1},D)$ under $l_1$-norm } \\ \hline
0.1 & 0.80 & 0.70 & 0.80 & 1.00 & 1.00 & 1.00 & 1.00 \\ 
1.0 & 0.80 & 0.80 & 0.90 & 1.00 & 1.00 & 1.00 & 1.00 \\ 
10.0 & 0.80 & 0.90 & 0.80 & 1.00 & 1.00 & 1.00 & 1.00 \\ \hline
\multicolumn{8}{c}{$\textbf{Ratio}^{\text{att}}(p=2,\epsilon_{2},D)$ under $l_2$-norm } \\ \hline
0.1 & 0.60 & 1.00 & 1.00 & 1.00 & 1.00 & 1.00 & 1.00 \\ 
1.0 & 0.90 & 1.00 & 1.00 & 1.00 & 1.00 & 1.00 & 1.00 \\ 
10.0 & 0.90 & 1.00 & 1.00 & 1.00 & 1.00 & 1.00 & 1.00 \\ 
\hline
\end{tabular}
}
\scalebox{0.8}{   
\begin{tabular}{p{0.06\columnwidth}|p{0.056\columnwidth}p{0.056\columnwidth}p{0.056\columnwidth}p{0.056\columnwidth}p{0.056\columnwidth}p{0.056\columnwidth}p{0.056\columnwidth}p{0.056\columnwidth}p{0.056\columnwidth}p{0.056\columnwidth}}
\hline
\multirow{2}{*}{$D$} & \multicolumn{10}{c}{$\epsilon_p(p=\infty)$}\\
\cline{2-11}
& 0.1 & 0.2 & 0.3 & 0.4 & 0.5 & 0.6 & 0.7 & 0.8 & 0.9 & 1.0 \\ 
\hline
\multicolumn{10}{c}{$\textbf{Ratio}^{\text{att}}(p=\infty,\epsilon_{\infty},D)$ under $l_{\infty}$-norm } \\ \hline
0.1 & 0.90 & 0.80 & 0.80 & 0.80 & 0.80 & 0.90 & 0.90 & 0.90 & 0.90 & 0.90 \\
1.0 & 0.50 & 1.00 & 1.00 & 1.00 & 1.00 & 1.00 & 1.00 & 1.00 & 1.00 & 1.00 \\
10.0 & 1.00 & 1.00 & 1.00 & 1.00 & 1.00 & 1.00 & 1.00 & 1.00 & 1.00 & 1.00 \\\hline
\end{tabular}}
\label{tab:attack_success}
\end{table}

\begin{figure*}[t!]
\centering
\subfloat[$l_{1}$-norm]{
\includegraphics[width=0.31\linewidth]{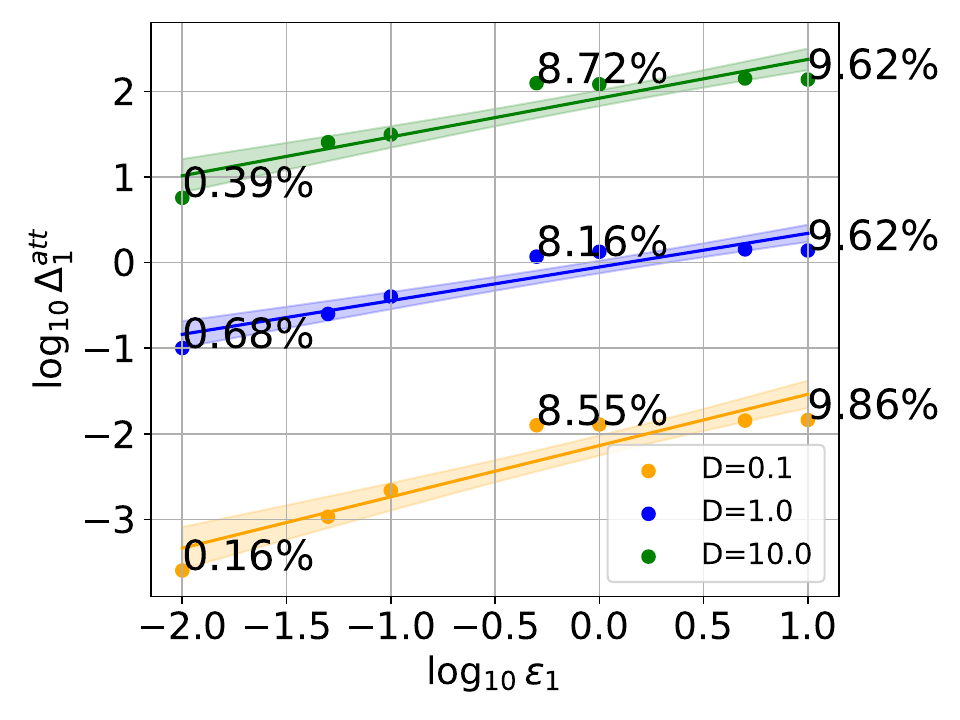}
}
\subfloat[$l_{2}$-norm]{
\includegraphics[width=0.31\linewidth]{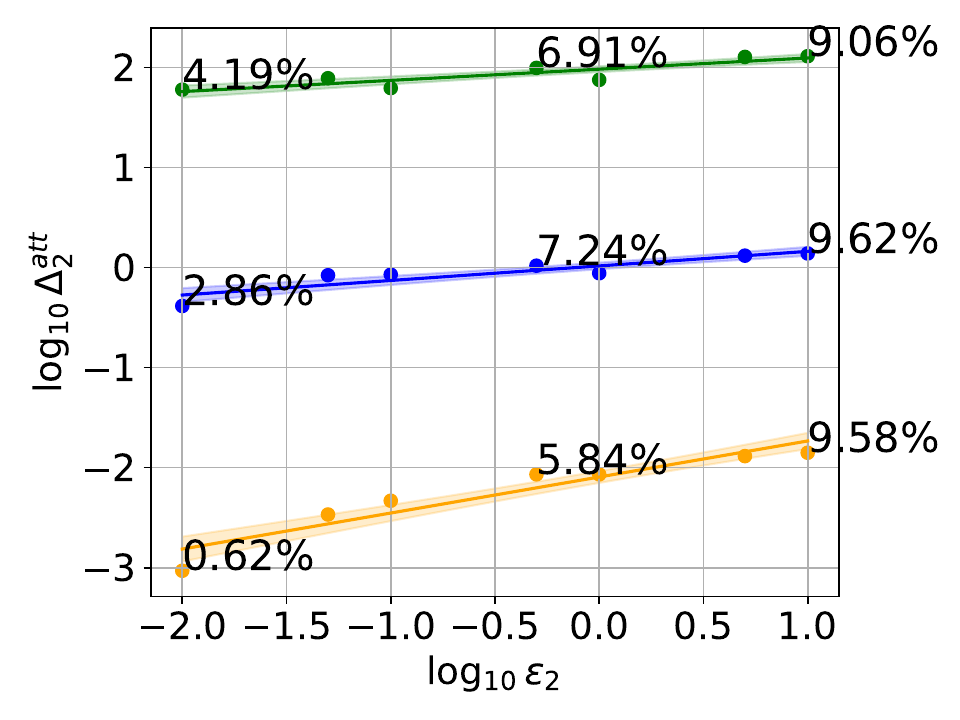}
}
\subfloat[$l_{\infty}$-norm]{
\includegraphics[width=0.31\linewidth]{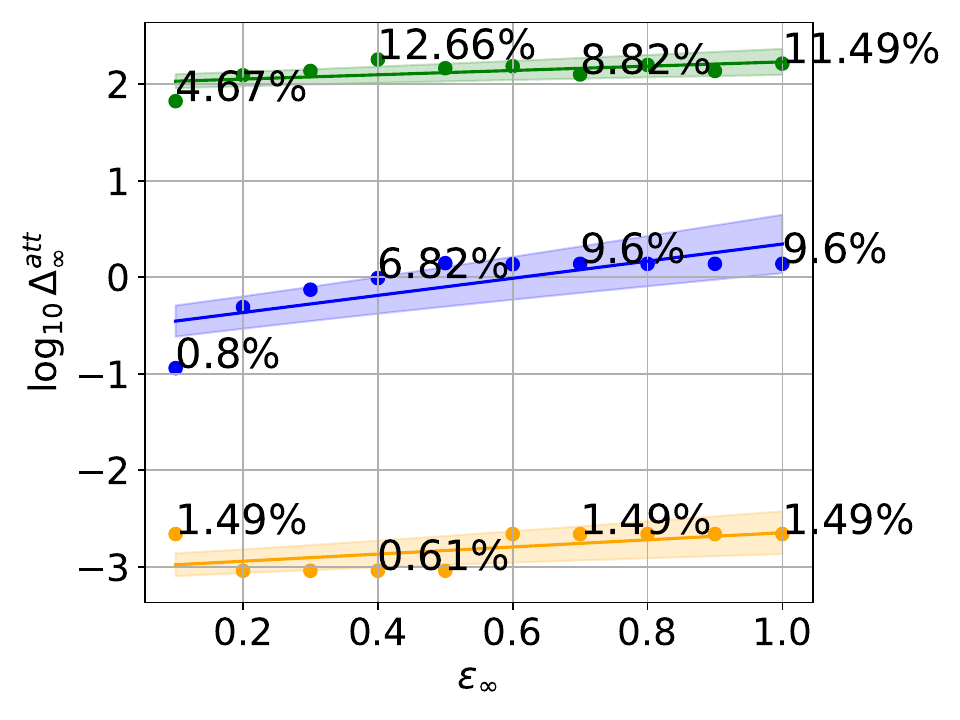}
}%
\caption{Average attack intensity under $l_{p}$-norm constraints of Alg.~\ref{alg:attack}. Parameters
$K_{\rm attack}=100, T_{\rm attack}=30$ in Alg.~\ref{alg:Mgreedy} and $\eta$ refers to Tab.~\ref{tab:eta}. ``Dots'' represents average attack intensity ($\textbf{Intensity}^{\rm att}(p, \epsilon_{p}, D)$) of Alg.~\ref{alg:attack} under $l_p$-norm ($p = 1, 2, \infty$) constraints. ``Lines'' illustrate trends of attack intensity increasing with $\epsilon_{p}$ and ``$Y$-axis'' represents average function differences ($\bar{\Delta}^{\rm att}(p,\epsilon_{p},D)$)).
\label{fig:attack_intensity}}
\end{figure*}

\paragraph{Robustness of the Defense Algorithm under Mixed Attacks}\label{results:robust}

We now evaluate the proposed defense algorithm under mixed attack types. In particular, we focus on three questions: 1) how the regularization parameter $\gamma$ affects the trade-off between clean-data utility and robustness; 2) how the proposed robust algorithm compares with a representative gradient-based robust baseline; and 3) whether the proposed robust algorithm can effectively mitigate the impact of attacks.

{\bf Utility--Robustness Trade-off under Different $\gamma$.}
The proposed defense algorithm optimizes Eq.~\ref{eq:robust_problem}, where the regularization parameter $\gamma$ controls the balance between clean-data utility and robustness. As shown in Table~\ref{tab:robust_metrics}, larger values of $\gamma$ generally lead to stronger robustness. For example, under the $l_{\infty}$-norm attack, the robustness value decreases from $0.84$ when $\gamma=0.1$ to $0.22$ when $\gamma=10$, and under the $l_2$-norm attack it decreases from $0.43$ to $0.27$. Moreover, for a fixed $\gamma$, the robustness gap across different attack types becomes smaller when $\gamma$ is larger, suggesting that the defense becomes more balanced against mixed attacks.

\begin{table}[t!]
\centering
\normalsize
\caption{Performance metrics of Alg.~\ref{alg:robust} under different $\gamma$ values $(K_{\rm robust}=100, T_{\rm robust}=30, D=1, \epsilon=1,1,0.5$ for $l_1,l_2,l_{\infty}$).}
\scalebox{1}{
\begin{tabular}{c|c|ccc}
\hline
\multicolumn{2}{c|}{Metrics} & \multicolumn{3}{c}{Different $\gamma$ values} \\ \cline{3-5}
\multicolumn{2}{c|}{} & $\gamma=0.1$ & $\gamma=1$ & $\gamma=10$\\ \hline
\multicolumn{2}{c|}{Loss (\%)  $\downarrow$} & 2.96 & \textbf{2.94} & 4.01 \\ \hline
\multirow{3}{*}{Robustness (\%) $\downarrow$} & $l_{\infty}$ & 0.84 & 0.81 & \textbf{0.22} \\
                            & $l_{1}$      & 0.88 & 0.87 & \textbf{0.34} \\
                            & $l_{2}$      & 0.43 & 0.45 & \textbf{0.27} \\ \hline
\multirow{3}{*}{Mitigation  (\%) $\uparrow$} & $l_{\infty}$ & 5.82 & \textbf{5.86} & 5.52 \\
                            & $l_{1}$      & 4.34 & \textbf{4.36} & 3.89 \\
                            & $l_{2}$      & 3.86 & \textbf{3.88} & 3.22 \\
\hline
\end{tabular}
}
\label{tab:robust_metrics}
\end{table}
This behavior is consistent with the role of the regularization term: a larger $\gamma$ encourages the weight vector $\mathbf{w}$ to remain more balanced across attack types, thereby reducing the dominance of any single attack type and improving stability under perturbed inputs. However, Table~\ref{tab:robust_metrics} also shows that excessively large $\gamma$ may increase the clean-data loss. Therefore, $\gamma$ induces a clear utility--robustness trade-off, and $\gamma=1$ provides a good balance in our experiments.

{\bf Comparison with a Gradient-Based Robust Baseline.} The representative methods summarized in Table~\ref{tab:related_positioning} differ substantially in model assumptions and attack settings. In particular, existing robust DR-submodular methods such as~\cite{lee2022,lian2024zeroth} mainly focus on monotone settings and do not explicitly address mixed attacks. We therefore compare our method with a representative gradient-based robust baseline that is directly compatible with the mixed-attack setting considered here.

As shown in Table~\ref{tab:defense_baseline_all}, the proposed robust algorithm consistently achieves near-zero or much smaller clean-data loss and substantially smaller robustness values than the gradient-based baseline, while maintaining positive mitigation across all three attack constraints. In contrast, the gradient-based baseline suffers from much larger clean-data loss and negative mitigation in all three representative settings, indicating that it fails to consistently offset the attack effect. These results support the advantage of structure-aware robust optimization for continuous data summarization.

\begin{table}[t]
\centering
\caption{Comparison with the gradient-based robust baseline under different attack constraints ($D=1, \gamma=1, \epsilon=1,1,0.5$ for $l_1,l_2,l_{\infty}$).}
\label{tab:defense_baseline_all}
\scalebox{1.1}{
\begin{tabular}{llccc}
\hline
\textbf{Metric} & \textbf{Method} & $\boldsymbol{l_1}$ & $\boldsymbol{l_2}$ & $\boldsymbol{l_{\infty}}$ \\
\hline
\multirow{2}{*}{Loss (\%) $\downarrow$}
& Proposed & -0.225 & 0.811 & 1.588 \\
& Baseline & 30.66 & 31.53 & 29.58 \\
\hline
\multirow{2}{*}{Robustness (\%) $\downarrow$}
& Proposed & 0.664 & 0.401 & 0.649 \\
& Baseline & 32.33 & 29.19 & 37.01 \\
\hline
\multirow{2}{*}{Mitigation (\%) $\uparrow$}
& Proposed & 6.299 & 6.641 & 6.747 \\
& Baseline & -47.53 & -45.40 & -50.95 \\
\hline
\end{tabular}
}
\end{table}

{\bf Defense Effectiveness and Mitigation.} We next examine whether the proposed defense can effectively reduce the impact of attacks. From Table~\ref{tab:robust_metrics}, all mitigation values are positive for all tested choices of $\gamma$ and all three attack types. This indicates that the proposed robust algorithm consistently alleviates the degradation caused by attacks.

Moreover, the mitigation results should be interpreted together with the loss values. Although the defense algorithm introduces a small utility loss under clean data, the reduction in attack impact is consistently larger, showing that the robust formulation is beneficial overall. For example, when $\gamma=1$, the defense achieves positive mitigation under all three norm-constrained attacks while maintaining relatively small clean-data loss. These results demonstrate that the proposed defense algorithm can effectively improve robustness against mixed attacks in continuous data summarization.

\paragraph{Visualization Results on Multiple Datasets}
\label{results:visual} We further provide qualitative visualization results for the attack algorithm (Alg.~\ref{alg:attack}), the plain greedy summarization algorithm (Alg.~\ref{alg:Mgreedy}), and the proposed robust defense algorithm (Alg.~\ref{alg:robust}) on CIFAR-10, MNIST, and Fashion-MNIST. To generate the displayed summaries, we first normalize the continuous solution $\mathbf{x}$ into $[0,1]$, then convert it into a discrete selection vector using a threshold, and finally select images according to the resulting discrete solution.

Figures~\ref{fig:mnist}--\ref{fig:fashion} show that the attacked outputs differ substantially from the plain greedy outputs, while the robust outputs recover a noticeably larger portion of the original selection patterns. For instance, in Fig.~\ref{fig:mnist}, only two selected images remain unchanged after attack, whereas the robust output recovers eleven images that coincide with the plain greedy output. Similar phenomena are observed on CIFAR-10 and Fashion-MNIST, consistent with the numerical results.

\begin{figure}[t!]
    \centering
    \scalebox{0.5}{
    \begin{tabular}{cc}
        \textbf{Alg.~\ref{alg:attack} outputs} &
        \makecell[l]{
            \begin{minipage}[b]{1.31\columnwidth}
                \includegraphics[width=\textwidth]{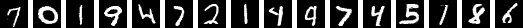}
            \end{minipage}
        } \\
        \textbf{Alg.~\ref{alg:Mgreedy} outputs} &
        \makecell[l]{
            \begin{minipage}[b]{1.4\columnwidth}
                \includegraphics[width=\textwidth]{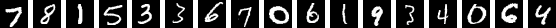}
            \end{minipage}
        } \\ 
        \textbf{Alg.~\ref{alg:robust} outputs} &
        \makecell[l]{
            \begin{minipage}[b]{1.57\columnwidth}
                \includegraphics[width=\textwidth]{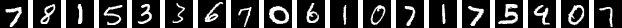}
            \end{minipage}
        }
    \end{tabular}
    }
    \caption{Visualization results of Alg.~\ref{alg:attack} (attack algorithm), Alg.~\ref{alg:Mgreedy} (plain submodular algorithm), and Alg.~\ref{alg:robust} (robust algorithm under attack) outputs under $l_1$-norm attack for dataset MNIST ($D=1$, $\epsilon=1$, threshold$=0.6$).}
    \label{fig:mnist}
\end{figure}

\begin{figure}[t!]
    \centering
    \scalebox{0.5}{
    \begin{tabular}{cc}
        \textbf{Alg. 2 outputs} &
        \makecell[l]{
            \begin{minipage}[b]{1.5\columnwidth}
                \includegraphics[width=\textwidth]{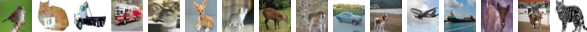}
            \end{minipage}
        } \\
        \textbf{Alg. 1 outputs} &
        \makecell[l]{
            \begin{minipage}[b]{1.5\columnwidth}
                \includegraphics[width=\textwidth]{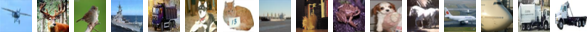}
            \end{minipage}
        } \\ 
        \textbf{Alg. 3 outputs} &
        \makecell[l]{
            \begin{minipage}[b]{1.5\columnwidth}
                \includegraphics[width=\textwidth]{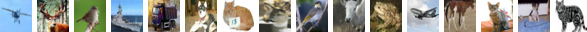}
            \end{minipage}
        }
    \end{tabular}
    }
    \caption{Visualization results of attacked (Alg.~\ref{alg:attack}), plain (Alg.~\ref{alg:Mgreedy} without attack) and robust (Alg.~\ref{alg:robust} with attack) outputs under $l_1$-norm attack for dataset CIFAR-10 ($D=1$, $\epsilon=1$ and threshold=0.493).}
    \label{fig:cifar10}
\end{figure}

\begin{figure}[t!]
    \centering
    \scalebox{0.5}{
    \begin{tabular}{cc}
        \textbf{Alg. 2 outputs} &
        \makecell[l]{
            \begin{minipage}[b]{1.32\columnwidth}
                \includegraphics[width=\textwidth]{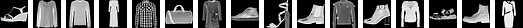}
            \end{minipage}
        } \\
        \textbf{Alg. 1 outputs} &
        \makecell[l]{
            \begin{minipage}[b]{1.4\columnwidth}
                \includegraphics[width=\textwidth]{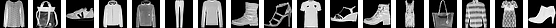}
            \end{minipage}
        } \\ 
        \textbf{Alg. 3 outputs} &
        \makecell[l]{
            \begin{minipage}[b]{1.57\columnwidth}
                \includegraphics[width=\textwidth]{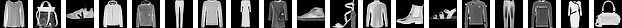}
            \end{minipage}
        }
    \end{tabular}
    }
    \caption{Visualization results of attacked (Alg.~\ref{alg:attack}), plain (Alg.~\ref{alg:Mgreedy} without attack) and robust (Alg.~\ref{alg:robust} with attack) outputs under $l_1$-norm attacks for dataset Fashion MNIST ($D=1$, $\epsilon=1$ and threshold=0.48).}
\label{fig:fashion}
\end{figure}
\paragraph{Additional Results for the Attack Algorithm (Alg.~\ref{alg:attack})}\label{app:conver} This section provides additional results on the convergence behavior of Alg.~\ref{alg:attack} and the per-model attack effects under different attack budgets.

{\bf Convergence of the Attack Algorithm.} Fig.~\ref{fig:attack_convergence} illustrates the effect of the step size $\eta$ on the convergence of Alg.~\ref{alg:attack} under different attack types. When $D=1$, choosing $\eta$ in the range $[0.05,0.15]$ yields stable convergence for all three norm-constrained attacks. In contrast, excessively small step sizes (e.g., $\eta=0.005$ for the $l_{1}$-norm attack, $\eta=0.0001$ for the $l_{2}$-norm attack, and $\eta=0.001$ for the $l_{\infty}$-norm attack) lead to non-convergence or convergence to worse solutions. This indicates that Alg.~\ref{alg:attack} is relatively insensitive to moderate changes in $\eta$, but overly small step sizes can significantly degrade convergence and solution quality.

Tab.~\ref{tab:eta} lists the $\eta$ values used for the $l_{1}$- and $l_{2}$-norm attacks under different $(D,\epsilon)$ settings. These values are chosen so that the norm of the generated perturbation $\mathbf{v}$ approaches the prescribed budget $\epsilon$ as closely as possible. For the $l_{\infty}$-norm attack, we use the theoretically motivated choice $\eta=1/\sqrt{K_{\rm attack}}$ with $K_{\rm attack}=100$, following Theorem~\ref{thm:attack}. In this case, whether each coordinate of $\mathbf{v}$ reaches $\epsilon$ is not itself a meaningful indicator of attack strength.

\begin{table}[t!]
\centering
\normalsize
\caption{Values of $\eta$ in Alg.~\ref{alg:attack} under $l_1/l_2$-norm attacks with varying $(D,\epsilon)$ configurations.}
\scalebox{0.77}{
\begin{tabular}{c|ccccccc}
\hline
\multirow{2}{*}{$D$} & \multicolumn{7}{c}{$\eta$ for different $\epsilon$ values} \\ 
\cline{2-8}
& 0.01 & 0.05 & 0.10 & 0.50 & 1.00 & 5.00 & 10.00 \\ \hline
0.1  & 0.0500 & 0.3000 & 0.8000 & 1.0000 & 1.0000 & 10.0000 & 20.0000 \\
1.0  & 0.0050 & 0.0400 & 0.0500 & 0.0500 & 0.1000 & 0.5000  & 2.0000 \\
10.0 & 0.0001 & 0.0004 & 0.0008 & 0.0030 & 0.0070 & 0.0500  & 0.1000 \\
\hline
\end{tabular}
}
\label{tab:eta}
\end{table}

\begin{figure*}[t!]
\centering
\subfloat[$l_{1}$-norm]{
\includegraphics[width=0.3\linewidth]{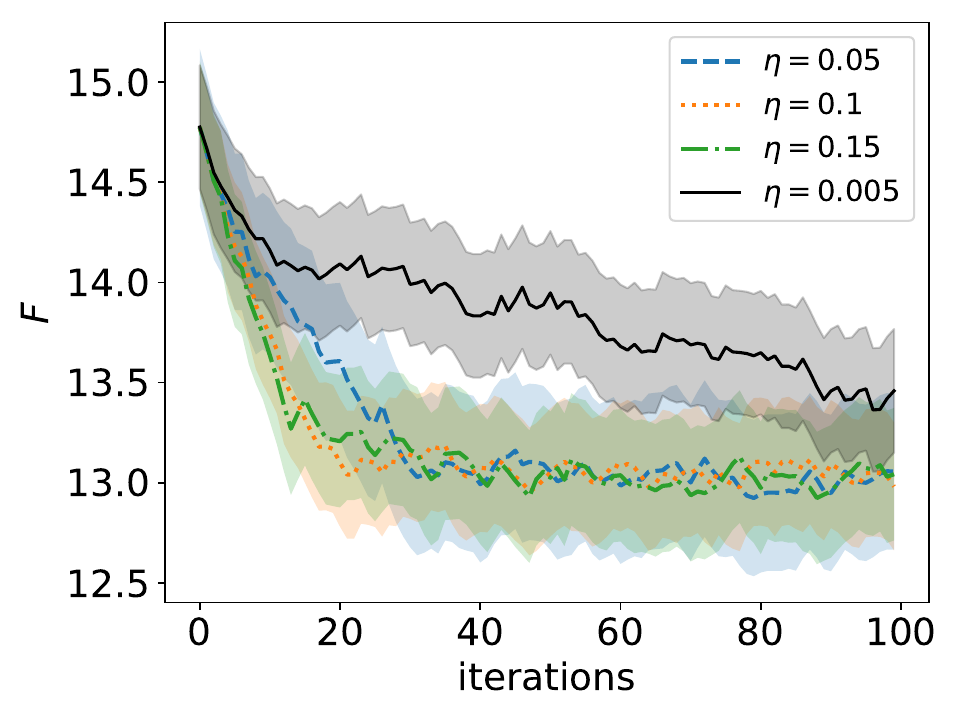}
}
\subfloat[$l_{2}$-norm]{
\includegraphics[width=0.3\linewidth]{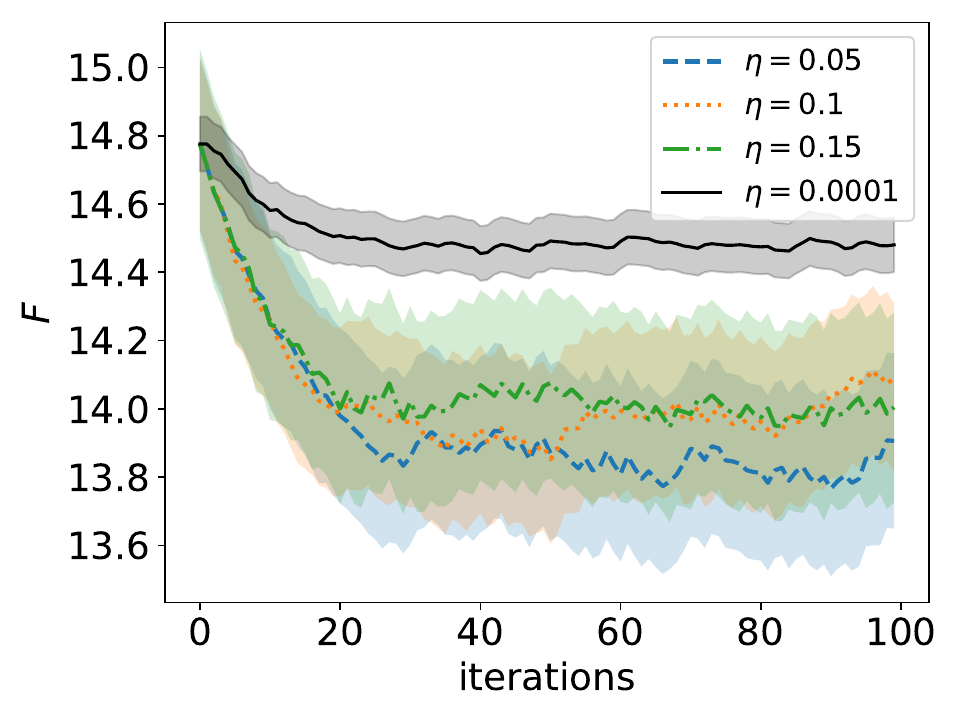}
}
\subfloat[$l_{\infty}$-norm]{
\includegraphics[width=0.3\linewidth]{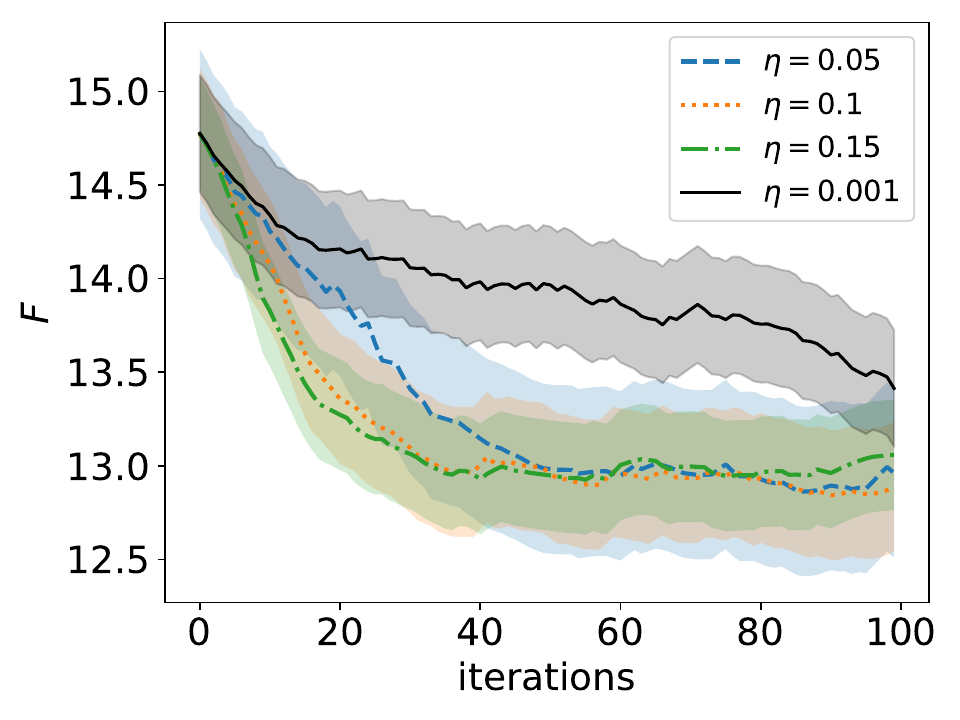}
}%
\caption{Convergence behavior of Alg.~\ref{alg:attack} under $l_{p}$-norm attacks $(p=1,2,\infty)$ with varying $\eta$ values. Parameters: $K_{\rm attack}=100$, $T_{\rm attack}=30$, and $D=1$.}
\label{fig:attack_convergence}
\end{figure*}

{\bf Per-model Attack Effects under Different Attack Budgets.} For a fixed $D$, with $K_{\rm attack}=100$ and $T_{\rm attack}=30$ in Alg.~\ref{alg:attack}, the objective value generally decreases as the attack budget increases. In particular, all models are successfully attacked when $\epsilon_{1}\geq 0.5$ in Tabs.~\ref{tab:att_per_results_l1d0.1}--\ref{tab:att_per_results_l1d10}, and when $\epsilon_{2}\geq 0.05$ in Tabs.~\ref{tab:att_per_results_l2d0.1}--\ref{tab:att_per_results_l2d10}, for all $D\in\{0.1,1,10\}$. For the $l_{\infty}$-norm attack, the per-model results are more sensitive to the value of $D$. For example, Model 7 is never successfully attacked when $D=0.1$, as shown in Tab.~\ref{tab:att_per_results_linftyd0.1}. In contrast, when $D=10$, all models can be successfully attacked, as shown in Tab.~\ref{tab:att_per_results_linftyd10}. These results are consistent with the averaged trends reported in the main text and further illustrate how the feasible-set scale affects the strength of the $l_{\infty}$-norm attack.

\begin{table*}[t!]
\centering
\caption{Objective value $F$ of Alg.~\ref{alg:Mgreedy} under $l_{1}$-norm attacks with varying $\epsilon_1 \in [0, 10]$. Parameters: $D=0.1$, and $\eta$ values from Tab.~\ref{tab:eta}.}
\scalebox{1}{
\begin{tabular}{l|ccccccccccc}
\toprule
\textbf{\boldmath$\epsilon_1$} & ${\rm Model}_1$ & ${\rm Model}_2$ & ${\rm Model}_3$ & ${\rm Model}_4$ & ${\rm Model}_5$ & ${\rm Model}_6$ & ${\rm Model}_7$ & ${\rm Model}_8$ & ${\rm Model}_9$ & ${\rm Model}_{10}$ \\
\midrule
0.00 & 0.1511 & 0.1461 & 0.1469 & 0.1429 & 0.1436 & 0.1491 & 0.1434 & 0.1479 & 0.1484 & 0.1484 \\
0.01   & 0.1493 & 0.1452 & 0.1454 & {\bf 0.1451} & 0.1430 & 0.1482 & {\bf 0.1469} & 0.1478 & 0.1470 & 0.1473 \\
0.05         & 0.1500 & 0.1454 & 0.1454 & {\bf 0.1458} & {\bf 0.1437} & 0.1435 & {\bf 0.1465} & 0.1464 & 0.1474 & 0.1429 \\
0.10          & 0.1498 & 0.1450 & 0.1420 & {\bf 0.1452} & 0.1386 & 0.1479 & {\bf 0.1460} & 0.1472 & 0.1422 & 0.1420 \\
0.50          & 0.1373 & 0.1334 & 0.1314 & 0.1327 & 0.1320 & 0.1328 & 0.1343 & 0.1368 & 0.1355 & 0.1359 \\
1.00          & 0.1365 & 0.1319 & 0.1352 & 0.1359 & 0.1296 & 0.1326 & 0.1324 & 0.1351 & 0.1368 & 0.1334 \\
5.00          & 0.1342 & 0.1308 & 0.1317 & 0.1316 & 0.1294 & 0.1315 & 0.1326 & 0.1350 & 0.1333 & 0.1350 \\
10.00         & 0.1336 & 0.1303 & 0.1310 & 0.1314 & 0.1308 & 0.1346 & 0.1304 & 0.1344 & 0.1328 & 0.1335 \\
\bottomrule
\end{tabular}
}
\label{tab:att_per_results_l1d0.1}
\end{table*}

\begin{table*}[t!]
  \centering
  \caption{Objective values $F$ of Alg.~\ref{alg:Mgreedy} under $l_{1}$-norm attacks with varying $\epsilon_1 \in [0, 10]$. Parameters: $D=1$, $\eta$ values from Tab.~\ref{tab:eta}.}
  \scalebox{1}{
  \begin{tabular}{l|ccccccccccc}
  \toprule
  \textbf{\boldmath$\epsilon_1$} & ${\rm Model}_1$ & ${\rm Model}_2$ & ${\rm Model}_3$ & ${\rm Model}_4$ & ${\rm Model}_5$ & ${\rm Model}_6$ & ${\rm Model}_7$ & ${\rm Model}_8$ & ${\rm Model}_9$ & ${\rm Model}_{10}$ \\
  \midrule
  0.00 & 14.7757 & 14.2996 & 14.3820 & 13.9502 & 14.0592 & 14.5925 & 13.9490 & 14.4686 & 14.5209 & 14.5194 \\
  0.01         & 14.7201 & 13.8402 & 14.3332 & {\bf 14.3451} & 13.6146 & 14.5270 & {\bf 14.3923} & 13.9039 & 14.4224 & 14.4154 \\
  0.05         & 14.4612 & 14.2015 & 14.1490 & {\bf 14.2326} & 13.9193 & 14.1123 & {\bf 14.2087} & 14.1635 & 13.7858 & 13.7707 \\
  0.10          & 14.2288 & 13.7346 & 14.2672 & 13.7487 & 13.5319 & 14.0480 & {\bf 13.9597} & 13.8832 & 13.8181 & 14.2858 \\
  0.50          & 13.6322 & 13.0018 & 13.1876 & 13.1219 & 12.8331 & 13.0850 & 13.2278 & 13.2038 & 13.3297 & 13.1609 \\
  1.00          & 13.1692 & 12.8830 & 12.8003 & 13.0511 & 12.7220 & 13.1636 & 13.0931 & 13.2244 & 12.8591 & 13.1685 \\
  5.00          & 13.1553 & 12.8361 & 13.0886 & 12.8006 & 12.7337 & 12.9865 & 12.8367 & 13.0297 & 12.7035 & 13.0768 \\
  10.00         & 12.9570 & 12.9438 & 12.8959 & 13.1496 & 12.8766 & 12.9447 & 12.7274 & 13.2199 & 12.8900 & 13.0697 \\
  \bottomrule
  \end{tabular}
  }
  \label{tab:att_per_results_l1d1}
\end{table*}

\begin{table*}[t!]
  \centering
  \caption{Objective values $F$ of Alg.~\ref{alg:Mgreedy} under $l_{1}$-norm attacks with varying $\epsilon_1 \in [0, 10]$. Parameters: $D=10$, $\eta$ values from Tab.~\ref{tab:eta}.}
  \scalebox{1}{
  \begin{tabular}{l|ccccccccccc}
  \toprule
  \textbf{\boldmath$\epsilon_1$} & ${\rm Model}_1$ & ${\rm Model}_2$ & ${\rm Model}_3$ & ${\rm Model}_4$ & ${\rm Model}_5$ & ${\rm Model}_6$ & ${\rm Model}_7$ & ${\rm Model}_8$ & ${\rm Model}_9$ & ${\rm Model}_{10}$ \\
  \midrule
  0.00 & 1474.2324 & 1426.8492 & 1435.0023 & 1391.8554 & 1402.8959 & 1410.4395 & 1391.6993 & 1443.6483 & 1448.8861 & 1448.7191 \\
  0.01         & 1462.9098 & 1416.6063 & 1422.4803 & 1376.0069 & 1398.6040 & {\bf 1442.1429} & {\bf 1415.4935} & 1425.0280 & 1424.2331 & 1433.6294 \\
  0.05         & 1449.6008 & 1407.9932 & 1423.4187 & 1383.4185 & 1345.3509 & {\bf 1443.2330} & 1375.9175 & 1383.7270 & 1373.5448 & 1433.1648 \\
  0.10          & 1416.3975 & 1364.5970 & 1371.7884 & 1375.6148 & 1353.9248 & {\bf 1439.1516} & {\bf 1418.7915} & 1418.2876 & 1373.6033 & 1428.7677 \\
  0.50          & 1344.7070 & 1302.4760 & 1289.9080 & 1302.3589 & 1269.6011 & 1302.7427 & 1305.0936 & 1330.5593 & 1257.7889 & 1321.0485 \\
  1.00          & 1320.7759 & 1308.1058 & 1287.8753 & 1299.5853 & 1284.2200 & 1299.0918 & 1289.3330 & 1340.8883 & 1308.6752 & 1324.1519 \\
  5.00          & 1286.7390 & 1272.4438 & 1285.6346 & 1283.5320 & 1285.0912 & 1273.1899 & 1296.5136 & 1306.5048 & 1271.8013 & 1296.2556 \\
  10.00         & 1306.1874 & 1281.9626 & 1285.0219 & 1288.2721 & 1274.9180 & 1282.1630 & 1281.8991 & 1300.2769 & 1285.9448 & 1311.3829 \\
  \bottomrule
  \end{tabular}
  }
  \label{tab:att_per_results_l1d10}
\end{table*}

\begin{table*}[t!]
  \centering
  \caption{Objective values $F$ of Alg.~\ref{alg:Mgreedy} under $l_{2}$-norm attacks with varying $\epsilon_2 \in [0, 10]$. Parameters: $D=0.1$, $\eta$ values from Tab.~\ref{tab:eta}.}
  \scalebox{1}{
  \begin{tabular}{l|ccccccccccc}
  \toprule
  \textbf{\boldmath$\epsilon_2$} & ${\rm Model}_1$ & ${\rm Model}_2$ & ${\rm Model}_3$ & ${\rm Model}_4$ & ${\rm Model}_5$ & ${\rm Model}_6$ & ${\rm Model}_7$ & ${\rm Model}_8$ & ${\rm Model}_9$ & ${\rm Model}_{10}$ \\
  \midrule
  0.00 & 0.1511 & 0.1461 & 0.1469 & 0.1429 & 0.1436 & 0.1491 & 0.1434 & 0.1479 & 0.1484 & 0.1484 \\
  0.01         & 0.1456 & 0.1453 & 0.1421 & {\bf 0.1465 }& {\bf 0.1436} & 0.1484 & {\bf 0.1435} & {\bf 0.1479} & 0.1479 & 0.1477 \\
  0.05         & 0.1479 & 0.1444 & 0.1402 & 0.1413 & 0.1385 & 0.1413 & 0.1402 & 0.1464 & 0.1464 & 0.1471 \\
  0.10          & 0.1482 & 0.1437 & 0.1430 & 0.1403 & 0.1370 & 0.1425 & 0.1397 & 0.1433 & 0.1387 & 0.1445 \\
  0.50          & 0.1455 & 0.1363 & 0.1368 & 0.1362 & 0.1340 & 0.1406 & 0.1355 & 0.1369 & 0.1385 & 0.1417 \\
  1.00          & 0.1472 & 0.1359 & 0.1383 & 0.1343 & 0.1358 & 0.1392 & 0.1360 & 0.1364 & 0.1369 & 0.1422 \\
  5.00          & 0.1353 & 0.1321 & 0.1345 & 0.1334 & 0.1310 & 0.1334 & 0.1344 & 0.1366 & 0.1322 & 0.1345 \\
  10.00         & 0.1341 & 0.1313 & 0.1320 & 0.1326 & 0.1329 & 0.1343 & 0.1304 & 0.1350 & 0.1306 & 0.1337 \\
  \bottomrule
  \end{tabular}
  }
  \label{tab:att_per_results_l2d0.1}
\end{table*}

\begin{table*}[t!]
\centering
\caption{Objective values $F$ of Alg.~\ref{alg:Mgreedy} under $l_{2}$-norm attacks with varying $\epsilon_2 \in [0, 10]$. Parameters: $D=1$, $\eta$ values from Tab.~\ref{tab:eta}.}
\scalebox{1}{
\begin{tabular}{l|ccccccccccc}
\toprule
\textbf{\boldmath$\epsilon_2$} & ${\rm Model}_1$ & ${\rm Model}_2$ & ${\rm Model}_3$ & ${\rm Model}_4$ & ${\rm Model}_5$ & ${\rm Model}_6$ & ${\rm Model}_7$ & ${\rm Model}_8$ & ${\rm Model}_9$ & ${\rm Model}_{10}$ \\
\midrule
0.00 & 14.7757 & 14.2996 & 14.3820 & 13.9502 & 14.0592 & 14.5925 & 13.9490 & 14.4686 & 14.5209 & 14.5194 \\
0.01         & 14.3609 & 14.0326 & 13.5897 & {\bf 14.0134} & 13.4768 & 13.5853 & 13.7545 & 14.1583 & 14.3729 & 14.0457 \\
0.05         & 13.3674 & 13.8002 & 13.1430 & 13.2985 & 13.0270 & 13.4715 & 13.6194 & 13.8901 & 13.6847 & 13.8387 \\
0.1          & 13.8550 & 13.3263 & 13.4509 & 13.2885 & 13.4143 & 13.3043 & 13.3568 & 14.1224 & 13.4534 & 13.4339 \\
0.5          & 13.2274 & 13.3292 & 13.3167 & 13.2822 & 13.0968 & 13.3608 & 13.3767 & 13.4204 & 13.3758 & 13.2969 \\
1.0          & 13.6029 & 13.3840 & 13.8898 & 13.6794 & 13.2521 & 13.5950 & 13.2556 & 13.5484 & 13.1153 & 13.4189 \\
5.0          & 13.1166 & 12.9472 & 13.0968 & 12.9297 & 12.8962 & 13.1894 & 12.9019 & 13.1366 & 12.9823 & 13.1764 \\
10.0         & 12.9900 & 13.1097 & 12.9613 & 12.8142 & 12.7174 & 13.0276 & 12.8289 & 13.1637 & 13.0128 & 13.0594 \\
\bottomrule
\end{tabular}
}
\label{tab:att_per_results_l2d1}
\end{table*}

\begin{table*}[t!]
\centering
\caption{Objective values $F$ of Alg.~\ref{alg:Mgreedy} under $l_{2}$-norm attacks with varying $\epsilon_2 \in [0, 10]$. Parameters: $D=10$, $\eta$ values from Tab.~\ref{tab:eta}.}
\scalebox{1}{
\begin{tabular}{l|cccccccccc}
\toprule
\textbf{\boldmath$\epsilon_2$} & ${\rm Model}_1$ & ${\rm Model}_2$ & ${\rm Model}_3$ & ${\rm Model}_4$ & ${\rm Model}_5$ & ${\rm Model}_6$ & ${\rm Model}_7$ & ${\rm Model}_8$ & ${\rm Model}_9$ & ${\rm Model}_{10}$ \\
\midrule
0.00 & 1474.2324 & 1426.8492 & 1435.0023 & 1391.8554 & 1402.8959 & 1410.4395 & 1391.6993 & 1443.6483 & 1448.8861 & 1448.7191 \\
0.01         & 1393.3562 & 1360.0638 & 1374.8273 & 1364.2667 & 1322.6239 & {\bf 1414.8718} & 1324.4119 & 1370.1080 & 1370.7363 & 1378.8443 \\
0.05         & 1383.6321 & 1383.4909 & 1330.0605 & 1307.5915 & 1314.5249 & 1349.5608 & 1320.0215 & 1362.2623 & 1339.1132 & 1404.0497 \\
0.10          & 1410.6906 & 1314.6269 & 1332.8700 & 1377.2207 & 1345.0731 & 1363.3326 & 1340.2059 & 1419.6103 & 1345.4399 & 1403.3197 \\
0.50          & 1350.5515 & 1332.4207 & 1305.3593 & 1304.0124 & 1300.3575 & 1336.1150 & 1331.5484 & 1346.6761 & 1348.2639 & 1329.5303 \\
1.00          & 1363.8197 & 1366.5286 & 1338.7542 & 1331.9835 & 1367.6243 & 1336.4774 & 1305.5736 & 1395.5431 & 1336.7558 & 1380.1395 \\
5.00          & 1302.3481 & 1292.5827 & 1293.3264 & 1294.2322 & 1302.1482 & 1298.1031 & 1268.5876 & 1322.8287 & 1303.6870 & 1323.0778 \\
10.00         & 1328.7872 & 1275.8841 & 1304.6443 & 1283.5058 & 1298.7070 & 1306.2990 & 1276.3863 & 1310.0792 & 1272.1141 & 1322.5080 \\
\bottomrule
\end{tabular}
}
\label{tab:att_per_results_l2d10}
\end{table*}

\begin{table*}[t!]
\centering
\caption{Objective values $F$ of Alg.~\ref{alg:Mgreedy} under $l_{\infty}$-norm attacks with varying $\epsilon_\infty \in [0, 1]$. Parameters: $D=0.1$, $\eta=0.1$.}
\scalebox{1}{
\begin{tabular}{l|cccccccccc}
\toprule
\textbf{\boldmath$\epsilon_\infty$} & ${\rm Model}_1$ & ${\rm Model}_2$ & ${\rm Model}_3$ & ${\rm Model}_4$ & ${\rm Model}_5$ & ${\rm Model}_6$ & ${\rm Model}_7$ & ${\rm Model}_8$ & ${\rm Model}_9$ & ${\rm Model}_{10}$ \\
\midrule
0.00 & 0.1511 & 0.1461 & 0.1469 & 0.1429 & 0.1436 & 0.1491 & 0.1434 & 0.1479 & 0.1484 & 0.1484 \\
0.10          & 0.1479 & 0.1449 & 0.1403 & 0.1411 & 0.1380 & 0.1466 & {\bf 0.1458} & 0.1465 & 0.1476 & 0.1472 \\
0.20          & 0.1482 & 0.1455 & 0.1450 & {\bf 0.1462} & 0.1387 & 0.1482 & {\bf 0.1447} & 0.1474 & 0.1473 & 0.1475 \\
0.30          & 0.1482 & 0.1455 & 0.1450 & {\bf 0.1462} & 0.1387 & 0.1482 & {\bf 0.1447} & 0.1474 & 0.1473 & 0.1475 \\
0.40          & 0.1482 & 0.1455 & 0.1450 & {\bf 0.1462} & 0.1387 & 0.1482 & {\bf 0.1447} & 0.1474 & 0.1473 & 0.1475 \\
0.50          & 0.1482 & 0.1455 & 0.1450 & {\bf 0.1462} & 0.1387 & 0.1482 & {\bf 0.1447} & 0.1474 & 0.1473 & 0.1475 \\
0.60          & 0.1479 & 0.1449 & 0.1403 & 0.1411 & 0.1380 & 0.1466 & {\bf 0.1458} & 0.1465 & 0.1476 & 0.1472 \\
0.70          & 0.1479 & 0.1449 & 0.1403 & 0.1411 & 0.1380 & 0.1466 & {\bf 0.1458} & 0.1465 & 0.1476 & 0.1472 \\
0.80          & 0.1479 & 0.1449 & 0.1403 & 0.1411 & 0.1380 & 0.1466 & {\bf 0.1458} & 0.1465 & 0.1476 & 0.1472 \\
0.90          & 0.1479 & 0.1449 & 0.1403 & 0.1411 & 0.1380 & 0.1466 & {\bf 0.1458} & 0.1465 & 0.1476 & 0.1472 \\
1.00          & 0.1479 & 0.1449 & 0.1403 & 0.1411 & 0.1380 & 0.1466 & {\bf 0.1458} & 0.1465 & 0.1476 & 0.1472 \\
\bottomrule
\end{tabular}
}
\label{tab:att_per_results_linftyd0.1}
\end{table*}

\begin{table*}[t!]
\centering
\caption{Objective values $F$ of Alg.~\ref{alg:Mgreedy} under $l_{\infty}$-norm attacks with varying $\epsilon_\infty \in [0, 1]$. Parameters: $D=1$, $\eta=0.1$.}
\scalebox{1}{
\begin{tabular}{l|cccccccccc}
\toprule
\textbf{\boldmath$\epsilon_\infty$} & ${\rm Model}_1$ & ${\rm Model}_2$ & ${\rm Model}_3$ & ${\rm Model}_4$ & ${\rm Model}_5$ & ${\rm Model}_6$ & ${\rm Model}_7$ & ${\rm Model}_8$ & ${\rm Model}_9$ & ${\rm Model}_{10}$ \\
\midrule
0.00 & 14.7757 & 14.2996 & 14.3820 & 13.9502 & 14.0592 & 14.5925 & 13.9490 & 14.4686 & 14.5209 & 14.5194 \\
0.10          & 14.7200 & 13.8373 & 13.8269 & {\bf 14.3320} & 13.6178 & {\bf 14.6066} & {\bf 13.9508} & {\bf 14.4725} & {\bf 14.5209} & 14.4869 \\
0.20          & 14.7557 & 13.5609 & 13.6631 & 13.5879 & 13.5796 & 13.9308 & 13.7069 & 14.0067 & 13.9347 & 13.8778 \\
0.30          & 14.7441 & 13.5926 & 13.5036 & 13.4356 & 12.7464 & 13.5099 & 13.4836 & 13.7012 & 13.0577 & 14.2977 \\
0.40          & 13.2733 & 13.0266 & 13.0663 & 13.2455 & 12.9559 & 13.4531 & 13.7557 & 13.1864 & 13.5648 & 14.1677 \\
0.50          & 13.1596 & 12.9217 & 12.8158 & 12.8778 & 12.8955 & 12.8753 & 13.0534 & 13.0272 & 12.6506 & 13.2123 \\
0.60          & 13.1162 & 12.9490 & 13.0452 & 12.8062 & 12.8238 & 13.0343 & 12.8726 & 13.1176 & 12.8893 & 13.2046 \\
0.70          & 13.0891 & 12.9490 & 13.0252 & 12.8062 & 12.7869 & 13.0343 & 12.8541 & 13.1074 & 12.9676 & 13.0981 \\
0.80          & 13.0891 & 12.9490 & 13.0252 & 12.8062 & 12.7869 & 13.0343 & 12.8541 & 13.1074 & 12.9676 & 13.0981 \\
0.90          & 13.0891 & 12.9490 & 13.0252 & 12.8062 & 12.7869 & 13.0343 & 12.8541 & 13.1074 & 12.9676 & 13.0981 \\
1.00          & 13.0891 & 12.9490 & 13.0252 & 12.8062 & 12.7869 & 13.0343 & 12.8541 & 13.1074 & 12.9676 & 13.0981 \\
\bottomrule
\end{tabular}
}
\label{tab:att_per_results_linftyd1}
\end{table*}

\begin{table*}[t!]
\centering
\caption{Objective values $F$ of Alg.~\ref{alg:Mgreedy} under $l_{\infty}$-norm attacks with varying $\epsilon_\infty \in [0, 1]$. Parameters: $D=10$, $\eta=0.1$.}
\scalebox{1}{
\begin{tabular}{l|cccccccccc}
\toprule
\textbf{\boldmath$\epsilon_\infty$} & ${\rm Model}_1$ & ${\rm Model}_2$ & ${\rm Model}_3$ & ${\rm Model}_4$ & ${\rm Model}_5$ & ${\rm Model}_6$ & ${\rm Model}_7$ & ${\rm Model}_8$ & ${\rm Model}_9$ & ${\rm Model}_{10}$ \\
\midrule
0.00 & 1474.2324 & 1426.8492 & 1435.0023 & 1391.8554 & 1402.8959 & 1410.4395 & 1391.6993 & 1443.6483 & 1448.8861 & 1448.7191 \\
0.10          & 1356.1932 & 1367.8265 & 1336.5555 & 1332.4540 & 1348.5302 & 1371.4058 & 1336.8640 & 1394.6237 & 1379.7231 & 1380.9254 \\
0.20          & 1299.1402 & 1318.0910 & 1281.0421 & 1293.9322 & 1290.9650 & 1298.2420 & 1298.5391 & 1316.1446 & 1321.3482 & 1314.8729 \\
0.30          & 1330.0928 & 1236.7211 & 1291.2222 & 1252.3489 & 1297.1734 & 1259.0189 & 1276.5407 & 1312.2521 & 1313.2799 & 1331.8045 \\
0.40          & 1241.2052 & 1302.3915 & 1282.0595 & 1273.9369 & 1163.9825 & 1110.1221 & 1219.5449 & 1310.2218 & 1317.0013 & 1246.7736 \\
0.50          & 1323.7156 & 1289.5550 & 1255.3676 & 1283.3641 & 1236.2072 & 1288.7349 & 1293.3822 & 1322.7892 & 1210.6488 & 1302.7711 \\
0.60          & 1315.1328 & 1290.6897 & 1221.0778 & 1312.1990 & 1299.5189 & 1169.9733 & 1250.7673 & 1262.1856 & 1305.9338 & 1309.6077 \\
0.70          & 1302.2765 & 1311.2195 & 1295.3673 & 1299.9777 & 1312.2751 & 1306.1988 & 1301.0384 & 1306.3242 & 1302.1175 & 1273.0084 \\
0.80          & 1345.0386 & 1235.4823 & 1293.4413 & 1308.1246 & 1229.8006 & 1243.8061 & 1229.2508 & 1237.6780 & 1237.8233 & 1329.4045 \\
0.90          & 1303.7923 & 1294.7092 & 1283.0093 & 1184.0212 & 1300.7484 & 1302.7943 & 1283.7163 & 1302.2064 & 1313.1154 & 1332.1812 \\
1.00          & 1323.4922 & 1234.0211 & 1288.2501 & 1247.8889 & 1237.5296 & 1282.0559 & 1163.0216 & 1313.1526 & 1225.3747 & 1321.0384 \\
\bottomrule
\end{tabular}
}
\label{tab:att_per_results_linftyd10}
\end{table*}

\paragraph{Additional Results for the Defense Algorithm (Alg.~\ref{alg:robust})}
We show several results for the performance of defense algorithm.

{\bf Why Intermediate $\gamma$ Values Yield Smaller Loss.} The non-monotonic behavior of the loss with respect to $\gamma$ can be understood from two extreme cases. When $\gamma \to \infty$, the regularization term in Eq.~\ref{eq:robust_problem} forces the weights $\mathbf{w}$ to become nearly uniform across attack types. As a result, the update of $\mathbf{x}$ is affected almost equally by all attacks, which may lead to a more conservative solution and hence larger clean-data loss. In contrast, when $\gamma \to 0$, the objective is dominated by the worst-performing attack type, so that one weight tends to approach $1$ while the others approach $0$. In this case, the update of $\mathbf{x}$ is driven almost entirely by a single attack model, which may also degrade clean-data performance. Therefore, both extremes may increase the loss, while an intermediate $\gamma$ often provides a better balance.

{\bf Per-model Results of the Robust Algorithm.} Besides the averaged metrics, we also report the objective values for each individual model under different settings. From Tab.~\ref{tab:robust_per_results_gamma0.1l1} to Tab.~\ref{tab:robust_per_results_gamma10linfty}, the results consistently satisfy $
F_{p}(\bar{\mathbf{x}}_{T}, \Omega_{i})
\geq
F_{p}^{\rm rob}(\bar{\mathbf{x}}_{T},\Omega_{i})
\geq
F_{p}^{\rm rob,att}(\bar{\mathbf{x}}_{T},\Omega_{i}(\mathbf{v}))
\geq
F_{p}^{\rm att}(\bar{\mathbf{x}}_{T}, \Omega_{i}(\mathbf{v})).
$
This shows that the robust algorithm introduces only moderate loss under clean data, while consistently improving over the attacked greedy solution under adversarial perturbations. Hence, the per-model results support the same conclusion as the averaged metrics: Alg.~\ref{alg:robust} provides effective robustness with acceptable utility loss.

\begin{table}[t!]
\centering
\normalsize
\caption{Objective values $F$ under $l_{1}$-norm attacks with $\gamma=0.1$. Parameters: $K_{\rm robust}=100, T_{\rm robust}=30, D=1$).}
\scalebox{0.7}{
\begin{tabular}{l|c|c|c|c}
\toprule
\textbf{Model} &$F_{p}(\bar{\mathbf{x}}_{T}, \Omega_{i})$ & $F_{p}^{\rm rob}(\bar{\mathbf{x}}_{T},\Omega_{i})$  & 
$F_{p}^{\rm rob, att}(\bar{\mathbf{x}}_{T},\Omega_{i}(\mathbf{v}))$& $F^{\rm att}_{p}(\bar{\mathbf{x}}_{T}, \Omega_{i}(\mathbf{v}))$ \\
\midrule
${\rm Model}_{1}$ & 14.7757 & 14.2978 & 14.0910 & 13.6322 \\
${\rm Model}_{2}$ & 14.2996 & 13.8470 & 13.7345 & 13.0018 \\
${\rm Model}_{3}$ & 14.3820 & 13.8146 & 13.7065 & 13.1876 \\
${\rm Model}_{4}$ & 13.9502 & 13.8111 & 13.6874 & 13.1219 \\
${\rm Model}_{5}$ & 14.0592 & 13.6137 & 13.5729 & 12.8331 \\
${\rm Model}_{6}$ & 14.5925 & 14.0321 & 13.8776 & 13.0850 \\
${\rm Model}_{7}$ & 13.9490 & 13.8161 & 13.6908 & 13.2278 \\
${\rm Model}_{8}$ & 14.4686 & 14.0322 & 13.9061 & 13.2038 \\
${\rm Model}_{9}$ & 14.5209 & 13.9969 & 13.8094 & 13.3297 \\
${\rm Model}_{10}$ & 14.5194 & 13.9817 & 13.9314 & 13.1609 \\
\bottomrule
\end{tabular}
}
\label{tab:robust_per_results_gamma0.1l1}
\end{table}

\begin{table}[t!]
  \centering
  \normalsize
  \caption{Objective values $F$ under $l_{2}$-norm attacks with $\gamma=0.1$. Parameters: $K_{\rm robust}=100, T_{\rm robust}=30, D=1$.}
  \scalebox{0.7}{
  \begin{tabular}{l|c|c|c|c}
  \toprule
\textbf{Model} &$F_{p}(\bar{\mathbf{x}}_{T}, \Omega_{i})$ &  $F_{p}^{\rm rob}(\bar{\mathbf{x}}_{T},\Omega_{i})$  & 
$F_{p}^{\rm rob, att}(\bar{\mathbf{x}}_{T},\Omega_{i}(\mathbf{v}))$& $F^{\rm att}_{p}(\bar{\mathbf{x}}_{T}, \Omega_{i}(\mathbf{v}))$ \\
\midrule
  ${\rm Model}_{1}$ & 14.7757 & 14.2978 & 14.1284 & 13.2274 \\
  ${\rm Model}_{2}$ & 14.2996 & 13.847  & 13.7828 & 13.3292 \\
  ${\rm Model}_{3}$ & 14.382  & 13.8146 & 13.7803 & 13.3167 \\
  ${\rm Model}_{4}$ & 13.9502 & 13.8111 & 13.7707 & 13.2822 \\
  ${\rm Model}_{5}$ & 14.0592 & 13.6137 & 13.6071 & 13.0968 \\
  ${\rm Model}_{6}$ & 14.5925 & 14.0321 & 13.9337 & 13.3608 \\
  ${\rm Model}_{7}$ & 13.949  & 13.8161 & 13.773  & 13.3767 \\
  ${\rm Model}_{8}$ & 14.4686 & 14.0322 & 13.9721 & 13.4204 \\
  ${\rm Model}_{9}$ & 14.5209 & 13.9969 & 13.9144 & 13.3758 \\
  ${\rm Model}_{10}$ & 14.5194 & 13.9817 & 13.9759 & 13.2969 \\
  \hline
  \end{tabular}
  }
\label{tab:robust_per_results_gamma0.1l2}
\end{table}
  
  \begin{table}[t!]
  \centering
  \normalsize
  \caption{Objective values $F$ under $l_{\infty}$-norm attacks with $\gamma=0.1$. Parameters: $K_{\rm robust}=100, T_{\rm robust}=30, D=0.5$.}
  \scalebox{0.7}{
  \begin{tabular}{l|c|c|c|c}
  \toprule
\textbf{Model} &$F_{p}(\bar{\mathbf{x}}_{T}, \Omega_{i})$ &$F_{p}^{\rm rob}(\bar{\mathbf{x}}_{T},\Omega_{i})$  & 
$F_{p}^{\rm rob, att}(\bar{\mathbf{x}}_{T},\Omega_{i}(\mathbf{v}))$& $F^{\rm att}_{p}(\bar{\mathbf{x}}_{T}, \Omega_{i}(\mathbf{v}))$ \\
\midrule
  ${\rm Model}_{1}$ & 14.7757 & 14.2978 & 14.1048 & 13.0891 \\
  ${\rm Model}_{2}$ & 14.2996 & 13.847  & 13.7329 & 12.949  \\
  ${\rm Model}_{3}$ & 14.382  & 13.8146 & 13.7147 & 13.0252 \\
  ${\rm Model}_{4}$ & 13.9502 & 13.8111 & 13.6962 & 12.8062 \\
  ${\rm Model}_{5}$ & 14.0592 & 13.6137 & 13.5727 & 12.7869 \\
  ${\rm Model}_{6}$ & 14.5925 & 14.0321 & 13.8854 & 13.0343 \\
  ${\rm Model}_{7}$ & 13.949  & 13.8161 & 13.6943 & 12.8541 \\
  ${\rm Model}_{8}$ & 14.4686 & 14.0322 & 13.9157 & 13.1074 \\
  ${\rm Model}_{9}$ & 14.5209 & 13.9969 & 13.82   & 12.9676 \\
  ${\rm Model}_{10}$ & 14.5194 & 13.9817 & 13.9353 & 13.0981 \\
  \hline
  \end{tabular}
  }
\label{tab:robust_per_results_gamma0.1linfty}
  \end{table}
  
  \begin{table}[t!]
  \centering
  \normalsize
  \caption{Objective values $F$ under $l_{1}$-norm attacks with $\gamma=1$. Parameters: $K_{\rm robust}=100, T_{\rm robust}=30, D=1)$.}
  \scalebox{0.7}{
  \begin{tabular}{l|c|c|c|c}
  \toprule
\textbf{Model} &$F_{p}(\bar{\mathbf{x}}_{T}, \Omega_{i})$ & $F_{p}^{\rm rob}(\bar{\mathbf{x}}_{T},\Omega_{i})$  & 
$F_{p}^{\rm rob, att}(\bar{\mathbf{x}}_{T},\Omega_{i})$& $F^{\rm att}_{p}(\bar{\mathbf{x}}_{T}, \Omega_{i}(\mathbf{v}))$ \\
\midrule
  ${\rm Model}_{1}$ & 14.7757 & 14.2299 & 14.0945 & 13.6322 \\
  ${\rm Model}_{2}$ & 14.2996 & 13.7697 & 13.7388 & 13.0018 \\
  ${\rm Model}_{3}$ & 14.382  & 13.8609 & 13.7211 & 13.1876 \\
  ${\rm Model}_{4}$ & 13.9502 & 13.8021 & 13.6995 & 13.1219 \\
  ${\rm Model}_{5}$ & 14.0592 & 13.6241 & 13.5714 & 12.8331 \\
  ${\rm Model}_{6}$ & 14.5925 & 14.0642 & 13.8867 & 13.085  \\
  ${\rm Model}_{7}$ & 13.949  & 13.806  & 13.6907 & 13.2278 \\
  ${\rm Model}_{8}$ & 14.4686 & 14.0595 & 13.917  & 13.2038 \\
  ${\rm Model}_{9}$ & 14.5209 & 14.012  & 13.8016 & 13.3297 \\
  ${\rm Model}_{10}$ & 14.5194 & 14.0363 & 13.9221 & 13.1609 \\
  \hline
  \end{tabular}
  }
\label{tab:robust_per_results_gamma1l1}
  \end{table}
  
  \begin{table}[t!]
  \centering
  \normalsize
  \caption{Objective values $F$ under $l_{2}$-norm attacks with $\gamma=1$. Parameters: $K_{\rm robust}=100, T_{\rm robust}=30, D=1)$.}
  \scalebox{0.7}{
  \begin{tabular}{l|c|c|c|c}
  \toprule
\textbf{Model} &$F_{p}(\bar{\mathbf{x}}_{T}, \Omega_{i})$ &  $F_{p}^{\rm rob}(\bar{\mathbf{x}}_{T},\Omega_{i})$  & 
$F_{p}^{\rm rob, att}(\bar{\mathbf{x}}_{T},\Omega_{i}(\mathbf{v}))$& $F^{\rm att}_{p}(\bar{\mathbf{x}}_{T}, \Omega_{i}(\mathbf{v}))$ \\
\midrule
  ${\rm Model}_{1}$ & 14.7757 & 14.2299 & 14.1512 & 13.2274 \\
  ${\rm Model}_{2}$ & 14.2996 & 13.7697 & 13.7972 & 13.3292 \\
  ${\rm Model}_{3}$ & 14.382  & 13.8609 & 13.788  & 13.3167 \\
  ${\rm Model}_{4}$ & 13.9502 & 13.8021 & 13.7607 & 13.2822 \\
  ${\rm Model}_{5}$ & 14.0592 & 13.6241 & 13.6126 & 13.0968 \\
  ${\rm Model}_{6}$ & 14.5925 & 14.0642 & 13.9551 & 13.3608 \\
  ${\rm Model}_{7}$ & 13.949  & 13.806  & 13.7672 & 13.3767 \\
  ${\rm Model}_{8}$ & 14.4686 & 14.0595 & 13.969  & 13.4204 \\
  ${\rm Model}_{9}$ & 14.5209 & 14.012  & 13.908  & 13.3758 \\
  ${\rm Model}_{10}$ & 14.5194 & 14.0363 & 13.9863 & 13.2969 \\
  \hline
  \end{tabular}
  }
  \label{tab:robust_per_results_gamma1l2}
  \end{table}
  
\begin{table}[t!]
  \centering
  \normalsize
 \caption{Objective values $F$ under $l_{\infty}$-norm attacks with $\gamma=1$. Parameters: $K_{\rm robust}=100, T_{\rm robust}=30, D=0.5)$.}
  \scalebox{0.7}{
  \begin{tabular}{l|c|c|c|c}
  \toprule
\textbf{Model} &$F_{p}(\bar{\mathbf{x}}_{T}, \Omega_{i})$ & $F_{p}^{\rm rob}(\bar{\mathbf{x}}_{T},\Omega_{i})$  & 
$F_{p}^{\rm rob, att}(\bar{\mathbf{x}}_{T},\Omega_{i}(\mathbf{v}))$& $F^{\rm att}_{p}(\bar{\mathbf{x}}_{T}, \Omega_{i}(\mathbf{v}))$ \\
\midrule
  ${\rm Model}_{1}$ & 14.7757 & 14.2299 & 14.1095 & 13.0891 \\
  ${\rm Model}_{2}$ & 14.2996 & 13.7697 & 13.7507 & 12.949  \\
  ${\rm Model}_{3}$ & 14.382  & 13.8609 & 13.7181 & 13.0252 \\
  ${\rm Model}_{4}$ & 13.9502 & 13.8021 & 13.7095 & 12.8062 \\
  ${\rm Model}_{5}$ & 14.0592 & 13.6241 & 13.5799 & 12.7869 \\
  ${\rm Model}_{6}$ & 14.5925 & 14.0642 & 13.8915 & 13.0343 \\
  ${\rm Model}_{7}$ & 13.949  & 13.806  & 13.7061 & 12.8541 \\
  ${\rm Model}_{8}$ & 14.4686 & 14.0595 & 13.922  & 13.1074 \\
  ${\rm Model}_{9}$ & 14.5209 & 14.012  & 13.8126 & 12.9676 \\
  ${\rm Model}_{10}$ & 14.5194 & 14.0363 & 13.9329 & 13.0981 \\
  \hline
  \end{tabular}
  }
\label{tab:robust_per_results_gamma1linfty}
  \end{table}
  
  \begin{table}[t!]
  \centering
  \normalsize
   \caption{Objective values $F$ under $l_{1}$-norm attacks with $\gamma=10$. Parameters: $K_{\rm robust}=100, T_{\rm robust}=30, D=1)$.}
  \scalebox{0.7}{
  \begin{tabular}{l|c|c|c|c}
  \toprule
\textbf{Model} &$F_{p}(\bar{\mathbf{x}}_{T}, \Omega_{i})$ & $F_{p}^{\rm rob}(\bar{\mathbf{x}}_{T},\Omega_{i})$ & 
$F_{p}^{\rm rob, att}(\bar{\mathbf{x}}_{T},\Omega_{i}(\mathbf{v}))$& $F^{\rm att}_{p}(\bar{\mathbf{x}}_{T}, \Omega_{i}(\mathbf{v}))$ \\
\midrule
  ${\rm Model}_{1}$ & 14.7757 & 14.0579 & 14.0259 & 13.6322 \\
  ${\rm Model}_{2}$ & 14.2996 & 13.7544 & 13.6889 & 13.0018 \\
  ${\rm Model}_{3}$ & 14.382  & 13.6549 & 13.6758 & 13.1876 \\
  ${\rm Model}_{4}$ & 13.9502 & 13.6953 & 13.5945 & 13.1219 \\
  ${\rm Model}_{5}$ & 14.0592 & 13.5546 & 13.5416 & 12.8331 \\
  ${\rm Model}_{6}$ & 14.5925 & 13.7843 & 13.8238 & 13.085  \\
  ${\rm Model}_{7}$ & 13.949  & 13.6951 & 13.5806 & 13.2278 \\
  ${\rm Model}_{8}$ & 14.4686 & 13.885  & 13.9145 & 13.2038 \\
  ${\rm Model}_{9}$ & 14.5209 & 13.7452 & 13.7495 & 13.3297 \\
  ${\rm Model}_{10}$ & 14.5194 & 13.9167 & 13.8738 & 13.1609 \\
  \hline
  \end{tabular}
  }
\label{tab:robust_per_results_ganma10l1}
  \end{table}
  
  \begin{table}[t!]
  \centering
  \normalsize
   \caption{Objective values $F$ under $l_{2}$-norm attacks with $\gamma=10$. Parameters: $K_{\rm robust}=100, T_{\rm robust}=30, D=1)$.}
  \scalebox{0.7}{
  \begin{tabular}{l|c|c|c|c}
  \toprule
\textbf{Model} & $F_{p}(\bar{\mathbf{x}}_{T}, \Omega_{i})$ & $F_{p}^{\rm rob}(\bar{\mathbf{x}}_{T},\Omega_{i})$ & 
$F_{p}^{\rm rob, att}(\bar{\mathbf{x}}_{T},\Omega_{i}(\mathbf{v}))$& $F^{\rm att}_{p}(\bar{\mathbf{x}}_{T}, \Omega_{i}(\mathbf{v}))$ \\
\midrule
  ${\rm Model}_{1}$ & 14.7757 & 14.0579 & 14.0914 & 13.2274 \\
  ${\rm Model}_{2}$ & 14.2996 & 13.7544 & 13.7465 & 13.3292 \\
  ${\rm Model}_{3}$ & 14.382  & 13.6549 & 13.7477 & 13.3167 \\
  ${\rm Model}_{4}$ & 13.9502 & 13.6953 & 13.6913 & 13.2822 \\
  ${\rm Model}_{5}$ & 14.0592 & 13.5546 & 13.5397 & 13.0968 \\
  ${\rm Model}_{6}$ & 14.5925 & 13.7843 & 13.7952 & 13.3608 \\
  ${\rm Model}_{7}$ & 13.949  & 13.6951 & 13.7238 & 13.3767 \\
  ${\rm Model}_{8}$ & 14.4686 & 13.885  & 13.9524 & 13.4204 \\
  ${\rm Model}_{9}$ & 14.5209 & 13.7452 & 13.8496 & 13.3758 \\
  ${\rm Model}_{10}$ & 14.5194 & 13.9167 & 13.9218 & 13.2969 \\
  \hline
  \end{tabular}
  }
\label{tab:robust_per_results_gamma10l2}
  \end{table}

  \begin{table}[t!]
  \centering
  \normalsize
  \caption{Objective values $F$ under $l_{\infty}$-norm attacks with $\gamma=10$. Parameters: $K_{\rm robust}=100, T_{\rm robust}=30, D=0.5)$.}
  \scalebox{0.7}{
  \begin{tabular}{l|c|c|c|c}
 \toprule
\textbf{Model} &$F_{p}(\bar{\mathbf{x}}_{T}, \Omega_{i})$ & $F_{p}^{\rm rob}(\bar{\mathbf{x}}_{T},\Omega_{i})$ & 
$F_{p}^{\rm rob, att}(\bar{\mathbf{x}}_{T},\Omega_{i}(\mathbf{v}))$& $F^{\rm att}_{p}(\bar{\mathbf{x}}_{T}, \Omega_{i}(\mathbf{v}))$ \\
\midrule
  ${\rm Model}_{1}$ & 14.7757 & 14.0579 & 14.0508 & 13.0891 \\
  ${\rm Model}_{2}$ & 14.2996 & 13.7544 & 13.7812 & 12.949  \\
  ${\rm Model}_{3}$ & 14.382  & 13.6549 & 13.7406 & 13.0252 \\
  ${\rm Model}_{4}$ & 13.9502 & 13.6953 & 13.7374 & 12.8062 \\
  ${\rm Model}_{5}$ & 14.0592 & 13.5546 & 13.5332 & 12.7869 \\
  ${\rm Model}_{6}$ & 14.5925 & 13.7843 & 13.8321 & 13.0343 \\
  ${\rm Model}_{7}$ & 13.949  & 13.6951 & 13.6681 & 12.8541 \\
  ${\rm Model}_{8}$ & 14.4686 & 13.885  & 13.8686 & 13.1074 \\
  ${\rm Model}_{9}$ & 14.5209 & 13.7452 & 13.7479 & 12.9676 \\
  ${\rm Model}_{10}$ & 14.5194 & 13.9167 & 13.8852 & 13.0981 \\
  \hline
  \end{tabular}
  }
\label{tab:robust_per_results_gamma10linfty}
  \end{table}

\clearpage
\bibliographystyle{IEEEtran}
\bibliography{sample-base}

\end{document}